\documentclass{article}

\PassOptionsToPackage{numbers,compress}{natbib}
\usepackage[main,final]{neurips_2026}

\usepackage{microtype}
\usepackage{graphicx}
\usepackage{subcaption}
\usepackage{booktabs} 
\usepackage{wrapfig}
\usepackage{multirow}
\usepackage{enumitem}
\usepackage{tikz}
\usetikzlibrary{arrows.meta}
\usepackage[hidelinks]{hyperref}
\usepackage{dblfloatfix} 
\usepackage[most]{tcolorbox}
\newtcbox{\compline}{
  on line,
  boxrule=0.38pt,
  colframe=blue!45!black,
  colback=blue!7,
  arc=0.8pt,
  left=0.8pt,
  right=0.8pt,
  top=0.8pt,
  bottom=0.8pt,
  boxsep=0pt
}

\usepackage[linesnumbered,ruled,vlined]{algorithm2e}
\usepackage{amsmath}
\usepackage{amssymb}
\usepackage{mathtools}
\usepackage{amsthm}

\usepackage[capitalize,noabbrev]{cleveref}
\usepackage[abbreviations]{foreign}

\usepackage{minitoc}
\doparttoc
\AtBeginDocument{%
  \setcounter{tocdepth}{2}%
  \setcounter{parttocdepth}{2}%
}
\newcommand{\appendixpart}{%
  \refstepcounter{part}%
  \adjustptc%
  \addcontentsline{toc}{part}{\protect\numberline{\thepart}Appendix}%
}

\theoremstyle{plain}
\newtheorem{theorem}{Theorem}[section]

\newtheorem{lemma}[theorem]{Lemma}
\newtheorem{corollary}[theorem]{Corollary}
\theoremstyle{definition}
\newtheorem{definition}[theorem]{Definition}
\newtheorem{assumption}[theorem]{Assumption}
\theoremstyle{remark}
\newtheorem{remark}[theorem]{Remark}

\usepackage[textsize=tiny]{todonotes}

\newboolean{showcomments}
\setboolean{showcomments}{false}
\ifthenelse{\boolean{showcomments}}
{ \newcommand{\mynote}[3]{
		\fbox{\bfseries\sffamily\scriptsize#1}
		{\small$\blacktriangleright$\textsf{\emph{\color{#3}{#2}}}$\blacktriangleleft$}}
	\newcommand{\zzz}[1]{{\setlength{\fboxsep}{2pt}\fcolorbox{black}{yellow}{\textsf{\emph{#1}}}}\xspace}}
{ \newcommand{\mynote}[3]{}
	\newcommand{\zzz}[1]{}}

\definecolor{BlueViolet}{RGB}{138,43,226}

\newcommand{\sys}{\textsc{eris}\xspace}

\title{\textsc{eris}: Enhancing the Privacy–Utility–Scalability Trade-Off in Federated Learning via Federated Shard Aggregation}

\author{%
  \textbf{Dario Fenoglio\textsuperscript{1,2}}\thanks{Corresponding author: \texttt{dario.fenoglio@usi.ch}} \quad
  \textbf{Pasquale Polverino\textsuperscript{1}} \quad
  \textbf{Akash Dhasade\textsuperscript{3}} \\
  \textbf{Jacopo Quizi\textsuperscript{1}} \quad
  \textbf{Martin Gjoreski\textsuperscript{1}} \quad
  \textbf{Marc Langheinrich\textsuperscript{1}} \\[3pt]
  \textsuperscript{1}Universit\`a della Svizzera italiana, Lugano, Switzerland \\
  \textsuperscript{2}Massachusetts Institute of Technology, Cambridge, MA, USA \\
  \textsuperscript{3}Carnegie Mellon University, Pittsburgh, PA, USA
}

\begin{document}
\faketableofcontents
\maketitle

\begin{abstract}
    Scaling Federated Learning (FL) to billion-parameter models forces a challenging trade-off between privacy, scalability, and model utility. 
    Existing solutions often tackle these challenges in isolation, sacrificing accuracy, relying on costly cryptographic tools, or introducing communication and optimization inefficiencies that affect convergence.
    We introduce \textsc{eris}, an FL framework centered on \emph{Federated Shard Aggregation} (FSA), a novel mechanism that partitions each client update into non-overlapping shards whose aggregation is distributed across multiple client-side aggregators.
    FSA removes the central aggregation bottleneck, limits the information visible to any single observer, and preserves the centralized FL update after reassembly. \textsc{eris} can further readily integrate Distributed Shifted Compression (DSC) to reduce transmitted payloads and exposed coordinates. 
    We prove that \textsc{eris} preserves convergence under standard assumptions and bounds mutual information leakage by the observable fraction of each update, decreasing with the number of client-side aggregators, and with the compression level when DSC is enabled. Experiments across image and text tasks, including large language models, show that \textsc{eris} achieves FedAvg-level utility while substantially reducing communication bottlenecks and improving robustness to membership inference and reconstruction attacks, without relying on heavy cryptography or utility-degrading perturbations.
\end{abstract}


\section{Introduction} 

The widespread digitalization has led to an unprecedented volume of data being continuously recorded. However, much of this data is sensitive, introducing privacy risks and regulatory constraints that limit its usability~\citep{gdpr}. Federated Learning (FL) has emerged as a distributed learning paradigm that enables multiple clients to collaboratively train machine learning (ML) models without directly sharing their private data~\citep{mcmahan_FL}. By keeping data local, FL can unlock sensitive distributed data from hospitals, corporations, vehicles, and personal devices that would otherwise remain inaccessible. At the same time, the growing demand for high-capacity models, including foundation models and large language models (LLMs), makes practical FL increasingly dependent on three fundamental requirements: preserving privacy, scaling efficiently, and maintaining model utility \cite{yuFederatedFoundationModels2024, woisetschlager2024survey, chengFederatedLargeLanguage2025}.

Although FL avoids direct data sharing, the exchanged updates can still reveal sensitive information about the underlying training data \cite{fenoglioSecurityPrivacyFLSurvey2026}. Adversaries can exploit gradients or model updates to reconstruct private samples or infer whether specific records were used during training~\citep{liGGL2022a, yue_rog_attack2023, SIA2021a, baiMIA2024, data_rec_gan, DLG, FCleak_label, gradinv, grnn, memb_inf_attack, huMIA2022a}. Existing privacy-preserving approaches mainly follow two directions. 
Cryptographic protocols 
can hide client updates from the server but introduce additional communication, computation, or hardware requirements, thereby hurting scalability \cite{SMC3, HE3, SMC, SMC2, HE, HE2, kairouzDistributedDiscrete2021, zhaoSEAR2022, yazdinejadAP2FL2024, HashemiTEE2021}. 
Perturbation-based mechanisms, such as differential privacy (DP), or pruning, 
reduce leakage by modifying the transmitted updates, but often degrade model utility, especially in large models or low-data regimes~\citep{HE2023, DP_client, xieDefendingMIA2021, zieglerDefendingRec2022, DP, agarwalCpSGD2018, zongDP2021, sunLDPFL2021, dingDP2021, girgisShuffledDP2021, shenMemDefense2024, pruning_FL, bibikarDynamicSparse2022, priprune2024, alistarhConvSpars2018, sunSoteria2021}. 
Thus, effective privacy protection in FL remains tightly coupled with costly system overheads or accuracy loss.

Scaling FL to modern models further exposes the limits of centralized aggregation. In traditional FL, a central server collects client updates and redistributes the updated model every round, creating severe network-utilization imbalance and a single aggregation bottleneck as the number of clients grows~\cite{bonawitzFederatedLearningScale}. This issue is amplified by billion-parameter models~\citep{BERT2019, openaiGPT42024}, where transmitting full updates becomes prohibitively expensive. Fully decentralized architectures distribute communication across nodes, but rely on neighborhood-only aggregation, often reducing utility or slowing convergence depending on the topology~\cite{bellet2022d,le2023refined, kalraProxyFLDecentralizedFederated2023, chenFedDualPairWiseGossip2023, ako2016, huDecentralizedFederatedLearning2019a}. Compression techniques, such as quantization or sparsification, reduce payloads, but aggressive compression can also affect convergence or degrade accuracy~\citep{jiangPruning2023, gradient_comp_efficiency, liuDecentralizedFL2022, zhaoBEERFast}. Similar limitations hold for parameter-efficient fine-tuning, 
which often remains outperformed by full-parameter fine-tuning~\citep{rajeRavan2025,sunImprovingLoRA024}. 
As a result, existing scalable FL methods address the central bottleneck or communication overhead at the cost of reduced model utility.

For FL to be practically useful, privacy and scalability should not come at the expense of the learning. However, existing methods introduce approximations, perturbations, or deviations from centralized aggregation that alter the optimization trajectory. DP and pruning can suppress informative gradient components~\citep{DP_client, xieDefendingMIA2021, zieglerDefendingRec2022, DP, agarwalCpSGD2018, shenMemDefense2024, pruning_FL, bibikarDynamicSparse2022, priprune2024}; compression can introduce additional variance or require more rounds~\citep{gradient_comp_efficiency, liuDecentralizedFL2022, zhaoBEERFast, liAccelerationCompressedGradient2020a}; and fully decentralized schemes often aggregate only within client-local neighborhoods, affecting convergence~\citep{bellet2022d, le2023refined, kalraProxyFLDecentralizedFederated2023, chenFedDualPairWiseGossip2023}. This creates a persistent trade-off: methods that improve privacy or scalability often fail to preserve the utility of centralized FedAvg-style collaboration.

To address these limitations, we propose \textsc{eris}, an FL framework designed to jointly satisfy privacy, scalability, and utility. 
At its core, \textsc{eris} introduces \emph{Federated Shard Aggregation} (FSA), which partitions each client update into non-overlapping shards and distributes their aggregation across multiple client-side aggregators. 
Because the shards are disjoint and complete, clients recover after reassembly the same global update as induced by centralized aggregation, while no single aggregator ever observes a full client update. 
Furthermore, we show that FSA readily integrates with \emph{Distributed Shifted Compression} (DSC)~\citep{soteriafl2022}, a key component of modern systems that effectively reduces the number of transmitted parameters but does not, by itself, guarantee privacy.
To the best of our knowledge, \textsc{eris} is the first FL framework that simultaneously provides FedAvg-equivalent utility, information-theoretic privacy amplification, and scalable distributed aggregation without relying on heavy cryptography or utility-degrading perturbations.
Our key contributions are:
\begin{itemize}[noitemsep, topsep=0.pt, parsep=0.5pt, partopsep=0pt, leftmargin=1.8em]
    \item \textbf{Federated Shard Aggregation.} We introduce FSA, a distributed aggregation mechanism that shards client updates across multiple aggregators. FSA removes the central aggregation bottleneck, limits the information visible to any single observer, and preserves the centralized FL update after reassembly.
    \item \textbf{Seamless integration with Distributed Shifted Compression.} We show that FSA naturally supports DSC as a pre-processing layer, reducing transmitted and exposed coordinates. 
    This further enhances scalability and privacy while retaining the strong utility provided by FSA.
   \item \textbf{Theoretical guarantees and large-scale validation.} We provide convergence guarantees and information-theoretic privacy bounds showing that leakage decreases with the number of aggregators and, when DSC is enabled, with the compression level. Extensive experiments on four image and two text datasets---from small models to modern LLMs---and under two threat models against six SOTA baselines confirm \textsc{eris}'s strong privacy--utility--scalability trade-off.
\end{itemize}

With \textsc{eris}, we hope to provide a foundation for next-generation FL systems where privacy and scalability can be strengthened without sacrificing model utility.

\section{Background and Related Work}
\label{sec.back}
\paragraph{Federated Learning (FL).} 
Traditional FL systems~\citep{mcmahan_FL} consist of $K\!\in\!\mathbb{N}$ clients, denoted by $\mathcal{K}\!=\!\{1,2,\dots,K\}$, coordinated by a central server to collaboratively train an ML model over a distributed dataset $D$. 
Each client $k\!\in\!\mathcal{K}$ holds a private dataset $D_k\!=\!\{d_{k,s}\}_{s=1}^{S_k}$ with $S_k$ samples. During each round, clients independently update model parameters $\mathbf{x}^t\!\in\!\mathbb{R}^n$ by minimizing a nonconvex local loss $f(D_k;\mathbf{x}^t)$, producing stochastic gradients $\tilde{\mathbf{g}}_k^t$. 
After local training, each client transmits its gradients to the server, which aggregates them using a permutation-invariant operation, and updates the global model as \(\mathbf{x}^{t+1}\!=\!\mathbf{x}^t\!-\!\lambda_t \tilde{\mathbf{g}}^t\), where $\lambda_t$ is the learning rate. It then broadcasts $\mathbf{x}^{t+1}$ back to the clients for next round. In general, FL aims to minimize:
{
\setlength{\abovedisplayskip}{3pt} 
\setlength{\belowdisplayskip}{3pt} 
\setlength{\jot}{1pt} 
\begin{equation}
\arg\min_{\mathbf{x}} \frac{1}{K} \sum_{k=1}^K f(D_k; \mathbf{x}), \quad \text{where} \quad f(D_k; \mathbf{x}) := \frac{1}{S_k} \sum_{s=1}^{S_k} f(d_{k,s}; \mathbf{x}).
\label{eq:fedavg_opt}
\end{equation}
}

For brevity, 
we denote the loss function of the current model as $f(\mathbf{x}^t)$ for the entire dataset $D$, $f_k(\mathbf{x}^t)$ for the local dataset $D_k$, and $f_{k,s}(\mathbf{x}^t)$ for a single sample $d_{k,s}$, respectively.

\vspace{-1pt}
\paragraph{Privacy-preserving FL.} $\;$ 
Although FL avoids direct data sharing, transmitted updates (\eg, gradients) can still leak sensitive information~\citep{fenoglioSecurityPrivacyFLSurvey2026,liGGL2022a, yue_rog_attack2023, SIA2021a, baiMIA2024, data_rec_gan, DLG, FCleak_label, gradinv, zhangMIA2023, heEMIA2024, embedding_leak, FCleak_input, data_rec_gan2, dimitrovDataLeakage2022, passive_active_memb, zariMIA2021, liPassiveMIA2022}. 
The central server, which collects full client updates, is therefore a primary vulnerability.
Two dominant families of privacy-preserving methods are \emph{cryptographic techniques} and \emph{perturbation-based mechanisms}.
Cryptographic techniques, such as secure aggregation~\citep{SMC3, HE3, SMC, SMC2, HE, HE2, kairouzDistributedDiscrete2021} and trusted execution environments~\citep{zhaoSEAR2022, yazdinejadAP2FL2024, HashemiTEE2021}, mask client updates from the server, but introduce significant computational overhead or require specialized hardware. 
Perturbation-based mechanisms, such as local differential privacy (LDP)~\citep{DP_client, xieDefendingMIA2021, zieglerDefendingRec2022, DP, agarwalCpSGD2018, zongDP2021, sunLDPFL2021, dingDP2021, girgisShuffledDP2021, miaoCompressedLDP2022, yangLDPsurvey2024, adnanLDP2022, jayaramanDPFLs2018, truexLDPFed2020, jinLDP2023} and gradient pruning~\citep{shenMemDefense2024, pruning_FL, bibikarDynamicSparse2022, priprune2024, alistarhConvSpars2018, sunSoteria2021, jiangPruning2023}, reduce leakage by modifying updates: LDP clips and adds noise to provide formal guarantees, while pruning removes informative gradient components. 
However, both approaches often incur substantial utility degradation, especially for large models~\citep{liLARGELANGUAGEMODELS2022}. 
Recent works such as LotteryFL~\citep{liLotteryFL2021} and PriPrune~\citep{priprune2024} attenuate this trade-off through personalized pruning schemes, while other methods combine LDP with compression to balance privacy and communication efficiency~\citep{agarwalCpSGD2018, zongDP2021, dingDP2021, soteriafl2022, jinLDP2023}, often at the cost of increased algorithmic complexity or slower convergence. A complementary line of work explores \emph{partial-update schemes} that reduce exposure by transmitting or exchanging only subsets of model updates.
Methods such as Ako~\citep{ako2016} and Shatter~\citep{biswasShatter2025} partition model updates to limit what any single participant observes.
However, their partial or neighborhood-based exchanges do not recover the exact centralized aggregate at each round, potentially altering the optimization trajectory; in the case of Shatter, this also incurs more than twice the communication overhead. 
These limitations motivate alternative privacy mechanisms that can limit information exposure while preserving the efficiency and utility of centralized aggregation.

\vspace{-1pt}
\paragraph{Scalable FL.} $\;$
Scalability is a central limitation of traditional FL systems, arising from single server bottleneck and high per-round communication cost. 
As the number of participating clients grows, the central server must collect and redistribute updates from all participants, creating severe network-utilization imbalance and server-side congestion. 
\emph{Decentralized architectures} alleviate this bottleneck by distributing communication and computation across multiple nodes
~\citep{kalraProxyFLDecentralizedFederated2023, chenFedDualPairWiseGossip2023, ako2016, huDecentralizedFederatedLearning2019a, liuDecentralizedFL2022, royBrainTorrent2019, pappasIPLS2021, bornsteinSWIFT2023, DBLPgossipe2023, zehtabiDSpodFL2025, shiImproving2023, gabrielliSurveyDFL2026}.
These methods commonly rely on peer-to-peer or gossip-based communication, where clients exchange updates with selected neighbors~\citep{ako2016, huDecentralizedFederatedLearning2019a, royBrainTorrent2019, pappasIPLS2021, bornsteinSWIFT2023, DBLPgossipe2023, zehtabiDSpodFL2025, shiImproving2023}. 
However, such neighborhood-based aggregation generally does not recover the exact centralized aggregate at each round, leading to client-specific models and making optimization dependent on the communication topology.
\emph{Multi-server FL} instead distributes coordination across multiple servers~\citep{quConvergenceMultiServer2023, wangConfederatedLearning2023}, but typically partitions clients rather than individual updates, such that each server may still observe complete updates from its assigned clients unless additional cryptographic mechanisms are employed~\citep{ratheeELSASecureAggregation2023}.
Another line of work employs \emph{model and parameter partitioning}, together with collective communication primitives such as reduce-scatter and all-gather, to improve throughput, memory efficiency, or straggler tolerance~\citep{ako2016, biswasBoostingAsynchronousDecentralized2025, rajbhandariZeRO2020}.
These mechanisms primarily use partitioning for systems or communication efficiency rather than as an information-exposure boundary.
As model sizes increase---reaching billions of parameters---the amount of data transmitted per round grows substantially, making large-scale FL increasingly impractical. 
\emph{Compression techniques} reduce this payload via quantization~\citep{zhaoBEERFast, liAccelerationCompressedGradient2020a, michelusiFiniteBitQuant2022, liCANITAFasterRates2021, alistarhQSGD2017, karimireddyErrorFeedback2019, reisizadehFedPAQ2020, gorbunovMARINAFasterNonConvex2021, mishchenko2023, khiriratCompressed2018, koloskovaDecentralizedStochasticOptimization2019} and sparsification~\citep{liuDecentralizedFL2022, zhaoBEERFast, liAccelerationCompressedGradient2020a, soteriafl2022, liCANITAFasterRates2021, gorbunovMARINAFasterNonConvex2021, khiriratCompressed2018, richtarik3PCT2022, wangATOMO2018, ivkinCommunicationefficient2019}, but naive or aggressive compression can degrade utility or slow convergence, undermining scalability gains~\citep{liAccelerationCompressedGradient2020a, michelusiFiniteBitQuant2022, ivkinCommunicationefficient2019}. 
These limitations motivate distributed aggregation mechanisms that can alleviate central bottlenecks while preserving the exact centralized update and limiting information exposure.

\section{\textsc{eris}} 
\label{sec:method}
In this work, we propose \textsc{eris}, a novel FL framework designed to address key limitations in privacy, scalability, and utility. This section states the design goals (Section~\ref{sec:design_goals}), describes the pipeline (Section~\ref{sec:eris_pipeline}), provides theoretical foundations for the convergence of the learning process (Section~\ref{sec:covergence_theory}), and establishes an information-theoretic upper bound on privacy leakage (Section~\ref{sec:privacy_theory}).

\subsection{Design Goals} \label{sec:design_goals}
We aim to collaboratively train a global model under three requirements: (i) \emph{privacy}: the information available to any single honest-but-curious observer should be limited, preventing direct access to complete client updates; (ii) \emph{scalability}: the protocol should remove the central aggregation bottleneck and balance communication and computation across multiple nodes; and (iii) \emph{utility}: the distributed protocol should preserve the collaborative update induced by centralized aggregation, avoiding approximation errors or perturbations that degrade the learning process.

\begin{figure*}
  \centering
    \captionsetup{skip=4pt} 
 \includegraphics[width=0.99\textwidth]{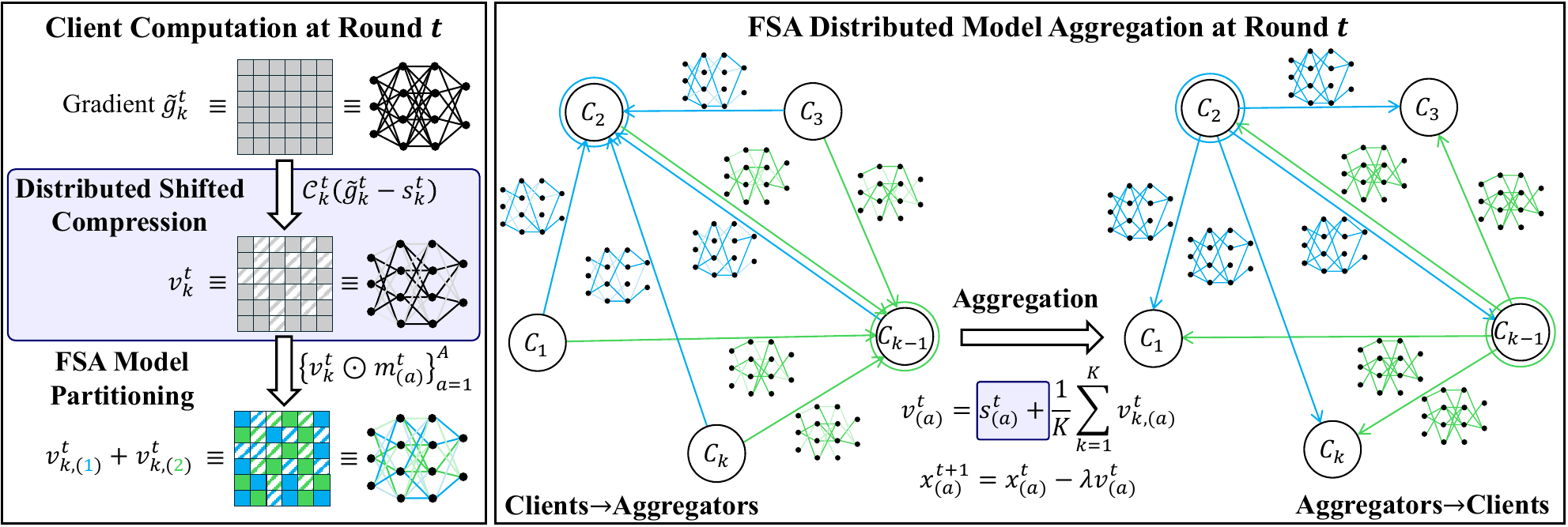}
  \caption{Illustration of \textsc{eris} at training round $t$ for two aggregators ($A\!=\!2$). \textbf{Left:} each client optionally applies \compline{DSC} and then performs FSA partitioning, generating shards \( \mathbf{v}^t_{k,(a)} \) sent to aggregators \( C_2 \) and \( C_{k-1} \). \textbf{Right:} each aggregator collects and aggregates the corresponding shards across clients to produce partial updated models \( \mathbf{x}^{t+1}_{(a)} \), which are sent back to the clients.}
  \label{fig:eris_pipeline}
  \vspace{-8pt}
\end{figure*}

\subsection{The \textsc{eris} Pipeline} \label{sec:eris_pipeline}
The \textsc{eris} pipeline is detailed in Algorithm~\ref{alg:eris} and shown in Figure~\ref{fig:eris_pipeline}. 
In each round $t$, every client $k$ computes a local update $\tilde{\mathbf{g}}_k^t$ on its private dataset $D_k$.
Now, instead of sending the update directly to a central server, \sys leverages \emph{Federated Shard Aggregation} (FSA) to distribute the aggregation task across multiple client-side aggregators.
We first present our core FSA mechanism applied to original gradient updates \ie uncompressed updates.
We then present the variant with Distributed Shifted Compression (DSC) which further enhances scalability and privacy within FSA.
The lines corresponding to DSC are uniquely \compline{highlighted} in Algorithm~\ref{alg:eris}.



\subsubsection{Federated Shard Aggregation (FSA)} \label{sec:fsa}
FSA is the main building block of \textsc{eris}. Instead of sending a complete client update to a central server, each client partitions its update into $A$ disjoint shards and sends each shard to a different aggregator. The aggregators independently combine their assigned shards and return the corresponding model segments to the clients, which then reassemble the global model. 

\paragraph{Shard-wise model partitioning.} $\;$
More formally, let $\mathbf{v}_k^t$ denote the pre-partition update of client $k$ at round $t$. In the basic FSA variant, no compression is applied and $\mathbf{v}_k^t=\tilde{\mathbf{g}}_k^t$. FSA partitions $\mathbf{v}_k^t$ into $A$ disjoint shards using a set of binary masks $\{\mathbf{m}_{(a)}^t\}_{a=1}^A \subset \{0,1\}^n$, where each mask selects the coordinates assigned to aggregator $a$. The masks satisfy:
{
\setlength{\abovedisplayskip}{4pt} 
\setlength{\belowdisplayskip}{3pt}
\begin{align}
    & \textit{Disjointness:} \quad 
    \mathbf{m}_{(a)}^t \odot \mathbf{m}_{(a')}^t = \mathbf{0}
    \quad \forall a \neq a',
    \nonumber \\
    & \textit{Completeness:} \quad 
    \sum_{a=1}^A \mathbf{m}_{(a)}^t = \mathbf{1}_n, 
    \nonumber
\end{align}
}
where $\mathbf{1}_n$ is the all-ones vector and $\odot$ denotes element-wise multiplication. 
Client $k$ then constructs the shard assigned to aggregator $a$ as \(\mathbf{v}_{k,(a)}^t = \mathbf{v}_k^t \odot \mathbf{m}_{(a)}^t,\) where \(a \in \{1,\dots,A\}\), and transmits each shard $\mathbf{v}_{k,(a)}^t$ to its corresponding aggregator. 
Because the masks are disjoint and complete, the original update can be recovered by summing the shards as \(\mathbf{v}_k^t\!=\!\sum_{a=1}^A \mathbf{v}_{k,(a)}^t\). 
Thus, FSA distributes each update across aggregators without discarding any coordinate. 

The $A$ aggregators can be selected before training, e.g., through decentralized peer-sampling or gossip protocols~\cite{jelasityGossipbasedPeerSampling2007}, or consensus-based committee-selection mechanisms~\cite{giladAlgorandScalingByzantine2017}. To limit the coordinate coverage accumulated by any individual non-colluding aggregator over training, we consider a deployment in which the aggregator set and shard assignments remain fixed across rounds, except when an aggregator becomes unavailable and reconfiguration is required.
 

\paragraph{Distributed model aggregation.} $\;$
Each aggregator $a$ receives the shards $\{\mathbf{v}_{k,(a)}^t\}_{k=1}^K$ from the par- ticipating clients and uses a permutation-invariant aggregation over its assigned coordinates, such as:
\begin{equation}
\setlength{\abovedisplayskip}{3pt} 
\setlength{\belowdisplayskip}{3pt} 
    \mathbf{v}_{(a)}^t=
    \frac{1}{K}\sum_{k=1}^K \mathbf{v}_{k,(a)}^t.
    \label{eq:fsa_aggregation}
\end{equation}
The aggregator then updates the corresponding shard of the global model as \( \mathbf{x}_{(a)}^{t+1}\!=\!\mathbf{x}_{(a)}^t - \lambda_t \mathbf{v}_{(a)}^t\).
Finally, each aggregator broadcasts its updated shard $\mathbf{x}_{(a)}^{t+1}$ to the clients, which reassemble the full model as \(\mathbf{x}_k^{t+1}\!=\!\sum_{a=1}^A\mathbf{m}_{(a)}^t \odot \mathbf{x}_{(a)}^{t+1}\). 
By the disjointness and completeness of the masks, this reassembled model is equivalent to the model obtained by aggregating the full client updates centrally. 

Collectively, FSA eliminates the single server bottleneck while still preserving the exact update induced by centralized FL, thereby maintaining strong \emph{utility}. 
By distributing shards across multiple aggregators, it balances communication and computation, improving \emph{scalability}. 
At the same time, each aggregator observes only a subset of every client update, limiting the information accessible to any single honest-but-curious observer and thus enhancing \emph{privacy}. 
This combination distinguishes \textsc{eris} from decentralized methods that offer scalability but sacrifice centralized aggregation and from perturbation based methods that offer privacy but compromise utility. 



\begin{algorithm*}[t]
\caption{\textsc{eris}: Federated Shard Aggregation with optional Distributed Shifted Compression}
{\fontsize{8.}{8.}\selectfont
\label{alg:eris}
\KwIn{
    Initial global model $\mathbf{x}^0$,
    number of aggregators $A$,
    learning rate $\lambda_t$,
    number of clients $K$,
    number of rounds $T$,
    \compline{reference vectors $\mathbf{s}_k^0=\mathbf{0}$ and $\mathbf{s}_{(a)}^0=\mathbf{0}$ for DSC}. 
}
\KwOut{Final global model $\mathbf{x}^T$}

\For{$t = 0, 1, \dots, T-1$}{
    \tcp{Client-side operations}
    \For{\textbf{each client} $k \in \{1, \dots, K\}$ \textbf{in parallel}}{
        Compute local stochastic gradient $\tilde{\mathbf{g}}_k^t$ \;

        \compline{
        \textit{Shifted compression:}
        $\mathbf{v}_k^t =
        \mathcal{C}_k^t\!\left(\tilde{\mathbf{g}}_k^t-\mathbf{s}_k^t\right)$
        }
        \tcp*[r]{without compression: $\mathbf{v}_k^t=\tilde{\mathbf{g}}_k^t$}

        \textit{Shard-wise partitioning:}
        $\{\mathbf{v}_{k,(a)}^t\}_{a=1}^A =
        \{\mathbf{v}_k^t \odot \mathbf{m}_{(a)}^t\}_{a=1}^A$ \;

        Send each shard $\mathbf{v}_{k,(a)}^t$ to aggregator $a$,
        for $a=1,\dots,A$ \;

        \compline{
        Update client reference vector:
        $\mathbf{s}_k^{t+1} =
        \mathbf{s}_k^t + \gamma_t \mathbf{v}_k^t$
        }
    }

    \tcp{Aggregator-side operations}
    \For{\textbf{each aggregator} $a \in \{1, \dots, A\}$ \textbf{in parallel}}{
        Aggregate assigned shard:
        $\mathbf{v}_{(a)}^t =
        \frac{1}{K}\sum_{k=1}^K \mathbf{v}_{k,(a)}^t$
        \;

        \compline{
        Compensate global shift:
        $\mathbf{v}_{(a)}^t \leftarrow
        \mathbf{s}_{(a)}^t + \mathbf{v}_{(a)}^t$
        }

        Update shard of the global model:
        $\mathbf{x}_{(a)}^{t+1} =
        \mathbf{x}_{(a)}^t - \lambda_t \mathbf{v}_{(a)}^t$
        \;

        \compline{
        Update aggregator reference:
        $\mathbf{s}_{(a)}^{t+1} =
        \mathbf{s}_{(a)}^t +
        \gamma_t \frac{1}{K}\sum_{k=1}^K \mathbf{v}_{k,(a)}^t$
        }

        Broadcast updated shard $\mathbf{x}_{(a)}^{t+1}$ to all clients \;
    }

    Each client $k$ reassembles the global model:
    $\mathbf{x}_k^{t+1} =
    \sum_{a=1}^A \mathbf{m}_{(a)}^t \odot \mathbf{x}_{(a)}^{t+1}$ \;
}
}
\vspace{-4pt}
\end{algorithm*}
\vspace{-0pt}

\subsubsection{Amplifying FSA with Distributed Shifted Compression} \label{sec:dsc}
FSA already provides the core privacy and scalability benefits of \textsc{eris}. 
However, the same architecture can naturally support additional privacy- and communication-enhancing mechanisms. In this work, we instantiate this idea with DSC, an optional layer that further reduces both transmitted and exposed coordinates while preserving the FSA structure and utility. We use the standard unbiased compressor definition adopted in communication-efficient FL~\citep{liAccelerationCompressedGradient2020a, liCANITAFasterRates2021, gorbunovMARINAFasterNonConvex2021}.
\begin{definition}[Compression operator]\label{def:compression_operator}
A randomized map $\mathcal{C}:\mathbb{R}^n \rightarrow \mathbb{R}^n$ is an $\omega$-compression operator if, for all $\mathbf{x} \in \mathbb{R}^n$, it satisfies $\omega \geq 0$ and:
\begin{equation}
\setlength{\abovedisplayskip}{4pt} 
\setlength{\belowdisplayskip}{2pt} 
    \mathbb{E}[\mathcal{C}(\mathbf{x})]=\mathbf{x}, 
    \qquad 
    \mathbb{E}\!\left[
    \|\mathcal{C}(\mathbf{x})-\mathbf{x}\|^2
    \right] 
    \leq \omega \|\mathbf{x}\|^2 .
    \label{eq:compressor_operator}
\end{equation}
\end{definition}
Def.~\ref{def:compression_operator} covers common randomized compressors such as quantization and sparsification~\citep{BERT2019, soteriafl2022, alistarhQSGD2017, liAccelerationCompressedGradient2020a, liCANITAFasterRates2021}. 
For instance, random sparsification can be written as $\mathcal{C}_k^t(\mathbf{x})\!=\!\mathbf{x}\odot\mathbf{m}_{\mathcal{C}_k^t}$, where each entry of $\mathbf{m}_{\mathcal{C}_k^t}$ equals $1/p_k$ with probability $p_k$ and $0$ otherwise, giving $\mathbb{E}[\mathbf{m}_{\mathcal{C}_k^t}]\!=\!\mathbf{1}_n$ and $\omega\!=\!(1-p_k)/p_k$.


To reduce communication while preserving favorable convergence behavior, we extend shifted compression~\citep{soteriafl2022} to the distributed FSA setting. 
Each client maintains a reference vector $\mathbf{s}_k^t$ and sends the compressed shifted update \(\mathbf{v}_k^t\!=\!\mathcal{C}_k^t(\tilde{\mathbf{g}}_k^t-\mathbf{s}_k^t)\), then updates \(\mathbf{s}_k^{t+1}\!=\!\mathbf{s}_k^t+\gamma_t\mathbf{v}_k^t\) 
to track the local update direction over time. 
The vector $\mathbf{v}_k^t$ is then passed to FSA, which shards and distributes it as before. 
Hence, DSC augments FSA by simply changing the vector that FSA aggregates. 
Since compression is applied before partitioning, aggregator $a$ receives only a compressed shard. 
To compensate for the shift, it maintains a shard-level reference $\mathbf{s}_{(a)}^t$ and replaces Equation~\eqref{eq:fsa_aggregation} with: 
\begin{equation}
\setlength{\abovedisplayskip}{3pt} 
\setlength{\belowdisplayskip}{3pt}
    \mathbf{v}_{(a)}^t
    =
    \mathbf{s}_{(a)}^t
    +
    \frac{1}{K}
    \sum_{k=1}^K
    \mathbf{v}_{k,(a)}^t, 
    \quad \quad
    \mathbf{s}_{(a)}^{t+1}=
    \mathbf{s}_{(a)}^t
    +
    \gamma_t
    \frac{1}{K}
    \sum_{k=1}^K
    \mathbf{v}_{k,(a)}^t .
    \label{eq:dsc_aggregator_shift_and_dsc_aggregator_reference}
\end{equation}
DSC amplifies both scalability and privacy. 
By reducing the number of transmitted parameters before FSA partitioning, it decreases client upload and per-aggregator traffic while also reducing the coordinates exposed in each shard. 
Importantly, privacy in \textsc{eris} does not rely on aggressive compression, which can otherwise harm utility: FSA already limits each aggregator to a subset of the update, while DSC can further shrink this observable subset. 
This modularity also shows that \textsc{eris} can incorporate other update transformations beyond DSC before the FSA partitioning, including DP (see Section~\ref{pg:pareto}), with the corresponding aggregation-side correction, when needed.
\subsection{Theoretical Analysis of Convergence and Utility} \label{sec:covergence_theory} \vspace{-2pt}
FSA across $A$ aggregators preserves the exact  update induced by centralized FL (Section~\ref{sec:fsa}). 
In Appendix~\ref{sec:convergence-analysis}, we further formally establish that FSA incurs no information loss or deviation from the centralized aggregation trajectory.
Building on this, we analyze the optional DSC layer on top of FSA, showing that \textsc{eris} preserves favorable convergence behavior under compressed updates.
\begin{theorem}[Utility of \textsc{eris} with DSC] \label{thm:utility_eris_main}
    Consider \textsc{eris} under the standard smoothness and unbiased estimator assumptions (Appendix~\ref{app:assumptions}), with compression operators $\mathcal{C}_k^t$ satisfying Definition~\ref{def:compression_operator}. 
    Let $L$ denote the smoothness constant, and let $\Gamma_1,\Gamma_2,\Gamma_3,\Gamma_4,\theta$ and $\Delta^t$ be as defined in Assumption~\ref{assumption:unbiased_estimator}. 
    Let $\alpha\!=\!\frac{3\beta \Gamma_1}{2(1+\omega)L^2\theta}$, for any $\beta\!>\!0$, and let the learning rate be defined as:
    {
    \setlength{\abovedisplayskip}{2pt} 
    \setlength{\belowdisplayskip}{3pt} 
    \begin{equation}
    \lambda_t \!\equiv\! \lambda \!\le\! \min \{\!
    \frac{\sqrt{\beta K}}{\sqrt{1+2\alpha \Gamma_4 + 4 \beta (1+\omega)}(1+\omega)L},
    \frac{1}{\bigl(1+2\alpha \Gamma_4 +4 \beta(1+\omega) + 2 \alpha \Gamma_3  /\lambda^2\bigr)L}
    \!\}
    \end{equation}
    }
    \noindent with shift stepsize $\gamma_t\!=\!\sqrt{(1+2 \omega)/(2(1+\omega)^3)}$. Then, \textsc{eris} satisfies the utility bound:
    {
    \setlength{\abovedisplayskip}{3pt} 
    \setlength{\belowdisplayskip}{3pt} 
    \begin{equation}
        \setlength{\abovedisplayskip}{3pt} 
        \setlength{\belowdisplayskip}{3pt} 
        \frac{1}{T}\sum_{t=0}^{T-1} \|\nabla f(\mathbf{x}^t)\|^2 \leq \frac{2\Phi_0}{\lambda T} + \frac{3\beta \Gamma_2}{(1 + \omega)L \lambda}, \label{eq:thr_utility_main}
    \end{equation}
    }
    where $\Phi_0\!:=\!f(\mathbf{x}^0)\!-\!f^*\!+\!\alpha L \Delta^0\!+\!\frac{\beta}{KL} \sum_{k=1}^K ||\nabla f_k(\mathbf{x}^0)\!-\!\mathbf{s}^0_k ||^2$. Equation \ref{eq:thr_utility_main} implies that the asymptotic utility of \textsc{eris} is governed by the gradient estimator variance $\Gamma_2$, which vanishes for lower-variance estimators such as SVRG/SAGA~\citep{defazioSAGA2014, SVRG2013}. Similarly, a larger local 
    batch size reduces gradient variance, leading to improved convergence, with $\Gamma_2\!=\!0$ when full local gradients are used.

\end{theorem}
In contrast to prior communication-efficient privacy-preserving FL methods~\citep{dingDP2021, soteriafl2022, lowyPrivateNonConvex2023}, the bound in \eqref{eq:thr_utility_main} depends primarily on the gradient-estimator variance $\Gamma_2$ and is independent of the specific privacy-preserving mechanism applied; notably, it contains no term that grows with $T$. The proof and a comparison with existing DP and compressed FL bounds are provided in Appendix~\ref{sec:utility_comm_eris}. 

\subsection{Theoretical Analysis of Privacy Guarantees} \label{sec:privacy_theory}
We analyze \textsc{eris} under the standard \textit{honest-but-curious} threat model~\citep{huangEvaluatingGradInv2021, LLMattack2022, arevaloTaskAgnostic2024}, where an adversary observes and stores transmitted updates, e.g. via eavesdropping or a compromised aggregator/server, and attempts to infer sensitive information about clients' private data $D_k$. FSA limits the adversary's view by ensuring that any fixed aggregator observes only one shard of each client update. DSC further reduces this view by retaining only a fraction $p$ of the coordinates before partitioning. 
To quantify privacy, we bound the mutual information between $D_k$ and the adversary’s view $\mathbf{v}^t_{k,(a)}$ over 
rounds.
\begin{theorem}[Privacy guarantee of \textsc{eris} with DSC]
\label{thm:privacy}
Let \(\mathbf{v}_{k,(a)}^t\!=\!(\tilde{\mathbf{g}}_k^t\!-\!\mathbf{s}_k^t) \odot \mathbf{m}_{\mathcal{C}_k^t} \odot \mathbf{m}_{(a)}^t\) denote the $a$-th transmitted shard of client $k$ at round $t$, where $\mathbf{m}_{\mathcal{C}_k^t}$ is a compression mask satisfying Definition~\ref{def:compression_operator} with probability $p$, and $\mathbf{m}_{(a)}^t$ selects one of $A$ disjoint shards. Write $\mathbf{U}_k^t:=\tilde{\mathbf{g}}_k^t-\mathbf{s}_k^t$. Conditional on the history $\mathcal{H}_t$ of past shards and masks, assume the masks are independent of $(D_k,\mathbf{U}_k^t)$ and expose $np/A$ coordinates in expectation. Assume \(I(D_k;U_{k,i}^t\mid\mathcal{H}_t=h,\mathbf{U}_{k,J}^t)\le C_{\max}<\infty\) uniformly over $i,t$, $J\subseteq\{1,\ldots,n\}\setminus\{i\}$, and almost every $h$. Then, under the honest-but-curious model, the mutual information over $T$ rounds satisfies:
{
\setlength{\abovedisplayskip}{3pt} 
\setlength{\belowdisplayskip}{3pt} 
\begin{equation}
I_k = I\bigl(D_k; \{\mathbf{v}_{k,(a)}^{t}\}_{t=0}^{T-1}\bigr) \le n \; T \; \frac{p}{A} \; C_{\max},
\label{eq:privay_bound}
\end{equation}
}
where $n$ is the model size and $C_{\max}$ bounds the per-coordinate mutual information conditional also on previously revealed coordinates within the round.
\end{theorem}
Theorem~\ref{thm:privacy} shows that the privacy-leakage bound scales with the number of observable coordinates over time, \(nTp/A\). For a common \(C_{\max}\), FSA reduces this upper bound by a factor \(1/A\), while DSC contributes the retention factor \(p\). This captures the most general setting: when DSC is disabled, \(p=1\) and the bound reduces to the FSA-only case \(I_k \le nT C_{\max}/A\). The absolute leakage depends on \(C_{\max}\), which is model-, data-, and training-dependent. Appendix~\ref{app:d1} provides an operational Gaussian-channel instantiation of \(C_{\max}\), while Appendix~\ref{app:weight_distribution} reports a corresponding empirical audit. Full proofs and the extension to colluding adversaries are provided in Appendix~\ref{app:privacy_proof}. These theoretical findings are corroborated by our experiments.

\section{Results} \label{sec:results}

We now evaluate \textsc{eris} along all three dimensions: privacy, scalability, and utility. 
We detail the experimental setup in Section~\ref{sec:exp_setup} and present the results in Section~\ref{sec:num_exps}.

\subsection{Experimental Setup} \label{sec:exp_setup}
\paragraph{Datasets.} $\;$ We evaluate \textsc{eris} on six publicly available datasets spanning image classification and text generation. For image classification, we use MNIST \citep{mnist} and CIFAR-10 \citep{cifar}; for text classification, IMDB Reviews \citep{imdb}; and for text generation, CNN/DailyMail \citep{cnn_dataset}. To evaluate data reconstruction attacks, we additionally use LFW 
\citep{lfw} and ImageNet \cite{dengImageNet2009}. Datasets are randomly partitioned among $K$ clients ($K{=}10$ for CNN/DailyMail, $K{=}25$ for IMDB, and $K{=}50$ for the others); while non-IID scenarios are generated using a Dirichlet distribution with $\alpha \in \{0.2,0.5\}$. 
We initialize GPT-Neo~\cite{gptneo} (1.3B) and DistilBERT~\cite{Sanh2019DistilBERTAD} (67M) from pre-trained checkpoints for CNN/DailyMail and IMDB, respectively, and subsequently perform full-parameter federated fine-tuning. ResNet-9 (1.65M) for CIFAR-10, ResNet-18~\cite{resnet2016} (11.7M) for ImageNet, and LeNet-5~\cite{lenet} (62K) for MNIST and LFW are instead trained from scratch.
Unless otherwise specified, all experiments use 5-fold cross-validation, and reported results are averaged across folds. Training hyperparameters are detailed in Appendix \ref{app:model_and_hyper}.

\paragraph{Baselines.} $\;$
We compare \textsc{eris} against several SOTA methods targeting privacy, scalability, or their trade-off in FL: \textit{Ako}~\citep{ako2016} and \textit{Shatter}~\citep{biswasShatter2025}, decentralized approaches with partial gradient exchange; \textit{SoteriaFL}~\citep{soteriafl2022}, which combines centralized shifted compression with DP; \textit{PriPrune}~\citep{priprune2024}, a pruning strategy that withholds the most informative gradient components from communication; and \textit{LDP}~\citep{sunLDPFL2021}. We also include \textit{FedAvg}~\citep{mcmahan_FL} as the standard baseline with no defenses or compression, and report results for an idealized upper bound (\textit{Min. Leakage}), where clients transmit no gradients and the attack is applied only to the last-round global model.

\paragraph{Privacy Attacks.} $\;$
Under the standard honest-but-curious model, we assume the attacker is a compromised aggregator or server observing client-transmitted updates at every round. 
All methods use the same round budget $T$, capped at FedAvg convergence; full exposure details are provided in Appendix~\ref{app:experimental_setup}.
We evaluate six representative attacks across two widely studied categories: \textit{Membership Inference Attacks (MIA)} and \textit{Data Reconstruction Attacks (DRA)}. For MIA, we adopt the privacy auditing framework of \citet{steinkePrivacyAuditOne2023}, repeating the evaluation at each round for every client; for text generation, we adapt the SPV-MIA of \citet{fuMIALLM2024} to our auditing setting. Reported results correspond to the maximum, over all $T$ rounds, of the average MIA accuracy across $K$ clients. 
Representative per-round privacy--utility trajectories are reported in App.~\ref{app:temporal_privacy}.
For DRA, we consider the strongest white-box threat model, which assumes access to the gradient of a single training sample, and implement DLG \citep{DLG}, iDLG \citep{FCleak_label}, and ROG \citep{yue_rog_attack2023}, with the latter specifically designed for reconstruction from obfuscated gradients. We further include GGL \cite{liGGL2022a}, an adaptive gradient inversion attack that accounts for the gradient transformation induced by the privacy defense. 

\subsection{Numerical Experiments} \label{sec:num_exps}
\begin{wrapfigure}{r}{0.480\textwidth}
  \vspace{-30pt}
  \centering
  \includegraphics[width=0.48\textwidth]{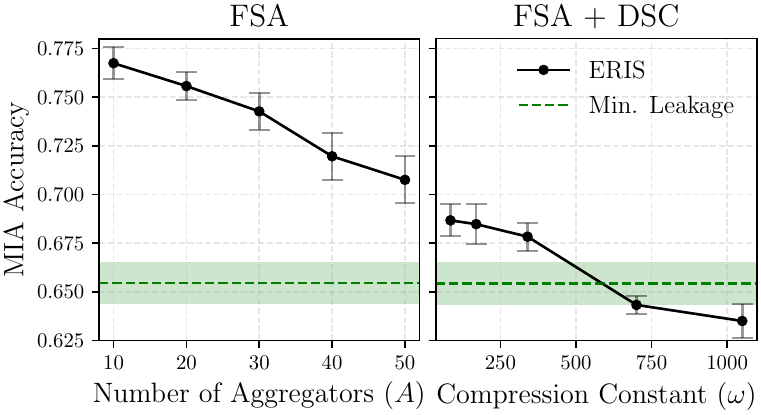}
  \caption{Effect of FSA (\textbf{left}) and DSC (\textbf{right}).
  }
  \label{fig:combined_A_w}
  \vspace{-10pt}
\end{wrapfigure}
\paragraph{Effect of FSA and DSC on privacy.} \label{pg:model_partitioning_compression}$\;$ 
We first isolate the two mechanisms that reduce the adversary's view in \textsc{eris}. Figure~\ref{fig:combined_A_w} (left) evaluates FSA by varying the number of aggregators $A$ on MNIST. Consistent with Theorem~\ref{thm:privacy}, increasing $A$ reduces privacy leakage by limiting each aggregator to a smaller shard of the client update, without affecting model accuracy. The observed trend follows the $1/A$ factor in the mutual-information bound. Figure~\ref{fig:combined_A_w} (right) evaluates DSC with $A\!=\!50$ fixed: stronger compression (higher $\omega$, i.e., lower retention probability $p$) further reduces MIA accuracy toward the idealized minimum-leakage baseline. These results validate the decomposition in Theorem~\ref{thm:privacy}: FSA provides the core privacy amplification through sharding, while DSC further shrinks the observable subset. Appendix~\ref{app:shited_compr} quantifies the corresponding effect of DSC on model utility. For DRA, we find that compression alone is insufficient, especially against ROG (Table~\ref{tab:comm_strategies_quality}), whereas FSA is highly effective: even with $A\!=\!2$ (i.e., half of the update exposed), reconstructions are highly distorted and no longer preserve meaningful features of the original.
\begin{table*}
    \captionsetup{skip=3pt}
    \caption{Mean test performance (ROUGE-1 for CNN/DailyMail, accuracy for others) and MIA accuracy, averaged over varying local sample sizes. 
    Configuration: $\epsilon{=}10$; pruning rate $p\!\in\!\{0.1,0.2,0.3\}$ on IMDB/CNN-DailyMail, and $p\!\in\!\{0.01,0.05,0.1\}$ on others.
    }
    \centering
    \scriptsize
    \renewcommand{\arraystretch}{1.}
    \setlength{\tabcolsep}{2.4pt}
    \begin{tabular}{l|cc|cc|cc|cc}
    \toprule
    & \multicolumn{2}{c|}{\textbf{CNN/DailyMail – GPT-Neo}} 
    & \multicolumn{2}{c|}{\textbf{IMDB – DistilBERT}} 
    & \multicolumn{2}{c|}{\textbf{CIFAR-10 – ResNet9}} 
    & \multicolumn{2}{c}{\textbf{MNIST – LeNet5}}\\
    \textbf{Method} 
    & \textbf{R-1 ($\uparrow$)} & \textbf{MIA Acc. ($\downarrow$)} 
    & \textbf{Acc. ($\uparrow$)} & \textbf{MIA Acc. ($\downarrow$)} 
    & \textbf{Acc. ($\uparrow$)} & \textbf{MIA Acc. ($\downarrow$)} 
    & \textbf{Acc. ($\uparrow$)} & \textbf{MIA Acc. ($\downarrow$)}\\
    \midrule
    FedAvg & 33.22$\pm$0.99 & 97.94$\pm$0.63 
           & 79.60$\pm$0.83 & 68.21$\pm$1.36 
           & 34.86$\pm$0.31 & 68.46$\pm$0.96 
           & 88.91$\pm$0.35 & 65.11$\pm$0.78 \\
    FedAvg ($\epsilon,\delta$)\,-LDP  & 26.00$\pm$0.28 & 51.98$\pm$3.13 
           & 53.97$\pm$0.04 & 50.55$\pm$1.18 
           & 19.00$\pm$0.47 & 63.35$\pm$0.85 
           & 61.03$\pm$1.03 & 57.24$\pm$0.59 \\
    SoteriaFL ($\epsilon,\delta$)  & 25.40$\pm$0.70 & 52.14$\pm$2.97 
           & 54.24$\pm$0.15 & 51.25$\pm$1.19 
           & 17.18$\pm$0.24 & 58.83$\pm$0.56 
           & 57.27$\pm$0.88 & 57.13$\pm$0.56 \\
    PriPrune ($p_1$) & 24.67$\pm$4.64 & 71.35$\pm$2.83 
           & 74.15$\pm$1.00 & 66.36$\pm$1.13 
           & 26.30$\pm$0.39 & 65.67$\pm$0.84 
           & 77.41$\pm$1.52 & 62.21$\pm$1.01 \\
    PriPrune ($p_2$) & 24.67$\pm$4.64 & 71.35$\pm$2.83 
           & 66.30$\pm$2.14 & 63.61$\pm$1.11 
           & 11.24$\pm$0.71 & 56.55$\pm$0.78 
           & 27.36$\pm$1.04 & 52.69$\pm$0.83 \\
    PriPrune ($p_3$) & 24.67$\pm$4.64 & 71.35$\pm$2.83 
           & 60.32$\pm$1.98 & 60.54$\pm$1.03 
           & 10.01$\pm$0.01 & 54.86$\pm$0.77 
           & 17.83$\pm$0.60 & 52.01$\pm$0.80 \\
    Shatter & 31.95$\pm$0.71 & 70.67$\pm$4.03
            & 76.94$\pm$1.40 & 57.41$\pm$2.01 
           & 12.40$\pm$1.85 & 63.00$\pm$2.01 
           & 15.86$\pm$4.82 & 56.23$\pm$1.50 \\
    \textsc{eris} & 33.22$\pm$1.01 & 70.55$\pm$3.95 
           & 79.30$\pm$0.91 & 57.40$\pm$1.99 
           & 34.84$\pm$0.64 & 63.02$\pm$1.98 
           & 88.87$\pm$0.67 & 56.14$\pm$1.47 \\
    \textsc{eris} (+\textsc{dsc}) & 32.83$\pm$0.78 & 69.55$\pm$3.94 
           & 79.07$\pm$0.80 & 56.31$\pm$0.81 
           & 34.68$\pm$0.48 & 60.48$\pm$0.91 
           & 89.00$\pm$0.23 & 55.97$\pm$0.77 \\
    \midrule
    Min.\ Leakage & 33.23$\pm$0.99 & 60.53$\pm$4.83 
           & 79.68$\pm$0.36 & 55.58$\pm$0.76 
           & 34.92$\pm$0.29 & 58.85$\pm$0.93 
           & 88.90$\pm$0.40 & 55.22$\pm$0.64 \\
    \bottomrule
    \end{tabular}
    \label{tab:overall_acc_mia_all}
    \vspace{-5pt}
\end{table*}
\vspace{-5pt}
\begin{figure*}
  \centering
    \captionsetup{skip=3pt} 
\includegraphics[width=0.99\textwidth]{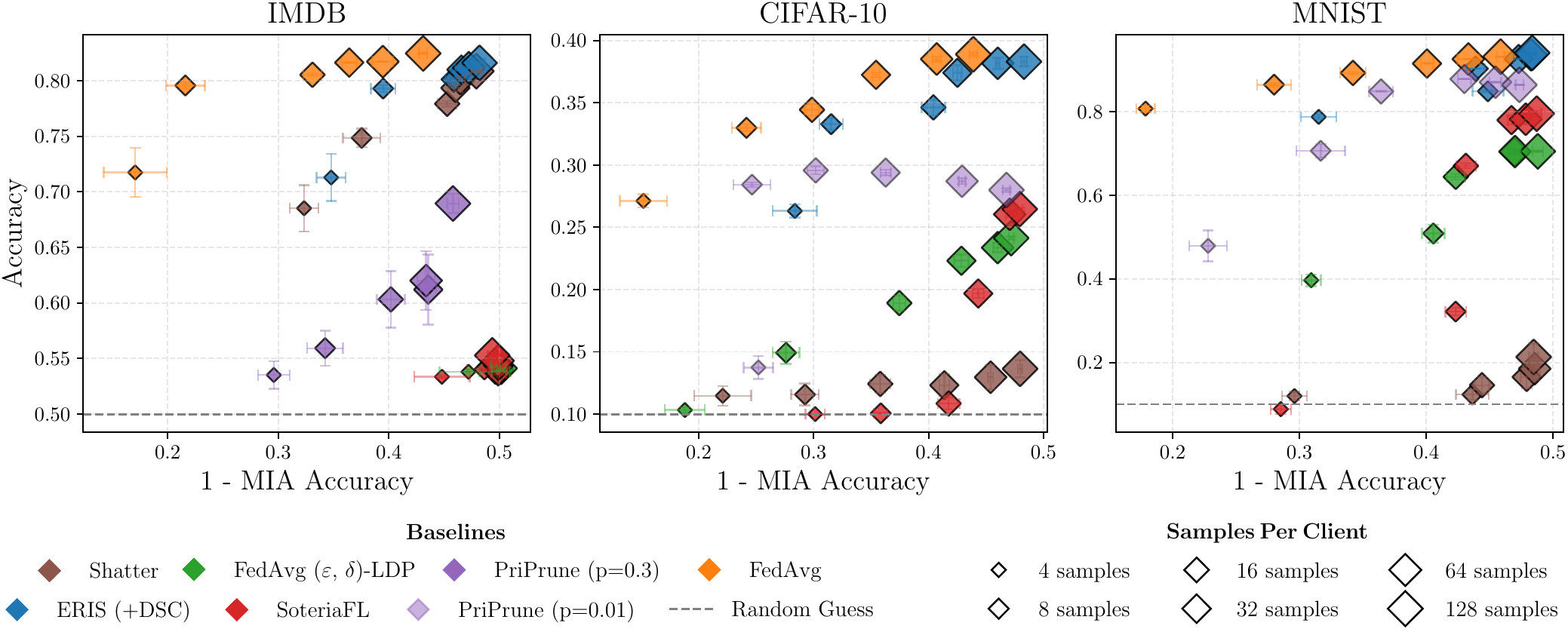}
  \caption{Comparison of test accuracy and (1- MIA) accuracy across varying model capacities (one per dataset) and client-side overfitting levels, controlled via the number of training samples per client.
  } 
  \label{fig:unbiased_acc_priv}
\end{figure*}

\paragraph{Balancing Utility and Privacy.}\label{pg:balancing_u_p} $\;$ 
To evaluate the utility–privacy trade-off, we benchmark \textsc{eris} with DSC against SOTA baselines across settings that influence memorization and overfitting. First, we vary model capacity, a key factor in memorization, spanning from large-scale 1.3B-parameter model on CNN/DailyMail to a lightweight 62K-parameter model on MNIST. Second, we control overfitting by varying the number of local training samples, from 4 to 128. Figure~\ref{fig:unbiased_acc_priv} shows that \textsc{eris} consistently maintains high utility, on par with non-private FedAvg, while significantly reducing privacy leakage—approaching the idealized \textit{Min. Leakage} scenario. In contrast, privacy-preserving methods such as FedAvg-LDP, PriPrune, and SoteriaFL reduce leakage only at the cost of degraded performance. This confirms prior findings~\citep{liLARGELANGUAGEMODELS2022} that DP can substantially impair utility, particularly for large models, where low leakage may largely reflect the model's inability to learn the task. Notably, while Shatter's partial gradient exchange offers privacy protection comparable to or weaker than \textsc{eris}, its fragmented collaboration substantially slows convergence, especially when models are trained from scratch. Table~\ref{tab:overall_acc_mia_all} summarizes mean utility and MIA accuracy over client training-sample sizes, confirming that \textsc{eris} achieves the strongest utility--privacy trade-off among all baselines. Appendix~\ref{app:nonIID_tradeoff} reports the same experiments under non-IID settings, with consistent conclusions.

\paragraph{Pareto Analysis under Varying Privacy Constraints.}\label{pg:pareto}$\;$ 
To further investigate the utility--privacy trade-off, we evaluate each method under varying strengths of its privacy mechanism: for DP-based methods, we vary $\epsilon$ and clipping thresholds; for PriPrune, the pruning rate; and for \textsc{eris} and Shatter, we add LDP on top of their 
masking mechanisms. Full settings 
are provided in Appendix~\ref{app:pareto}. Figure~\ref{fig:pareto_16} plots accuracy against privacy ($1-$MIA accuracy) on CIFAR-10 with 16 training samples per client. 
The Pareto front contains solutions for which no method improves utility without increasing leakage, or reduces leakage without sacrificing utility. 
\textsc{eris} contributes a majority of the Pareto-optimal points, confirming its ability to balance privacy and utility more effectively than the baselines.

\begin{figure}[t]
    \centering

    \begin{minipage}[t]{0.35\textwidth}
        \centering
        \captionsetup{type=figure, skip=4pt}
        \includegraphics[width=\textwidth]{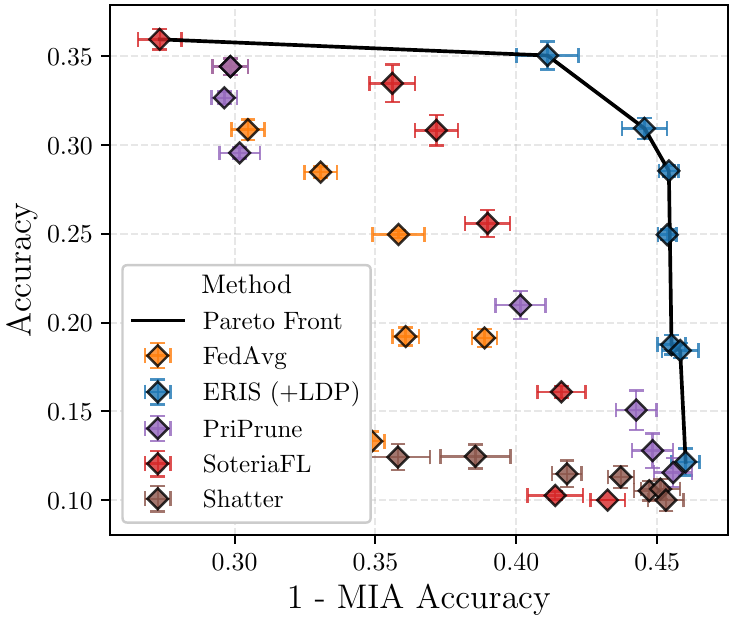}
        \caption{Pareto analysis of utility and privacy on CIFAR-10.}
        \label{fig:pareto_16}
    \end{minipage}
    \hfill
    \begin{minipage}[t]{0.58\textwidth}
        \centering
        \captionsetup{type=table, skip=4pt}
        \caption{Scalability comparison: per-client upload and minimum distribution time per round.}
        \label{tab:comm_efficiency}
        \small
        \renewcommand{\arraystretch}{1.}
        \setlength{\tabcolsep}{1.3pt}
    
        \resizebox{\textwidth}{!}{%
        \begin{tabular}{lcccc}
            \toprule
            \multirow{2}{*}{\textbf{Method}} & 
            \multicolumn{2}{c}{\textbf{CNN/DailyMail}} & 
            \multicolumn{2}{c}{\textbf{CIFAR-10}} \\
            \cmidrule(lr){2-3} \cmidrule(lr){4-5}
            & \textbf{Upload} & \textbf{Dist. Time} & \textbf{Upload} & \textbf{Dist. Time} \\
            \midrule
            FedAvg (-LDP)    & 5.2GB (100\%) & 5200s & 6.6MB (100\%) & 33s \\
            Shatter          & 5.2GB (100\%) & 780s  & 6.6MB (100\%) & 1.32s \\
            PriPrune ($p_1$) & 4.68GB (90\%) & 4940s & 6.53MB (99\%) & 32.84s \\
            PriPrune ($p_2$) & 4.16GB (80\%) & 4680s & 6.27MB (95\%) & 32.17s \\
            PriPrune ($p_3$) & 3.64GB (70\%) & 4420s & 5.9MB (90\%)  & 31.35s \\
            SoteriaFL        & 0.26GB (5\%)  & 2730s  & 0.33MB (5\%)  & 17.32s \\
            \textsc{eris} & 4.68GB (90\%) & 468s & 6.47MB (98\%) & 0.65s \\
            \textsc{eris} (+\textsc{dsc}) & 46.8MB (1\%) & 236s & 0.04MB (0.6\%) & 0.33s \\
            \bottomrule
        \end{tabular}%
        }
    \end{minipage}
    \vspace{-8pt}
\end{figure}
\paragraph{Scalability.}\label{pg:comm_eff}$\;$
Table~\ref{tab:comm_efficiency} evaluates scalability in terms of per-client upload size and minimum distribution time per round, assuming 20MB/s bandwidth and using the same experimental setting as Table~\ref{tab:overall_acc_mia_all}. The results show that \textsc{eris} improves scalability through two complementary mechanisms. First, FSA removes the central aggregation bottleneck by distributing communication across aggregators. Even without DSC, \textsc{eris} substantially reduces distribution time, from 5200s to 468s on CNN/DailyMail and from 33s to 0.65s on CIFAR-10, with larger gains as the number of aggregators increases (CNN/DailyMail: $A\!=\!10$; CIFAR-10: $A\!=\!50$). Second, DSC further reduces the transmitted payload, yielding the full \textsc{eris} gains: on CNN/DailyMail, upload size drops from 5.2GB in FedAvg to 46.8MB (1\%) and distribution time to 236s; on CIFAR-10, communication drops to 0.6\% of the full update and distribution time to 0.33s. Unlike prior decentralized methods, \textsc{eris} preserves the full collaborative update after reassembly, retaining the utility of centralized aggregation while scaling to billion-parameter models. A full comparison across additional FL and decentralized FL methods, including per-client upload/download, total 
communication, and distribution time for each dataset/model pair, is reported in Appendix~\ref{app:scalability}. 

\section{Discussion} \label{sec:discussion} 
\paragraph{Limitations.} \label{sec:limitations}
While \textsc{eris} achieves strong empirical and theoretical performance, it also introduces system-level trade-offs. 
First, distributing aggregation shifts coordination to clients, making reliability dependent on aggregator availability, client--aggregator connectivity, and client-side computational capacity. 
To assess robustness to aggregator and link failures, Appendix~\ref{app:aggregator_dropout} shows that accuracy remains nearly unchanged up to 70\% aggregator dropout and 50\% link failures, with degradation mainly due to slower convergence. 
In addition, to support heterogeneous deployments, FSA can allocate larger shards to stronger aggregators and smaller shards to weaker or bandwidth-constrained nodes, at the cost of worst-case leakage being determined by the largest shard rather than $n/A$.  
Our scalability analysis focuses primarily on communication costs, since payload transmission increasingly dominates end-to-end latency as model sizes and client populations grow, and abstracts from connection-establishment and synchronization overheads. A full end-to-end systems evaluation is left to future work. Second, \textsc{eris}'s privacy amplification degrades under aggregator collusion: leakage increases with coalition size, as analyzed in Corollary~\ref{cor:colluding_privacy} and Figure~\ref{fig:collusion}, while remaining substantially lower than centralized FL in the evaluated settings. Importantly, this can be compensated by increasing $A$ proportionally to the anticipated coalition size.
We also remark that \textsc{eris}'s information-theoretic guarantee is distinct from record-level DP: it bounds the information revealed about a client's dataset through the observed shards, but does not provide an $(\epsilon,\delta)$-DP guarantee for individual records. DP can be integrated with FSA when such guarantees are required.

\paragraph{Benefits.}
As alluded to in Section~\ref{sec:covergence_theory}, FSA supports the execution of any centralized FL algorithm beyond FedAvg, such as FedAdam~\cite{reddi2021adaptive}, FedYogi~\cite{reddi2021adaptive}, FedNova~\cite{wang2020tackling}, \etc since it exhibits no deviation from the centralized aggregation trajectory.
Furthermore, FSA acts as a modular aggregation layer: transformations such as compression \cite{liCANITAFasterRates2021}, 
DP \cite{sunLDPFL2021, soteriafl2022}, or robust aggregations \cite{blanchardkrum2017b, fenoglioFBP2024f} can be readily integrated while preserving FSA's benefits. Similarly, secure aggregation~\cite{SMC, SMC2} can be executed independently within each shard, limiting colluding aggregators to shard-level aggregates rather than individual client shards, without altering the resulting global update but at the cost of additional cryptographic overhead. 
We view this flexibility and modularity as key advantages of FSA, enabling \textsc{eris} to serve as a general foundation for scalable and privacy-preserving FL.


\section{Conclusion} \label{sec:conclusions}
We introduced \textsc{eris}, an FL framework that jointly addresses privacy, scalability, and utility through Federated Shard Aggregation (FSA). By sharding client updates across multiple aggregators, FSA removes the central server bottleneck, limits the information visible to any single observer, and preserves the centralized FL update after reassembly. We further instantiate Distributed Shifted Compression (DSC) on top of FSA to reduce transmitted and exposed coordinates while preserving favorable convergence behavior. We provide convergence bounds and information-theoretic privacy guarantees, and validate them through extensive experiments across diverse datasets and model scales. Our results show that privacy and scalability in FL can be enhanced without sacrificing utility.


\begin{ack}
This study was funded by the Swiss National Science Foundation, and the European Union’s Horizon Europe research initiative through the following projects: SmartCHANGE (Project No. 101080965), XAI-PAC (Grant No. Z00P2\_216405), and starting grant ``Multiresolution methods for unstructured data'' (TMSGI2\_211684). AD was supported in part by the Swiss National Science Foundation's Postdoc.Mobility fellowship (Grant No. P500-2\_239026). The authors declare no competing interests.
\end{ack}

\newpage
\bibliographystyle{unsrtnat} 
\bibliography{references} 

\newpage
\appendix
\onecolumn

\appendixpart
\parttoc

\newpage
\section*{Appendix}
\noindent This appendix is organised as follows. Section~\ref{app:assumptions} restates the main assumptions used in our theoretical analysis. Section~\ref{sec:convergence-analysis} proves that FSA maintains the same convergence behavior as centralized FedAvg in the absence of compression. Section~\ref{sec:utility_comm_eris} presents the convergence and communication analysis for \textsc{eris} with DSC. Section~\ref{app:privacy_proof} provides the privacy guarantees and includes an extension of Theorem~\ref{thm:privacy} to colluding aggregators. Section~\ref{app:experimental_setup} details our experimental setup, including models, hyperparameters, privacy attacks, datasets, licenses, and hardware for full reproducibility. Finally, Sections~\ref{app:weight_distribution}--\ref{app:temporal_privacy} report additional experimental results supporting our claims, including evaluations of scalability, compression, data reconstruction attacks, privacy--utility trade-offs under varying heterogeneity conditions (IID and non-IID) and both biased and unbiased gradient estimators, and the temporal evolution of privacy leakage throughout training.

\section{Assumptions} \label{app:assumptions}

For clarity and completeness of the Appendix, we restate the core assumptions used in the main theorems—Theorem~\ref{thm:utility_eris_main} and Theorem~\ref{thm:privacy}. 
These include smoothness and unbiased local estimator conditions commonly adopted in the FL literature. 

\begin{assumption}[Smoothness]\label{assumption:smoothness}
There exists some $L \ge 0$, such that for all local functions $f_{i,j}$ (indexed by $i \in [n]$ and $j \in [m]$), we have
\begin{equation}
    \|\nabla f_{i,j}(x_1) - \nabla f_{i,j}(x_2)\| \;\le\; L \,\|x_1 - x_2\|,
    \quad
    \forall\, x_1, x_2 \in \mathbb{R}^d, \label{eq:assump1_part1}
\end{equation}
or equivalently expressed with the following general bound:
\begin{equation}
    f_{i,j}(x_1) \le f_{i,j}(x_2) + \langle \nabla f_{i,j}(x_2), x_1 - x_2\rangle + \frac{L}{2}\|x_1-x_2\|^2.
    \label{eq:assump1_part2}
\end{equation}
\end{assumption}

\begin{assumption}[Unbiased local estimator] \label{assumption:unbiased_estimator}
The gradient estimator $\tilde{\mathbf{g}}_k^t$ is unbiased $\mathbb{E}_t[\tilde{\mathbf{g}}_k^t]=\nabla f_k(\mathbf{x^t)}$ for $k \in \mathcal{K}$, where $\mathbb{E}_t$ takes the expectation conditioned on all history before round $t$. Moreover, there exist constants $\Gamma_1$, $\Gamma_2$, $\Gamma_3$, $\Gamma_4$, and $\theta$ such that: 
\begin{equation}
    \mathbb{E}_t \bigl[\frac{1}{K} \sum_{k=1}^K ||\tilde{\mathbf{g}}_k^t -\nabla f_k(\mathbf{x^t)} ||^2\bigr] \leq \Gamma_1 \Delta^t + \Gamma_2
\tag{11a} \label{eq:assump4a_app} \end{equation}
\begin{equation}
    \mathbb{E}_t[\Delta^{t+1}] \leq (1-\theta)\Delta^t + \Gamma_3||\nabla f(\mathbf{x}^t)||^2+ \Gamma_4 \mathbb{E}_t[||\mathbf{x}^{t+1}-\mathbf{x}^t||^2]
\tag{11b} \label{eq:assump4b} \end{equation}
\begin{remark}
    The parameters $\Gamma_1$ and $\Gamma_2$ capture the variance of the gradient estimators, e.g., $\Gamma_1 = \Gamma_2 = 0$ if the client computes local full gradient $\tilde{\mathbf{g}}^t_i = \nabla f_i(\mathbf{x^t)}$, and $\Gamma_1 \neq 0$
    and $\Gamma_2 = 0$ if the client uses variance-reduced gradient estimators such as SVRG/SAGA. 
\end{remark}
\end{assumption}

\section{Convergence Equivalence of FSA} \label{sec:convergence-analysis}
This section shows that Federated Shard Aggregation (FSA)---Algorithm~\ref{alg:eris} instantiated without DSC, i.e., with the identity compressor $\mathcal{C}_k^t\!=\!\mathrm{Id}$ and reference vectors fixed to zero---produces exactly the same global iterate sequence as the standard single-server algorithm (e.g., \emph{FedAvg}). Consequently, every convergence guarantee proved for FedAvg carries over verbatim. The proof is algebraic and does not rely on additional smoothness or convexity assumptions beyond those stated in Appendix~\ref{app:assumptions}.

\paragraph{Notation.}  Recall that client $k$ holds $S_k$ data points and that $S:=\sum_{k=1}^K S_k$.  Let $\tilde{\mathbf{g}}_k^t$ denote the (possibly stochastic) gradient that client~$k$ transmits at communication round~$t$ and write $\tilde{\mathbf{g}}^t=\frac{1}{S}\sum_{k=1}^K S_k \; \tilde{\mathbf{g}}_k^t$ for the sample--weighted mean gradient.

\begin{theorem}[Convergence equivalence of FSA]\label{thm:convergence_equivalence} Run Algorithm~\ref{alg:eris} with $A\ge1$ aggregators, $\mathcal{C}_k^t=\mathrm{Id}$, and $\mathbf{s}_k^t=\mathbf 0$ for all $k,t$. Let ${\mathbf{x}}^t$ with ${t\ge0}$ be the resulting iterates and let ${\tilde{\mathbf{x}}}^t$ be the iterates obtained by FedAvg ($A=1$) using the same initialization, learning rates ${\lambda_t}$, and client gradients ${\tilde{\mathbf{g}}}_k^t$. Then for every round $t\ge0$
\refstepcounter{equation}
\begin{equation}
    \mathbf{x}^t=\tilde{\mathbf{x}}^t.
\end{equation}
Hence all convergence bounds that hold for FedAvg under Assumptions~\ref{assumption:smoothness} and \ref{assumption:unbiased_estimator} (with $\omega=0$) apply unchanged to \textsc{eris} with FSA only.
\end{theorem}
\begin{proof}[Sketch]
Partition each client gradient into $A$ disjoint coordinate shards using the categorical masks $\{\mathbf m_{(a)}^t\}_{a=1}^A$ introduced in Section~\ref{sec:eris_pipeline}: $\tilde{\mathbf{g}}_{k,(a)}^t=\tilde{\mathbf{g}}_k^t\odot\mathbf{m}^t_{(a)}$. Because the masks are disjoint and sum to the all--ones vector, the original gradient decomposes exactly as $\tilde{\mathbf{g}}_k^t=\sum_{a=1}^A\tilde{\mathbf g}_{k,(a)}^t$. Aggregator $a$ forms the weighted average of its shard
\begin{equation}
    \bar{\mathbf g}_{(a)}^t \;:=\; \frac{1}{S} \sum_{k=1}^K S_k\,\tilde{\mathbf g}_{k,(a)}^t.
\end{equation}
Summing over all aggregators and swapping summation order yields
\begin{equation}
    \sum_{a=1}^A \bar{\mathbf g}_{(a)}^t
      = \frac{1}{S} \sum_{k=1}^K S_k \sum_{a=1}^A \tilde{\mathbf g}_{k,(a)}^t
      = \frac{1}{S} \sum_{k=1}^K S_k \; \tilde{\mathbf g}_k^t
      = \tilde{\mathbf g}^t.
\end{equation}

FSA therefore updates the global model via $\mathbf{x}^{t+1}=\mathbf{x}^{t} -\lambda_t\sum_{a=1}^A\bar{\mathbf{g}}_{(a)}^t = \mathbf{x}^{t} - \lambda_t \tilde{\mathbf g}^t$, which is exactly the FedAvg rule. By induction on~$t$ the iterates coincide.
\end{proof}

\begin{remark}
The identity above is purely algebraic, hence it remains valid when clients perform multiple local SGD steps, when the data are non IID, or when the global objective is nonconvex (e.g.,~\citep{mcmahan_FL, liConvergenceFedAvgNonIID2020, liConvergenceTheoryFederated2022}). The key insight is that splitting the gradient vector dimension-wise introduces no additional approximation error; the final aggregated gradient is mathematically identical to that obtained by a single server aggregating all client gradients in one place. This ensures that the convergence behavior of FSA matches that of traditional federated learning approaches, while its sole effect is to distribute network load.
\end{remark}

\section{Convergence and Utility for \textsc{eris} with DSC} \label{sec:utility_comm_eris}
In this section, we present the proof of Theorem \ref{thm:utility_eris_main} (Utility of \textsc{eris} with DSC), modifying the general proof strategy of \citep{soteriafl2022} to accommodate our decentralized setting, model partitioning, and the absence of differential privacy. 

\subsection{Proof of Theorem \ref{thm:utility_eris_main}}
\begin{proof}
    Let $\mathbb{E}_t$ denote the expectation conditioned on the full history up to round $t$. By invoking Theorem~\ref{thm:convergence_equivalence}, we simplify the analysis by omitting model partitioning and treating $\mathbf{v}^t$ as the aggregated update. Thus, the update rule becomes $\mathbf{x}^{t+1} = \mathbf{x}^t - \lambda_t \mathbf{v}^t$. We now apply this rule within the smoothness inequality~\eqref{eq:assump1_part2}:

    \begin{equation}
    \mathbb{E}_t[f(\mathbf{x}^{t+1})] \le \mathbb{E}_t \left[f(\mathbf{x}^{t}) - \lambda_t \langle \nabla f(\mathbf{x}^{t}), \mathbf{v}^t\rangle + \frac{\lambda^2L}{2}\|\mathbf{v}^t\|^2\right] 
    \label{eq:F0}
    \end{equation}

    First, to verify the unbiased nature of $\mathbf{v}^t$, we consider:
    \begin{align}
        \mathbb{E}_t [\mathbf{v}^t] &= \mathbb{E}_t\left[\mathbf{s}^t+\frac{1}{K} \sum_{k=1}^K \mathbf{v}_k^t\right] 
         \nonumber \\ 
        &=\mathbb{E}_t \left[\frac{1}{K} \sum_{k=1}^K \mathbf{s}_k^t+\frac{1}{K} \sum_{k=1}^K \mathcal{C}_k^t(\tilde{\mathbf{g}}_k^t-\mathbf{s}_k^t) \right]
        \nonumber \\ 
        &\overset{(\ref{eq:compressor_operator})}{=} \mathbb{E}_t \left[\frac{1}{K}\sum_{k=1}^K \tilde{\mathbf{g}}_k^t\right] = \frac{1}{K} \sum_{k=1}^K \mathbb{E}_t [\tilde{\mathbf{g}}_k^t] \overset{(a)}{=} \frac{1}{K}\sum_{k=1}^K \nabla f_k(\mathbf{x}^t) = \nabla f(\mathbf{x}^t)
        \label{eq:F4}
    \end{align}
    where (a) due to Assumption \ref{assumption:unbiased_estimator}, which states that each $\tilde{\mathbf{g}}_k^t$ is an unbiased estimator of $\nabla f_k(\mathbf{x^t)}$ (i.e., $\mathbb{E}_t[\tilde{\mathbf{g}}_k^t]=\nabla f_k(\mathbf{x^t)}$).
    
    Substituting Equation \eqref{eq:F4} into \eqref{eq:F0}, we obtain:
    \begin{equation}
        \mathbb{E}_t[f(\mathbf{x}^{t+1})] \le \mathbb{E}_t \left[f(\mathbf{x}^{t}) - \lambda_t ||\nabla f(\mathbf{x}^{t})||^2 + \frac{\lambda^2L}{2}\|\mathbf{v}^t\|^2\right] 
        \label{eq:F6}
    \end{equation}
    We further bound the term $\mathbb{E}_t[||\mathbf{v}^t||^2]$ in Lemma \ref{lemma:6}, whose proof is available in the Appendix \ref{proof:lemma6}.

    \begin{lemma}
    Consider that $\mathbf{v}^t$ is constructed according to Algorithm \ref{alg:eris}, it holds that

    \begin{equation}
        \mathbb{E}_t[\|\mathbf{v}^t\|^2] \leq \mathbb{E}_t \left[
        \frac{(1 + \omega)}{K^2} \sum_{k=1}^K \| \tilde{\mathbf{g}}_k^t - \nabla f_k(\mathbf{x}^t) \|^2
        \right] + \frac{\omega}{K^2} \sum_{k=1}^{K} \|\nabla f_k(\mathbf{x}^t) - \mathbf{s}_k^t\|^2 + \|\nabla f(\mathbf{x}^t)\|^2.
        \label{eq:F7}
    \end{equation} 
    \label{lemma:6}
    \end{lemma}
    To further our analysis, we now derive upper bounds for the first two terms on the right-hand side of Equation \eqref{eq:F7}. The first term can be controlled using Equation \eqref{eq:assump4a_app} from Assumption \ref{assumption:unbiased_estimator}, yielding the bound $\Gamma_1 \Delta^t + \Gamma_2$. Next, we establish that the second term decreases over time, as formalized in the following lemma (proof available in Appendix \ref{proof:lemma7}).

    \begin{lemma}
    Let Assumption \ref{assumption:smoothness} hold, and let the shift $\mathbf{s}_k^{t+1}$ be updated according to Algorithm \ref{alg:eris}. Then, for $\gamma_t = \sqrt{\frac{1+2\omega}{2(1+\omega)^3}}$, we have:
    \begin{align}
        \mathbb{E}_t \left[
        \frac{1}{K} \sum_{k=1}^{K} \|\nabla f_k(\mathbf{x}^{t+1}) - \mathbf{s}_k^{t+1} \|^2
        \right] 
        &\leq \mathbb{E}_t \left[
        \left(1 - \frac{1}{2(1+\omega)}\right) \frac{1}{K} \sum_{k=1}^{K} \|\nabla f_k(\mathbf{x}^t) - \mathbf{s}_k^t\|^2 \right. \notag \\
        &\quad + \frac{1}{(1+\omega)K} \sum_{k=1}^{K} \|\tilde{\mathbf{g}}_k^t - \nabla f_k(\mathbf{x}^t) \|^2 \notag \\
        &\quad \left. + 2(1+\omega)L^2 \|\mathbf{x}^{t+1} - \mathbf{x}^t\|^2
        \right]. \label{eq:lemma7}
    \end{align}
    \label{lemma:7}
    \end{lemma}

    For clarity, we introduce the notation: $\mathcal{S}^t := \frac{1}{K} \sum_{k=1}^K ||\nabla f_k(\mathbf{x}^t ) - \mathbf{s}^t_k||^2$. For some $\alpha \geq 0, \beta \geq 0$, we now define a potential function to analyze the convergence behavior:
    \begin{equation}
        \Phi_t :=f(\mathbf{x}^t)-f^*+\alpha L \Delta^t + \frac{\beta}{L}\mathcal{S}^t, \label{eq:pot_function}
    \end{equation}
    Using Lemmas \ref{lemma:6} and \ref{lemma:7}, we demonstrate in Lemma \ref{lemma:8} that this potential function decreases in expectation at each iteration (proof provided in Appendix \ref{proof:lemma8}).

    \begin{lemma}
    Under Assumptions \ref{assumption:smoothness} and \ref{assumption:unbiased_estimator}, if the learning rate is chosen as
    \begin{equation}
        \lambda_t \triangleq \lambda \leq \min \left(
        \frac{1}{(1 + 2\alpha \Gamma_4 + 4\beta (1 + \omega) + 2\alpha \Gamma_3/\lambda^2)L},
        \frac{\sqrt{\beta K}}{\sqrt{1 + 2\alpha \Gamma_4 + 4\beta (1 + \omega)}(1 + \omega)L}
        \right),
    \end{equation}
    where $\alpha = \frac{3\beta \Gamma_1}{2(1+\omega)\theta L^2}$ for any $\beta > 0$, and the shift step size $\gamma_t$ is defined as in Lemma \ref{lemma:7}, it follows that for every round $t \geq 0$, the expected potential function satisfies the following bound:    
    \begin{equation}
        \mathbb{E}_t[\Phi_{t+1}] \leq \Phi_t - \frac{\lambda_t}{2} \|\nabla f(\mathbf{x}^t)\|^2 + \frac{3\beta \Gamma_2}{2(1 + \omega)L}.
        \label{eq:lemma8}
    \end{equation}

    \begin{remark}
        Since the last term is generally a small constant during time (see Assumption \ref{assumption:unbiased_estimator}) and $\frac{\lambda_t}{2} \|\nabla f(\mathbf{x}^t)\|^2$ is positive, Equation \eqref{eq:lemma8} indicates that the potential decrease over the time.
    \end{remark}
    \label{lemma:8}
    \end{lemma}

    With Lemma \ref{lemma:8} established, we now proceed to the proof of Theorem \ref{thm:utility_eris_main}, which characterizes the utility and the number of communication rounds required for \textsc{eris} to reach a given accuracy level.
    We begin by summing Equation \eqref{eq:lemma8} from rounds $t = 0$ to $T - 1$:
    \begin{align}
         &\sum_{t=0}^{T-1} \mathbb{E}[\Phi_{t+1}] 
         \leq \sum_{t=0}^{T-1} \mathbb{E} [\Phi_t] -\sum_{t=0}^{T-1} \left( \frac{\lambda_t}{2}\|\nabla f(\mathbf{x}^t)\|^2 
         + \frac{3\beta \Gamma_2}{2(1 + \omega)L} \right)  \nonumber \\
        &\mathbb{E}[\Phi_{T}] - \mathbb{E}[\Phi_{0}] \leq -\sum_{t=0}^{T-1}   \frac{\lambda_t}{2}\|\nabla f(\mathbf{x}^t)\|^2 + \frac{3\beta \Gamma_2 T}{2(1 + \omega)L} \nonumber
    \end{align}    
    Since by construction, we typically have $\mathbb{E}[\Phi_t] \geq 0$, by choosing the learning rate $\lambda_t$ as in Lemma \ref{lemma:8}, we finally obtain
    \begin{equation}
        \frac{1}{T}\sum_{t=0}^{T-1} \|\nabla f(\mathbf{x}^t)\|^2 \leq \frac{2\Phi_0}{\lambda T} + \frac{3\beta \Gamma_2}{(1 + \omega)L \lambda},
    \end{equation}
    which proves that per $T \rightarrow \infty$
    \begin{equation}
        \lim_{T \rightarrow \infty} \frac{1}{T} \sum_{t=0}^{T-1} \|\nabla f(\mathbf{x}^t)\|^2 \leq \frac{3\beta \Gamma_2}{(1 + \omega)L \lambda}.
    \end{equation}

    While to achieve a predefined utility level $\epsilon \geq \frac{1}{T}\sum_{t=0}^{T-1} ||\nabla f(\mathbf{x}^t)||^2$, the total 
    rounds $T$ must satisfy:
    \begin{equation}
    \setlength{\abovedisplayskip}{4pt} 
    \setlength{\belowdisplayskip}{3pt} 
        T 
        \;\;\ge\;\;
        \frac{2\,\Phi_0}{\lambda\,\Bigl(\epsilon - \tfrac{3\beta\,\Gamma_2}{(1 + \omega)\,L\,\lambda}\Bigr)}.
    \end{equation}
    If $\epsilon$ is strictly less than the residual $\tfrac{3\beta\,\Gamma_2}{(1+\omega)\,L\,\lambda}$, no finite $T$ can achieve the utility $\epsilon$ in an average sense, therefore conditions on the adopted estimator need to be changed.

    \begin{remark}
        \textsc{eris} with DSC utility is asymptotically governed by the variance in $\tilde{\mathbf{g}}$, which directly depends on the used estimator (e.g., $\Gamma_2=0$ with SVRG/SAGA) or on the dimension of the batch size (e.g., $\Gamma_2=0$ with local full gradients). Compared to SoteriaFL \citep{soteriafl2022}, the upper bound of \textsc{eris} with DSC utility does not have a component growing with $T$, limiting the convergence.   
    \end{remark}
    
\end{proof}

\begin{corollary}[Utility of \textsc{eris}--SGD with DSC]
\label{cor:eris_sgd_no_dp}
Consider the FL setting in~\eqref{eq:fedavg_opt} with $K$ clients, where each client $k \in \mathcal{K}$ holds a local dataset $D_k = \{d_{k,s}\}_{s=1}^{S_k}$, and let Assumptions~\ref{assumption:smoothness} and~\ref{assumption:unbiased_estimator} hold.
Assume that, at each round $t$, client $k$ uses a mini-batch SGD estimator
\[
    \tilde{\mathbf{g}}_k^t
    \;=\;
    \frac{1}{b_k}\sum_{s \in \mathcal{B}_k^t} \nabla f_{k,s}(\mathbf{x}^t),
\]
where $\mathcal{B}_k^t \subseteq \{1,\dots,S_k\}$ is a uniformly sampled mini-batch of size $b_k$, and that stochastic gradients are uniformly bounded as
$\|\nabla f_{k,s}(\mathbf{x})\| \le G$ for all $k,s,\mathbf{x}$.

For notational simplicity, suppose that all clients share the same dataset size and batch size, i.e., $S_k \equiv m$ and $b_k \equiv b$ for all $k$.
Then the constants in Assumption~\ref{assumption:unbiased_estimator} are
\[
    \Gamma_1 = \Gamma_3 = \Gamma_4 = 0, 
    \qquad
    \Gamma_2 = \frac{(m-b)G^2}{m b},
    \qquad
    \theta = 1.
\]

Let the compression operators $\mathcal{C}_k^t$ satisfy Definition~\ref{def:compression_operator} with parameter $\omega\ge 0$, and run \textsc{eris} with constant learning rate $\lambda_t \equiv \lambda$ and shift stepsize $\gamma_t \equiv \gamma$ as in Theorem~\ref{thm:utility_eris_main}, with
\begin{equation}
    \lambda \;\le\; \frac{1}{\bigl(1 + 4\beta(1+\omega)\bigr)L},
    \qquad
    \gamma \;=\; \sqrt{\frac{1 + 2\omega}{2(1+\omega)^3}},
    \label{eq:eris_sgd_stepsizes}
\end{equation}
for some fixed $\beta > 0$.
Then \textsc{eris} with DSC--SGD satisfies
\begin{equation}
    \frac{1}{T}\sum_{t=0}^{T-1} \mathbb{E}\bigl\|\nabla f(\mathbf{x}^t)\bigr\|^2
    \;\le\;
    \frac{2\Phi_0}{\lambda T}
    \;+\;
    \frac{3\beta (m-b) G^2}{(1+\omega)L\,\lambda\,m b},
    \label{eq:eris_sgd_bound_exact}
\end{equation}
where
\[
    \Phi_0
    :=
    f(\mathbf{x}^0)-f^*
    + \alpha L \Delta^0
    + \frac{\beta}{K L}\sum_{k=1}^{K}
    \bigl\|\nabla f_k(\mathbf{x}^0) - \mathbf{s}_k^0\bigr\|^2,
\]
with $\alpha = \frac{3\beta \Gamma_1}{2(1+\omega)L^2\theta}$, $\Delta^t := \frac{1}{K}\sum_{k=1}^K \|\mathbf{s}_k^t - \nabla f_k(\mathbf{x}^t)\|^2$, and the stepsize $\lambda$ is constrained by Theorem~\ref{thm:utility_eris_main} as
\[
    \lambda \;\le\;
    \lambda_{\max} := \min\{\lambda_1,\lambda_2\},\quad
    \lambda_1 = \frac{\sqrt{\beta K}}{\sqrt{1+4\beta(1+\omega)}(1+\omega)L},\quad
    \lambda_2 = \frac{1}{(1+4\beta(1+\omega))L}.
\]

For any fixed $K,\omega$ we can choose $\beta > 0$ sufficiently small so that $\lambda_1 \le \lambda_2$, and hence $\lambda_{\max} = \lambda_1$.
We then set $\lambda = \lambda_1$ and let $T \to \infty$, so that the term $\frac{2\Phi_0}{\lambda T}$ vanishes.
Substituting $\lambda_1$ into~\eqref{eq:eris_sgd_bound_exact} yields
\begin{align}
    \frac{1}{T}\sum_{t=0}^{T-1} \mathbb{E}\bigl\|\nabla f(\mathbf{x}^t)\bigr\|^2
    &\;\le\; 
    \frac{3\beta (m-b) G^2}{(1+\omega)L\,m b} \cdot
    \frac{\sqrt{1+4\beta(1+\omega)}(1+\omega)L}{\sqrt{\beta K}}, \\
    & \;\le\;
    \frac{3\sqrt{\beta} (m-b) G^2}{m b} \,
    \frac{\sqrt{1+4\beta(1+\omega)}}{\sqrt{K}}.
    \label{eq:bound_cor_1}
\end{align}
Since $\beta$ is a fixed constant and $\sqrt{1+4\beta(1+\omega)} = \Theta(\sqrt{1+\omega})$, we obtain the asymptotic bound
\begin{equation}
    \frac{1}{T}\sum_{t=0}^{T-1} \mathbb{E}\bigl\|\nabla f(\mathbf{x}^t)\bigr\|^2
    \;=\;
    \mathcal{O}\!\left(
        \frac{(m-b) G^2}{m b} \,
        \frac{\sqrt{1+\omega}}{\sqrt{K}}
    \right).
    \label{eq:bound_cor_3}
\end{equation}
Finally, for $b = \Theta(m)$, we can remove the explicit dependence on $b$ and write
\begin{equation}
    \frac{1}{T}\sum_{t=0}^{T-1} \mathbb{E}\bigl\|\nabla f(\mathbf{x}^t)\bigr\|^2
    \;=\;
    \mathcal{O}\!\left(
        \frac{G^2\sqrt{1+\omega}}{\sqrt{K}m}
    \right).
    \label{eq:eris_sgd_bound_asymptotic}
\end{equation}
Consequently, in this regime the stationarity error of \textsc{eris} with DSC--SGD is controlled solely by the variance of the local mini-batch gradients and grows with the compression variance $(1+\omega)$ while decreasing with the square root of the number of clients $K$ and the amount of local data $m$.
\end{corollary}


\subsection{Proof of Lemma \ref{lemma:6}}
\begin{proof} \label{proof:lemma6}
    By the definition of  $\mathbf{v}^t$ , we derive the following expression:
    \begin{align}
        \mathbb{E}_t[\|\mathbf{v}^t\|^2] &= \mathbb{E}_t \left[ \left\| \frac{1}{K} \sum_{k=1}^{K} \mathbf{s}_k^t + \frac{1}{K} \sum_{k=1}^{K} \mathcal{C}_k^t (\tilde{\mathbf{g}}_k^t - \mathbf{s}_k^t) \right\|^2 \right] \nonumber \\
        &=\mathbb{E}_t \left[ \left\| \frac{1}{K} \sum_{k=1}^{K} \mathbf{s}_k^t + \frac{1}{K} \sum_{k=1}^{K} \mathcal{C}_k^t (\tilde{\mathbf{g}}_k^t - \mathbf{s}_k^t) + \frac{1}{K} \sum_{k=1}^K \tilde{\mathbf{g}}_k^t -\frac{1}{K} \sum_{k=1}^K \tilde{\mathbf{g}}_k^t \right\|^2 \right] \nonumber \\
        &= \mathbb{E}_t \left[ \left\| \frac{1}{K} \sum_{k=1}^{K} \mathcal{C}_k^t (\tilde{\mathbf{g}}_k^t - \mathbf{s}_k^t) - \frac{1}{K} \sum_{k=1}^{K} (\tilde{\mathbf{g}}_k^t - \mathbf{s}_k^t) + \frac{1}{K} \sum_{k=1}^{K} \tilde{\mathbf{g}}_k^t \right\|^2 \right] \nonumber \\
        &=\mathbb{E}_t \left[ \left\| \frac{1}{K} \sum_{k=1}^{K} \mathcal{C}_k^t (\tilde{\mathbf{g}}_k^t - \mathbf{s}_k^t) - \frac{1}{K} \sum_{k=1}^{K} (\tilde{\mathbf{g}}_k^t - \mathbf{s}_k^t) \right\|^2 \right] +  \mathbb{E}_t \left[ \left\| \frac{1}{K} \sum_{k=1}^{K} \tilde{\mathbf{g}}_k^t \right\|^2 \right] \nonumber \\ 
        & \quad + 2\left\langle \mathbb{E}_t \left[ \frac{1}{K} \sum_{k=1}^{K} \mathcal{C}_k^t (\tilde{\mathbf{g}}_k^t - \mathbf{s}_k^t) - \frac{1}{K} \sum_{k=1}^{K} (\tilde{\mathbf{g}}_k^t - \mathbf{s}_k^t) \right], \mathbb{E}_t \left[ \frac{1}{K} \sum_{k=1}^{K} \tilde{\mathbf{g}}_k^t \right] \right\rangle \nonumber \\
        &\overset{(\ref{eq:compressor_operator})}{\leq} \mathbb{E}_t \left[ \frac{\omega}{K^2} \sum_{k=1}^{K} \| \tilde{\mathbf{g}}_k^t - \mathbf{s}_k^t \|^2 \right] + \mathbb{E}_t \left[ \left\| \frac{1}{K} \sum_{k=1}^{K} \tilde{\mathbf{g}}_k^t \right\|^2 \right], \label{eq:lemma6_1}
    \end{align}
    where \eqref{eq:lemma6_1} follows because the cross term vanishes: the compression error has zero mean ($\mathbb{E}_t [\mathcal{C}_k^t(\cdot)]=\cdot$), making their inner product zero in expectation.
    
    Next, we establish upper bounds for each term in Equation \eqref{eq:lemma6_1}.
    \begin{itemize}
        \item For the first term, we expand and decompose it as follows:
        \begin{align}
            \mathbb{E}_t \left[ \frac{\omega}{K^2} \sum_{k=1}^{K} \| \tilde{\mathbf{g}}^t_k - \mathbf{s}^t_k \|^2 \right] 
            &= \mathbb{E}_t \left[ \frac{\omega}{K^2} \sum_{k=1}^{K} \| (\tilde{\mathbf{g}}^t_k - \nabla f_k(\mathbf{x}^t)) + (\nabla f_k(\mathbf{x}^t) - \mathbf{s}^t_k) \|^2 \right] \nonumber \\
            &= \mathbb{E}_t \left[ \frac{\omega}{K^2} \sum_{k=1}^{K} \left( \| \tilde{\mathbf{g}}^t_k - \nabla f_k(\mathbf{x}^t) \|^2 + \| \nabla f_k(\mathbf{x}^t) - \mathbf{s}^t_k \|^2 \right. \right. \nonumber \\
            &\quad \Bigg. \left. + 2 (\tilde{\mathbf{g}}^t_k - \nabla f_k(\mathbf{x}^t))^\top (\nabla f_k(\mathbf{x}^t) - \mathbf{s}^t_k) \right) \Bigg] \label{eq:zero_out_term}\\
            &=\mathbb{E}_t \left[ \frac{\omega}{K^2} \sum_{k=1}^{K} \| \tilde{\mathbf{g}}^t_k - \nabla f_k(\mathbf{x}^t) \|^2 \right] \nonumber \\ 
            & \quad + \frac{\omega}{K^2} \sum_{k=1}^{K} \| \nabla f_k(\mathbf{x}^t) - \mathbf{s}^t_k \|^2, \label{eq:lemma6_first_term}
        \end{align}
        where the last equality holds because the expectation of the cross-term vanishes due to the unbiased estimator assumption, i.e., $\mathbb{E}_t[\tilde{\mathbf{g}}^t_k] = \nabla f_k(\mathbf{x}^t)$, as specified in Assumption  \ref{assumption:unbiased_estimator}.

        \item Similarly, for the second term, we proceed as follows:
        \begin{align}
            \mathbb{E}^t \left[ \left\| \frac{1}{K} \sum_{k=1}^{K} \tilde{\mathbf{g}}^t_k \right\|^2 \right] 
            &= \mathbb{E}^t \left[ \frac{1}{K^2} \sum_{k=1}^{K} \left\| (\tilde{\mathbf{g}}^t_k - \nabla f_k(\mathbf{x}^t)) + \nabla f_k(\mathbf{x}^t) \right\|^2 \right] \nonumber \\
            &= \mathbb{E}^t \left[ \frac{1}{K^2} \sum_{k=1}^{K} \left( \| \tilde{\mathbf{g}}^t_k - \nabla f_k(\mathbf{x}^t) \|^2 + \| \nabla f_k(\mathbf{x}^t) \|^2 \right. \right. \nonumber \\
            &\quad \Bigg. \left. + 2 (\tilde{\mathbf{g}}^t_k - \nabla f_k(\mathbf{x}^t))^\top \nabla f_k(\mathbf{x}^t) \right) \Bigg] \nonumber \\
            &= \mathbb{E}^t \left[ \frac{1}{K^2} \sum_{k=1}^{K} \| \tilde{\mathbf{g}}^t_k - \nabla f_k(\mathbf{x}^t) \|^2 \right] + \|\nabla f(\mathbf{x}^t)\|^2, \label{eq:lemma6_second_term}
        \end{align}
    \end{itemize}
    The proof concludes by substituting \eqref{eq:lemma6_first_term} and \eqref{eq:lemma6_second_term} into \eqref{eq:lemma6_1}.
\end{proof}

\subsection{Proof of Lemma \ref{lemma:7}}
\begin{proof} \label{proof:lemma7}
    By the definition of the shift update $\mathbf{s}_k^{t+1} = \mathbf{s}_k^t + \gamma^t \mathcal{C}_k^t (\tilde{\mathbf{g}}_k^t-\mathbf{s}_k^t)$, we have:
    \begin{flalign}
        \mathbb{E}_t \left[ \frac{1}{K} \sum_{k=1}^{K} \|\nabla f_k(\mathbf{x}^{t+1}) - \mathbf{s}^{t+1}_k\|^2 \right] = \mathbb{E}_t \left[ \frac{1}{K} \sum_{k=1}^{K} \left\| \nabla f_k(\mathbf{x}^{t+1}) - \mathbf{s}^t_k - \gamma_t \mathcal{C}^t_k (\tilde{\mathbf{g}}^t_k - \mathbf{s}^t_k) \right\|^2 \right] \nonumber &&
    \end{flalign}
    \vspace{-10pt}
    \begin{align}
        &= \mathbb{E}_t \left[ \frac{1}{K} \sum_{k=1}^{K} \left\| \big( \nabla f_k(\mathbf{x}^{t+1}) - \nabla f_k(\mathbf{x}^t) \big) + \big( \nabla f_k(\mathbf{x}^t) - \mathbf{s}^t_k - \gamma_t \mathcal{C}^t_k (\tilde{\mathbf{g}}^t_k - \mathbf{s}^t_k) \big) \right\|^2 \right] \nonumber \\
        &\leq \mathbb{E}_t \left[ \frac{1}{K} \sum_{k=1}^{K} \left( (1 + \frac{1}{\beta_t}) \|\nabla f_k(\mathbf{x}^{t+1}) - \nabla f_k(\mathbf{x}^t)\|^2 \right. \right. \nonumber \\
        & \quad \left. \left. + (1 + \beta_t) \|\nabla f_k(\mathbf{x}^t) - \mathbf{s}^t_k - \gamma_t \mathcal{C}^t_k (\tilde{\mathbf{g}}^t_k - \mathbf{s}^t_k)\|^2 \right. \bigg) \right. \Bigg]  \label{eq:lemma7_1}\\
        &\overset{(\ref{eq:assump1_part1})}{\leq} \mathbb{E}_t \left[ (1 + \frac{1}{\beta_t}) L^2 \|\mathbf{x}^{t+1} - \mathbf{x}^t\|^2 + (1 + \beta_t) \frac{1}{K} \sum_{k=1}^{K} \|\nabla f_k(\mathbf{x}^t) - \mathbf{s}^t_k - \gamma_t \mathcal{C}^t_k (\tilde{\mathbf{g}}^t_k - \mathbf{s}^t_k)\|^2 \right], \label{eq:lemma7_main}
    \end{align}
    where the Equation \eqref{eq:lemma7_1} is obtained from Young’s inequality $\|\mathbf{a}+\mathbf{b}\|^2 \leq (1+\frac{1}{\beta})\|\mathbf{a}\|^2 + (1+\beta)\|\mathbf{b}\|^2$ with any \( \beta_t > 0 \).

    To further bound the second term in \eqref{eq:lemma7_main}, we expand the squared norm:

    \begin{flalign}
        \quad \quad \mathbb{E}_t & \left[ \frac{1}{K} \sum_{k=1}^{K} \left\| \nabla f_k(\mathbf{x}^t) - \mathbf{s}^t_k - \gamma_t \mathcal{C}^t_k (\tilde{\mathbf{g}}^t_k - \mathbf{s}^t_k) \right\|^2 \right] \nonumber  && \\
        & \quad \quad \quad= \mathbb{E}_t \left[ \frac{1}{K} \sum_{k=1}^{K} \left( \|\nabla f_k(\mathbf{x}^t) - \mathbf{s}^t_k\|^2 + \gamma_t^2 \|\mathcal{C}^t_k (\tilde{\mathbf{g}}^t_k - \mathbf{s}^t_k)\|^2 \nonumber \right. \right.  \\
        & \quad \quad \quad \quad \left. \left. - 2\gamma_t \langle \nabla f_k(\mathbf{x}^t) - \mathbf{s}^t_k, \mathcal{C}^t_k (\tilde{\mathbf{g}}^t_k - \mathbf{s}^t_k) \rangle \right) \right. \Bigg] \label{eq:lemma7_expansion} 
    \end{flalign}

    Since the expectation of the inner product term satisfies:
    \begin{equation}
        \mathbb{E}_t[\langle\nabla f_k(\mathbf{x}^t)-\mathbf{s}^t_k, \mathcal{C}_k^t(\tilde{\mathbf{g}}^t_k-\mathbf{s}^t_k)\rangle] = \mathbb{E}_t[||\nabla f_k(\mathbf{x}^t)-\mathbf{s}^t_k||^2], \nonumber
    \end{equation}

    Equation \eqref{eq:lemma7_expansion} simplifies to:
    
    \begin{flalign}
        \quad \quad \mathbb{E}_t & \left[ \frac{1}{K} \sum_{k=1}^{K} \left\| \nabla f_k(\mathbf{x}^t) - \mathbf{s}^t_k - \gamma_t \mathcal{C}^t_k (\tilde{\mathbf{g}}^t_k - \mathbf{s}^t_k) \right\|^2 \right] \nonumber && \\ & \quad \quad \quad  = \mathbb{E}_t \left[ \frac{1}{K} \sum_{k=1}^{K} \left( (1 - 2\gamma_t) \|\nabla f_k(\mathbf{x}^t) - \mathbf{s}^t_k\|^2 + \gamma_t^2 \|\mathcal{C}^t_k (\tilde{\mathbf{g}}^t_k - \mathbf{s}^t_k)\|^2 \right) \right] \label{eq:lemma7_simplification}
    \end{flalign}

    Then, applying Definition \ref{def:compression_operator} to the term $(\tilde{\mathbf{g}}^t_k-\mathbf{s}^t_k)$, we derive the following inequality:

    \begin{align}
        \mathbb{E}_t\ \Bigl[\|\mathcal{C}_k^t(\tilde{\mathbf{g}}_k^t-\mathbf{s}_k^t)\|^2\Bigr]
        &=\mathbb{E}_t\ \Bigl[\bigl\|\bigl(\tilde{\mathbf{g}}_k^t-\mathbf{s}_k^t\bigr) +\bigl(\mathcal{C}_k^t (\tilde{\mathbf{g}}_k^t-\mathbf{s}_k^t) - (\tilde{\mathbf{g}}_k^t-\mathbf{s}_k^t) \bigr)\bigr\|^2\Bigr] \nonumber \\
        &= \mathbb{E}_t\ \Bigl[\|\tilde{\mathbf{g}}_k^t-\mathbf{s}_k^t\|^2\Bigr] + 2 \mathbb{E}_t\ \Bigl[\bigl\langle \tilde{\mathbf{g}}_k^t-\mathbf{s}_k^t, \mathcal{C}_k^t(\tilde{\mathbf{g}}_k^t-\mathbf{s}_k^t) - (\tilde{\mathbf{g}}_k^t- \mathbf{s}_k^t) \bigr\rangle\Bigr]\nonumber \\
        &\quad + \mathbb{E}_t\ \Bigl[\bigl\|\mathcal{C}_k^t (\tilde{\mathbf{g}}_k^t-\mathbf{s}_k^t) - (\tilde{\mathbf{g}}_k^t- \mathbf{s}_k^t)\bigr\|^2\Bigr]\nonumber \\
        &= \mathbb{E}_t\ \Bigl[\|\tilde{\mathbf{g}}_k^t-\mathbf{s}_k^t\|^2\Bigr] + \mathbb{E}_t\ \Bigl[\bigl\|\mathcal{C}_k^t(\tilde{\mathbf{g}}_k^t-\mathbf{s}_k^t) - (\tilde{\mathbf{g}}_k^t-\mathbf{s}_k^t)\bigr\|^2\Bigr] \nonumber \\
        &\leq \bigl(1 + \omega\bigr) \mathbb{E}_t\ \Bigl[\|\tilde{\mathbf{g}}_k^t-\mathbf{s}_k^t\|^2\Bigr]. \label{eq:mot2}
    \end{align}
    
    Substituting Equation \eqref{eq:mot2} into \eqref{eq:lemma7_simplification}, we simplify the second term as follows:
    \begin{flalign}
        \quad \quad \mathbb{E}_t & \left[ \frac{1}{K} \sum_{k=1}^{K} \left\| \nabla f_k(\mathbf{x}^t) - \mathbf{s}^t_k - \gamma_t \mathcal{C}^t_k (\tilde{\mathbf{g}}^t_k - \mathbf{s}^t_k) \right\|^2 \right] \nonumber && \\
        &  \quad \quad \quad \leq \mathbb{E}_t \left[ \frac{1}{K} \sum_{k=1}^{K} \left( (1 - 2\gamma_t) \|\nabla f_k(\mathbf{x}^t) - \mathbf{s}^t_k\|^2 + \gamma_t^2 (1 + \omega) \|\tilde{\mathbf{g}}^t_k - \mathbf{s}^t_k\|^2 \right) \right] \nonumber \\
        & \quad \quad \quad  \overset{(\ref{eq:lemma6_first_term})}{=} \mathbb{E}_t \left[ \frac{1}{K} \sum_{k=1}^{K} \left( (1 - 2\gamma_t + \gamma_t^2(1+\omega)) \|\nabla f_k(\mathbf{x}^t) - \mathbf{s}^t_k\|^2 \right. \right. \nonumber \\
        & \left. \quad \quad \quad \quad + \gamma_t^2 (1+\omega) \|\tilde{\mathbf{g}}^t_k - \nabla f_k(\mathbf{x}^t)\|^2 \right) \Biggr]. \label{eq:lemma7_second_term}
    \end{flalign}
    
    By plugging \eqref{eq:lemma7_second_term} into \eqref{eq:lemma7_main}, we obtain:

    \begin{flalign}
        \mathbb{E}_t \left[ \frac{1}{K} \sum_{k=1}^{K} \|\nabla f_k(\mathbf{x}^{t+1}) - \mathbf{s}^{t+1}_k\|^2 \right] 
        & \leq \mathbb{E}_t \left[ (1 + \frac{1}{\beta_t}) L^2 \|\mathbf{x}^{t+1} - \mathbf{x}^t\|^2  \right. \nonumber && \\
        & \quad + (1 + \beta_t) \frac{1}{K} \sum_{k=1}^{K}  (1 - 2\gamma_t + \gamma_t^2(1+\omega)) \|\nabla f_k(\mathbf{x}^t) - \mathbf{s}^t_k\|^2  \nonumber\\
        & \quad + \left. (1 + \beta_t) \frac{1}{K} \sum_{k=1}^{K}\gamma_t^2 (1+\omega) \|\tilde{\mathbf{g}}^t_k - \nabla f_k(\mathbf{x}^t)\|^2 \right]
        \label{eq:44}
    \end{flalign}
    
    Finally, setting $\beta_t = \frac{1}{1+2\omega}$, and $\gamma_t=\sqrt{\frac{1+2 \omega}{2(1+\omega)^3}}$, we approximate the second term in Equation \eqref{eq:44} with the following upper bound:
    \begin{align}
        (1 + \beta_t)(1 - 2\gamma_t + \gamma_t^2(1+\omega)) &=\frac{2(1+\omega)}{(1+2\omega} \left( 1-2\sqrt{\frac{1+2 \omega}{2(1+\omega)^3}} + \frac{(1+2 \omega)}{2(1+\omega)^2}  \right) \nonumber\\
        &\leq 1 - \frac{1}{2(1+\omega)} \quad \quad \forall \omega \geq 0. \nonumber
    \end{align}
    Substituting this bound into Equation \eqref{eq:44}, we obtain Equation \eqref{eq:lemma7}, thereby completing the proof of Lemma \ref{lemma:7}.
    \end{proof}

\subsection{Proof of Lemma \ref{lemma:8}}
\begin{proof} \label{proof:lemma8}
    Given the definition $\mathcal{S}^t := \frac{1}{K} \sum_{k=1}^K ||\nabla f_k(\mathbf{x}^t ) - \mathbf{s}^t_k||^2$, we can derive a recursive bound for $\mathcal{S}^{t+1}$ using Lemma \ref{lemma:7}:
    \begin{align}
        \mathbb{E}_t \left[
        \mathcal{S}^{t+1}
        \right] 
        &\leq \mathbb{E}_t \left[
        \left(1 - \frac{1}{2(1+\omega)}\right) \mathcal{S}^t 
         + \frac{1}{(1+\omega)K} \sum_{k=1}^{K} \|\tilde{\mathbf{g}}_k^t - \nabla f_k(\mathbf{x}^t) \|^2 + 2(1+\omega)L^2 \|\mathbf{x}^{t+1} - \mathbf{x}^t\|^2
        \right] \nonumber \\
        &\overset{(\ref{eq:assump4a_app})}{\leq} \mathbb{E}_t \left[
        \left(1 - \frac{1}{2(1+\omega)}\right) \mathcal{S}^t 
         + \frac{\Gamma_1 \Delta^t + \Gamma_2}{(1+\omega)} + 2(1+\omega)L^2 \|\mathbf{x}^{t+1} - \mathbf{x}^t\|^2
        \right] \label{eq:s_definition}
    \end{align}

    We now use Equation \eqref{eq:s_definition} along with \eqref{eq:F6} to bound the potential function $\Phi_{t+1}$, as defined in Equation \eqref{eq:pot_function}: 
    \begin{align}
        \mathbb{E}_t[\Phi_{t+1}] &:=\mathbb{E}_t \left[f(\mathbf{x}^{t+1}) - f^* + \alpha L \Delta^{t+1} + \frac{\beta}{L}\mathcal{S}^{t+1} \right] \nonumber \\
        &\leq \mathbb{E}_t \Bigg[ f(\mathbf{x}^{t}) - \lambda_t ||\nabla f(\mathbf{x}^{t})||^2 + \frac{\lambda^2L}{2}\|\mathbf{v}^t\|^2 - f^* + \alpha L \Delta^{t+1} \Bigg. \nonumber  \\ 
        &\quad \Bigg. +\frac{\beta}{L} \left( \left(1 - \frac{1}{2(1+\omega)}\right) \mathcal{S}^t + \frac{\Gamma_1 \Delta^t + \Gamma_2}{(1+\omega)} + 2(1+\omega)L^2 \|\mathbf{x}^{t+1} - \mathbf{x}^t\|^2 \right) \Bigg] \nonumber \\
        &\overset{(\ref{eq:assump4b})}{\leq} \mathbb{E}_t \Bigg[ f(\mathbf{x}^{t}) - f^* - \lambda_t ||\nabla f(\mathbf{x}^{t})||^2 + \frac{\lambda^2L}{2}\|\mathbf{v}^t\|^2 \Bigg. \nonumber \\ 
        &\quad + \alpha L \bigg( (1-\theta)\Delta^t + \Gamma_3||\nabla f(\mathbf{x}^t)||^2 + \Gamma_4 ||\mathbf{x}^{t+1}-\mathbf{x}^t||^2 \bigg) \nonumber \\ 
        &\quad \Bigg. +\frac{\beta}{L} \left( \left(1 - \frac{1}{2(1+\omega)}\right) \mathcal{S}^t + \frac{\Gamma_1 \Delta^t + \Gamma_2}{(1+\omega)} + 2(1+\omega)L^2 \|\mathbf{x}^{t+1} - \mathbf{x}^t\|^2 \right) \Bigg] \nonumber \\
        &= \mathbb{E}_t \Bigg[ f(\mathbf{x}^{t}) - f^* - \lambda_t ||\nabla f(\mathbf{x}^{t})||^2 + \left(\frac{1}{2} + \alpha \Gamma_4 + \beta (1+\omega) \right) L \lambda_t^2 \|\mathbf{v}^t\|^2 \Bigg. \nonumber \\ 
        &\quad + \alpha L \bigg( (1-\theta)\Delta^t + \Gamma_3||\nabla f(\mathbf{x}^t)||^2 \bigg) \Bigg. \nonumber \\
        &\quad +\frac{\beta}{L} \left( \left(1 - \frac{1}{2(1+\omega)}\right) \mathcal{S}^t + \frac{\Gamma_1 \Delta^t + \Gamma_2}{(1+\omega)} \right) \Bigg] \label{eq:34}
    \end{align}
    where the last equality follows the adopted update rule $\mathbf{x}^{t+1} = \mathbf{x}^{t} - \lambda_t\mathbf{v}^{t}$. Now adopting $\mathcal{S}^t$ into the definition of $\mathbb{E}_t[||\mathbf{v}^t||^2]$ provided in Lemma \ref{lemma:6}, we obtain:
    \begin{align}
        \mathbb{E}_t[\|\mathbf{v}^t\|^2] &\leq \mathbb{E}_t \left[
        \frac{(1 + \omega)}{K^2} \sum_{k=1}^K \| \tilde{\mathbf{g}}_k^t - \nabla f_k(\mathbf{x}^t) \|^2
        + \frac{\omega}{K} \mathcal{S}^t + \|\nabla f(\mathbf{x}^t)\|^2 \right] \nonumber\\
        &\overset{(\ref{eq:assump4a_app})}{\leq} \mathbb{E}_t \left[
        \frac{(1 + \omega)}{K} (\Gamma_1 \Delta^t+\Gamma_2)
        + \frac{\omega}{K} \mathcal{S}^t + \|\nabla f(\mathbf{x}^t)\|^2 \right] \label{eq:new_lemma6}
    \end{align}

    Substituting Equation \eqref{eq:new_lemma6} into \eqref{eq:34}, we obtain as follows:
    \begin{align}
        \mathbb{E}_t[\Phi_{t+1}] &\leq f(\mathbf{x}^{t+1}) - f^* \nonumber \\
        & \quad + \left[\left(\frac{1}{2} + \alpha \Gamma_4+2 \beta(1+\omega)\right) \frac{(1+\omega)\Gamma_1\lambda_t^2}{K} +\alpha(1-\omega)+\frac{\beta \Gamma_1}{(1+\omega)L^2} \right] L \Delta^t \nonumber \\
        & \quad + \left[\left(\frac{1}{2} + \alpha \Gamma_4 + 2 \beta (1+\omega) \right)\frac{\omega L^2 \lambda_t^2}{K} + \beta \big(1-\frac{1}{2(1+\omega)} \big) \right] \frac{\mathcal{S}^t}{L} \nonumber \\
        & \quad - \left[\lambda_t - \left(\frac{1}{2} + \alpha \Gamma_4 + 2 \beta (1+\omega) \right) L \lambda_t^2 - \alpha L \Gamma_3 \right] ||\nabla f(x^t)||^2 \nonumber \\
        & \quad + \left[ \left(\frac{1}{2} + \alpha \Gamma_4 + 2 \beta (1+\omega) \right) \frac{(1+\omega)L \lambda_t^2}{K} + \frac{\beta}{(1+\omega)L} \right] \Gamma_2 \label{eq:54}
    \end{align}

    To ensure that the right-hand side of Equation \eqref{eq:54} remains consistent with the potential function $\Phi_t := f(\mathbf{x}^t) -f^* + \alpha L\Delta^t + \frac{\beta}{L} \mathcal{S}_t$, we select the parameters $\alpha$, $\beta$, and $\lambda_t$ to satisfy the following constraints: 
    \begin{align}
         \left(\frac{1}{2} + \alpha \Gamma_4+2 \beta(1+\omega)\right) \frac{(1+\omega)\Gamma_1\lambda_t^2}{K} +\alpha(1-\omega)+\frac{\beta \Gamma_1}{(1+\omega)L^2} \leq \alpha \label{eq:cond1}
     \end{align}
    
    \begin{align}
        \left(\frac{1}{2} + \alpha \Gamma_4 + 2 \beta (1+\omega) \right)\frac{\omega L^2 \lambda_t^2}{K} + \beta \big(1-\frac{1}{2(1+\omega)} \big) \leq \beta \label{eq:cond2}
    \end{align}

    Although these are not the strictest possible bounds 
    for a fair comparison with the utility results of SoteriaFL, we adopt the same choices for $\alpha$, $\beta$, and $\lambda_t$, ensuring they satisfy conditions \eqref{eq:cond1} and \eqref{eq:cond2}:
    \begin{equation}
        \alpha \geq \frac{3\beta \Gamma_1}{2(1+\omega)L^2\theta}  \quad \quad \quad \forall \beta>0 \label{eq:alpha}
    \end{equation}
    \begin{equation}
        \lambda_t \equiv\lambda \leq \frac{\sqrt{\beta K}}{\sqrt{1+2\alpha \Gamma_4 + 4 \beta (1+\omega )}(1+\omega)L} \label{eq:lambda1}
    \end{equation}

    Here, Equation \eqref{eq:alpha} follows from the constraint in \eqref{eq:cond1}, while \eqref{eq:lambda1} ensures compatibility with the potential function definition in \eqref{eq:pot_function}. Additionally, we impose a further bound on $\lambda_t$ to guarantee that the negative gradient squared term remains sufficiently large (i.e., \(\ge \tfrac{\lambda_t}{2}\,\|\nabla f(\mathbf{x}^t)\|^2\)), obtaining:
    \begin{equation}
        \lambda_t \equiv\lambda \leq \frac{1}{(1+2\alpha \Gamma_4 +4 \beta(1+\omega) + 2 \alpha \Gamma_3 /\lambda
        ^2)L} \label{eq:lambda2}
    \end{equation}


    Finally, substituting the conditions \eqref{eq:alpha}–\eqref{eq:lambda2} into Equation \eqref{eq:54}, we obtain:  
    \begin{equation}
        \mathbb{E}_t[\Phi_{t+1}] \leq \Phi_t - \frac{\lambda_t}{2} ||\nabla f(x^t)||^2 + \frac{3 \beta}{2(1+\omega)L}\Gamma_2
    \end{equation}
    The last term is obtained directly by applying the bound from Equation \eqref{eq:lambda2}, completing the proof.

\end{proof}

\subsection{Utility Comparison} \label{app:utility_comparison_th}
\begin{table}[t]
\caption{Asymptotic utility / accuracy bounds (average squared gradient norm after $T$ rounds) for different (local) differentially-private FL algorithms for the nonconvex problem in \eqref{eq:fedavg_opt}, compared to the non-DP utility bound of \textsc{eris}\textnormal{-}SGD with DSC. Here $K$ is the number of clients, $m$ the number of samples per client, $n$ the model dimension, $\omega$ the compressor variance parameter, and $(\varepsilon,\delta)$ the privacy parameters. All bounds hide absolute constants and, where standard, additional logarithmic factors. For SoteriaFL, $\tau :=(1+\omega)^{3/2} / \sqrt{K}$. Note that smaller values of the bound correspond to better utility / accuracy.}
\centering
\small
\begin{tabular}{lll}
\toprule
\textbf{Algorithm} & \textbf{Privacy} &
\textbf{Utility / Accuracy} \\[2pt]
\midrule
Distributed DP\textnormal{-}SRM~\citep{wangEfficientPriv2023} &
$(\varepsilon,\delta)$-DP &
$\tilde{\mathcal{O}}\!\left(
    \dfrac{\sqrt{n\log(1/\delta)}}{K m \varepsilon}
\right)$ \\[7pt]
SDM\textnormal{-}DSGD~\citep{zhangPrivateComm2020a} &
$(\varepsilon,\delta)$-LDP &
$\tilde{\mathcal{O}}\!\left(
    \dfrac{\sqrt{n\log(1/\delta)}}{\sqrt{K} m \varepsilon}
\right)$ \\[7pt]
Q\textnormal{-}DPSGD\textnormal{-}1~\citep{dingDP2021} &
$(\varepsilon,\delta)$-LDP &
$\tilde{\mathcal{O}}\!\left(
    \dfrac{\bigl(\frac{\tilde{\nu}^2}{K}+\frac{1}{m}\bigr)^{\!\!2/3}
           \bigl(n\log(1/\delta)\bigr)^{1/3}}
          {m^{2/3}\varepsilon^{2/3}}
\right)$ \\[7pt]
CDP\textnormal{-}SGD~\citep{soteriafl2022} &
$(\varepsilon,\delta)$-LDP &
$\tilde{\mathcal{O}}\!\left(
    \dfrac{\sqrt{(1+\omega)\,n\log(1/\delta)}}{\sqrt{K} m \varepsilon}
\right)$ \\[7pt]
SoteriaFL\textnormal{-}SGD~\citep{soteriafl2022} &
$(\varepsilon,\delta)$-LDP &
$\tilde{\mathcal{O}}\!\left(
    \dfrac{\sqrt{(1+\omega)\,n\log(1/\delta)}}{\sqrt{K} m \varepsilon}\,
    (1+\sqrt{\tau})
\right)$ \\[6pt]
\textsc{eris}\textnormal{-}SGD (+\textsc{dsc})  &
--- &
$\tilde{\mathcal{O}}\!\left(
    \dfrac{\sqrt{1+\omega}}{\sqrt{K}m}
\right)$ \\[2pt]
\bottomrule
\end{tabular}
\label{tab:utility_comparison}
\end{table}

Table~\ref{tab:utility_comparison} summarizes the asymptotic utility/accuracy guarantees of existing differentially-private FL algorithms, most of them with communication compression, and compares them to our non-DP utility bound for \textsc{eris}\textnormal{-}SGD with DSC. Distributed DP\textnormal{-}SRM~\citep{wangEfficientPriv2023} provides a global $(\varepsilon,\delta)$\nobreakdash-DP baseline without compression: its utility improves linearly in the number of clients $K$ and in the number of samples per client $m$, but it does not consider LDP and therefore is not directly comparable to the LDP-based protocols in the rest of the table.

SDM\textnormal{-}DSGD~\citep{zhangPrivateComm2020a} and Q\textnormal{-}DPSGD\textnormal{-}1~\citep{dingDP2021} are early attempts to combine local DP with compressed communication. However, SDM\textnormal{-}DSGD assumes random-$k$ sparsification and requires $1+\omega \ll \log T$ (i.e., communicating at least $k \gtrsim n / \log T$ coordinates per round), and its bound hides logarithmic factors that grow faster than $(1+\omega)$. Q\textnormal{-}DPSGD\textnormal{-}1 relies on a different compression assumption ($\mathbb{E}\bigl[\|\mathcal{C}(\mathbf{x})-\mathbf{x}\|^2 \bigr] \leq \tilde{\nu}^2$, with parameter $\tilde{\nu}^2$ playing a similar role to our $1+\omega$) and incurs a strictly worse utility by a factor $T^{1/6}$ compared to later methods, as already observed in~\citep{soteriafl2022}.

CDP\textnormal{-}SGD~\citep{soteriafl2022} can be seen as a direct compressed analogue of DP\textnormal{-}SGD: it achieves $(\varepsilon,\delta)$\nobreakdash-LDP and a utility that degrades with $\sqrt{(1+\omega)n}/(\sqrt{K}\, m \varepsilon)$, but still requires $\mathcal{O}(m^2)$ communication rounds when the local dataset size $m$ is large. SoteriaFL\textnormal{-}SGD/GD improves upon CDP\textnormal{-}SGD via shifted compression: it preserves the same dependence on $(1+\omega),K,m,n$ up to a mild factor $(1+\sqrt{\tau})$, where $\tau=(1+\omega)^{3/2}/\sqrt{K}$ becomes negligible as $K \gg (1+\omega)^3$, while reducing the total communication to $\mathcal{O}(m)$ rounds.

In contrast, \textsc{eris}\textnormal{-}SGD with DSC does not inject any differentially-private noise to ensure formal $(\varepsilon,\delta)$\nobreakdash-DP (although standard LDP mechanisms can be applied on top of it as shown in Figure~\ref{fig:pareto_16}), and thus its bound cannot be directly compared in terms of privacy guarantees. Nevertheless, once privacy noise is removed, our analysis shows that \textsc{eris} achieves a dimension-free non-private utility bound that scales as $\tilde{\mathcal{O}} \!\bigl(\sqrt{1+\omega} / (m\sqrt{K})\bigr)$, yielding faster convergence under the same optimization assumptions. Moreover, \textsc{eris} exhibits the same favorable dependence on the number of clients $K$ and on the compression variance $(1+\omega)$ as SoteriaFL-style methods, while operating in a fully serverless, sharded architecture. 
Empirically (see Figure~\ref{fig:pareto_16}), when we add the \emph{same} LDP mechanism to both methods, \textsc{eris} and SoteriaFL achieve comparable utility for a given $(\varepsilon,\delta)$, but \textsc{eris} requires less injected noise to reach this regime due to the privacy amplification inherent in its decentralized, sharded aggregation architecture.

\section{Privacy Guarantees for \textsc{eris} with DSC} 
\label{app:privacy_proof}
In this section, we present the detailed proof of Theorem~\ref{thm:privacy}, which establishes an upper bound on the information leakage incurred by \textsc{eris} with DSC under the honest-but-curious threat model. The analysis follows an information-theoretic approach by bounding the mutual information between a client's local dataset $D_k$ and the adversary's partial view of the transmitted model updates \( \mathbf{v}_{k,(a)}^t=(\tilde{\mathbf{g}}_k^t - \mathbf{s}_k^t) \odot \mathbf{m}_{\mathcal{C}_k^t} \odot \mathbf{m}_{(a)}^t \) over $T$ communication rounds. We then extend the result to colluding adversaries, who may share observations to amplify their attack.

\subsection{Proof of Theorem \ref{thm:privacy}} \label{app:d1}
\begin{proof}
For lighter notation, define the combined mask
$\mathbf{m}_k^t:=\mathbf{m}_{\mathcal{C}_k^t}\odot\mathbf{m}_{(a)}^t$,
so that $\mathbf{v}_{k,(a)}^t=\mathbf{U}_k^t\odot\mathbf{m}_k^t$.
We work directly with the pre-mask update $\mathbf{U}_k^t$ and write
\[
\mathcal{H}_t:=\sigma\!\left(\{(\mathbf{v}_{k,(a)}^\ell,\mathbf{m}_k^\ell):0\le\ell<t\}\right)
\]
for this observer's history before round $t$. Revealing the masks can only increase information. The chain rule therefore gives
\begin{align*}
 I\bigl(D_k;\{\mathbf{v}_{k,(a)}^t\}_{t=0}^{T-1}\bigr)
 &\le I\bigl(D_k;\{(\mathbf{v}_{k,(a)}^t,\mathbf{m}_k^t)\}_{t=0}^{T-1}\bigr)\\
 &=\sum_{t=0}^{T-1} I\bigl(D_k;\mathbf{v}_{k,(a)}^t,\mathbf{m}_k^t\mid\mathcal{H}_t\bigr)\\
 &=\sum_{t=0}^{T-1} I\bigl(D_k;\mathbf{U}_k^t\odot\mathbf{m}_k^t\mid\mathcal{H}_t,\mathbf{m}_k^t\bigr),
\end{align*}
where the last equality uses $I(D_k;\mathbf{m}_k^t\mid\mathcal{H}_t)=0$.

Fix a history $h$ and a mask realization $\mathbf{m}$, and order the revealed coordinates as
\[
S(\mathbf{m}):=\{i:m_i\neq0\}=\{i_1,\ldots,i_s\},
\qquad J_j:=\{i_1,\ldots,i_{j-1}\}.
\]
The nonzero mask values are known invertible scalings. Conditional mask independence and the coordinate-wise chain rule imply
\begin{align*}
&I\bigl(D_k;\mathbf{U}_k^t\odot\mathbf{m}_k^t\mid\mathcal{H}_t=h,\mathbf{m}_k^t=\mathbf{m}\bigr)\\
&\qquad=I\bigl(D_k;\{U_{k,i}^t\}_{i\in S(\mathbf{m})}\mid\mathcal{H}_t=h\bigr)\\
&\qquad=\sum_{j=1}^{s}I\bigl(D_k;U_{k,i_j}^t\mid\mathcal{H}_t=h,\mathbf{U}_{k,J_j}^t\bigr)
\;\le\;|S(\mathbf{m})|\,C_{\max}.
\end{align*}
In particular, the constant in the theorem can be chosen as
\[
C_{\max}:=\sup_{t,i,J}\operatorname*{ess\,sup}_{h}
I\bigl(D_k;U_{k,i}^t\mid\mathcal{H}_t=h,\mathbf{U}_{k,J}^t\bigr),
\qquad J\subseteq\{1,\ldots,n\}\setminus\{i\}.
\]
Here the conditional mutual information averages over $\mathbf{U}_{k,J}^t$; the essential supremum is over the pre-round history $h$.
By the expected-exposure assumption, $\mathbb{E}[|S(\mathbf{m}_k^t)|]=np/A$. Taking expectations gives
\[
I\bigl(D_k;\mathbf{U}_k^t\odot\mathbf{m}_k^t\mid\mathcal{H}_t,\mathbf{m}_k^t\bigr)
\le\frac{n}{A}\,p\,C_{\max},
\]
and summing over $t=0,\ldots,T-1$ yields
\[
I\bigl(D_k;\{\mathbf{v}_{k,(a)}^t\}_{t=0}^{T-1}\bigr)
\le T\,\frac{n}{A}\,p\,C_{\max}.
\]
\end{proof}


\begin{remark}[Gaussian-channel instantiation of $C_{\max}$] \label{remark:privacy_remark}
For a fixed client $k$ and history $\mathcal{H}_t=h_t$, let $U_{k,i}^t:=\widetilde{g}_{k,i}^t-s_{k,i}^t$ denote coordinate $i$ of the pre-mask update.
Consider the Gaussian-channel instantiation
\[
U_{k,i}^{t}
=
\mu_{i,t}(D_k,h_t)+Z_{i,t},
\qquad
Z_{i,t}\perp D_k\mid \mathcal{H}_t=h_t,
\qquad
Z_{i,t}\sim\mathcal{N}\!\left(0,\sigma_{w,i,t}^{2}\right),
\]
where, conditional on $\mathcal{H}_t=h_t$, the noise vector is independent of $D_k$, its coordinates are mutually independent, and $\sigma_{w,i,t}^2>0$. The dataset-dependent means may be correlated across coordinates. Define the finite between-dataset variance
\(
\sigma_{b,i,t}^{2}
:=
\operatorname{Var}_{D_k\mid \mathcal{H}_t=h_t}
\!\left[
\mu_{i,t}(D_k,h_t)
\right].
\)
Then
\[
\operatorname{Var}\!\left(U_{k,i}^{t}\mid \mathcal{H}_t=h_t\right)
=
\sigma_{b,i,t}^{2}+\sigma_{w,i,t}^{2}.
\]
Using the Gaussian maximum-entropy property, and reporting mutual information in bits, gives
\[
I\!\left(D_k;U_{k,i}^{t}\mid \mathcal{H}_t=h_t\right)
\leq
\frac{1}{2}
\log_2
\left(
1+\frac{\sigma_{b,i,t}^{2}}{\sigma_{w,i,t}^{2}}
\right).
\]
Moreover, $U_{k,i}^t\perp\mathbf{U}_{k,J}^t\mid(D_k,\mathcal{H}_t=h_t)$ for $i\notin J$, so
\[
\begin{aligned}
I(D_k;U_{k,i}^t\mid\mathcal{H}_t=h_t,\mathbf{U}_{k,J}^t)
&=I(D_k;U_{k,i}^t\mid\mathcal{H}_t=h_t)
 -I(U_{k,i}^t;\mathbf{U}_{k,J}^t\mid\mathcal{H}_t=h_t)\\
&\le \frac12\log_2\!\left(1+\frac{\sigma_{b,i,t}^2}{\sigma_{w,i,t}^2}\right).
\end{aligned}
\]
A uniform supremum of this variance-ratio bound over coordinates, rounds, and histories therefore bounds $C_{\max}$ in Theorem~\ref{thm:privacy}.
The scalar bound does not require the dataset-dependent means
$\mu_{i,t}(D_k,h_t)$ to be Gaussian; equality holds when these means follow a Gaussian random-effects model. Consequently, the ratio
$\sigma_{b,i,t}^{2}/\sigma_{w,i,t}^{2}$ provides an operational signal-to-noise characterization of the per-coordinate quantity entering $C_{\max}$.
\end{remark}

\paragraph{Operational estimation of $C_{\max}$.}
At a fixed update state, consider $M$ independently sampled client datasets and $R$ repeated stochastic updates for each dataset. Under the conditional independent-noise model of Remark~\ref{remark:privacy_remark}, the following marginal variance estimates give a plug-in estimate of their conditional information bound. For the theorem, the dataset law at a realized history is $D_k\mid\mathcal{H}_t=h_t$; resampling datasets at a fixed checkpoint instead audits the chosen resampling law. Let
$U_{i,t}^{(m,r)}$ denote coordinate $i$ of repetition $r$ for dataset $m$, and define
\[
\overline{U}_{i,t}^{(m)}
=
\frac{1}{R}
\sum_{r=1}^{R}
U_{i,t}^{(m,r)},
\qquad
\overline{U}_{i,t}
=
\frac{1}{M}
\sum_{m=1}^{M}
\overline{U}_{i,t}^{(m)}.
\]
The within-dataset variance can be estimated as
\[
\widehat{\sigma}_{w,i,t}^{2}
=
\frac{1}{M}
\sum_{m=1}^{M}
\frac{1}{R-1}
\sum_{r=1}^{R}
\left(
U_{i,t}^{(m,r)}
-
\overline{U}_{i,t}^{(m)}
\right)^2,
\]
while the between-dataset component is estimated by correcting the variance of the dataset-specific means for finite $R$:
\[
\widehat{\sigma}_{b,i,t}^{2}
=
\left[
\frac{1}{M-1}
\sum_{m=1}^{M}
\left(
\overline{U}_{i,t}^{(m)}
-
\overline{U}_{i,t}
\right)^2
-
\frac{\widehat{\sigma}_{w,i,t}^{2}}{R}
\right]_{+},
\]
where $[z]_{+}:=\max\{z,0\}$. This gives the empirical per-coordinate estimate
\[
\widehat{C}_{i,t}
=
\frac{1}{2}
\log_2
\left(
1+
\frac{\widehat{\sigma}_{b,i,t}^{2}}
{\widehat{\sigma}_{w,i,t}^{2}}
\right).
\]
For a finite audit set $\mathcal{A}$ of coordinate--round pairs, we denote
\[
\widehat{C}_{\max}^{\mathrm{obs}}
:=
\max_{(i,t)\in\mathcal{A}}
\widehat{C}_{i,t}.
\]
Unlike $C_{\max}$ in Theorem~\ref{thm:privacy}, $\widehat{C}_{\max}^{\mathrm{obs}}$ is an observed maximum over the audited states and should therefore not be interpreted as a certified supremum over all possible datasets and training histories. Marginal variance estimates alone do not verify the independent-noise assumption. If $\sigma_{w,i,t}^2=0$, this Gaussian-channel argument gives no finite bound.

\subsection{Privacy Under Colluding Aggregators} \label{app:privacy_collusion}

We now extend our analysis to a coalition of aggregators that share their shards before attempting the attack.  
Let $\mathcal{C}\subseteq\{1,\dots,A\}$ denote the set of colluding aggregators with cardinality $A_c := |\mathcal{C}|$.

\begin{corollary}[Colluding–aggregator privacy bound]
\label{cor:colluding_privacy}
Assume the setting of Theorem~\ref{thm:privacy}.
For every communication round $t\in\{0,\dots,T-1\}$ let the union mask
\[
  \mathbf{m}_{\mathrm{col}}^{\,t}  
  := \bigvee_{a\in\mathcal{C}}\mathbf{m}_{(a)}^{t}
  \quad\bigl(\;\vee\text{ denotes the element‑wise logical \textsc{or}}\bigr)
\]
select the coordinates revealed to the colluding coalition.
Define the coalition’s view of client $k$ at round $t$ as
\[
  \mathbf{v}_{k,\mathrm{col}}^{\,t}
  := (\tilde{\mathbf{g}}_k^t - \mathbf{s}_k^t)\;\odot\;
     \mathbf{m}_{\mathcal{C}_k^t}\;\odot\;
     \mathbf{m}_{\mathrm{col}}^{\,t}.
\]
Assuming the mask-independence and uniform conditional-information
assumptions of Theorem~\ref{thm:privacy} also hold with
$\mathcal{H}_t$ denoting the coalition's past shards and masks, under the \emph{honest‑but‑curious} threat model, the mutual information between the client’s private dataset $D_k$ and the coalition’s transcript over $T$ rounds satisfies
\[
  I\!\left(D_k;\{\mathbf{v}_{k,\mathrm{col}}^{\,t}\}_{t=0}^{T-1}\right)
  \;\le\;
  n\,T\;\frac{pA_c}{A}\;C_{\max},
\]
where $C_{\max}$ bounds the same conditional per-coordinate terms for the coalition history.
\end{corollary}

\begin{proof}
The extension follows the same steps as in Appendix~\ref{app:privacy_proof}, using the coalition history and replacing the per-shard mask $\mathbf{m}_{(a)}^t$ by the union mask
\[
  \mathbf{m}_{\mathrm{col}}^{\,t}
  \;=\;\bigvee_{a\in\mathcal{C}}\mathbf{m}_{(a)}^{t},
\]
where $\mathcal{C}$ is the set of colluding shards of size $A_c$. Since the shards are disjoint, the retained support $S_{\mathrm{col}}^t:=\{i:m_{\mathcal{C}_k^t,i}m_{\mathrm{col},i}^t\neq0\}$ satisfies
\[
  \mathbb{E}|S_{\mathrm{col}}^t|
  = \sum_{a\in\mathcal{C}}\mathbb{E}|S(\mathbf{m}_{\mathcal{C}_k^t}\odot\mathbf{m}_{(a)}^t)|
  = A_c\,\frac{np}{A}.
\]
The masks remain conditionally independent of $(D_k,\mathbf{U}_k^t)$ given the coalition history. Apply the same coordinate-wise chain rule with that history and the stated uniform bound $C_{\max}$.

Substituting this modification into the rest of the derivation yields
\[
  I\bigl(D_k;\{\mathbf{v}_{k,\mathrm{col}}^{\,t}\}_{t=0}^{T-1}\bigr)
  \;\le\;
  T\,n\,\frac{p\,A_c}{A}\,C_{\max},
\]
as claimed.  In particular:
\begin{itemize}
  \item If $A_c=1$, this reduces to Theorem~\ref{thm:privacy}.
  \item If $A_c=A$, the sharding protection vanishes and the bound becomes
  \(
    I \le n\,T\,p\,C_{\max},
  \)
  governed solely by the compression mechanism.
\end{itemize}

\end{proof}

\begin{remark}
For a common $C_{\max}$ across coalition histories, Corollary~\ref{cor:colluding_privacy} shows that the leakage bound grows \emph{linearly} with the coalition size~$A_c$: increasing collusion progressively weakens the privacy amplification provided by FSA, rather than causing an abrupt loss of protection at a fixed coalition size. Consequently, anticipating up to $A_c^{\max}$ colluding aggregators, one can retain the original leakage bound of Theorem~\ref{thm:privacy} by increasing the shard count to $A \,\mapsto\, A\cdot A_c^{\max}$ or, equivalently, by decreasing the retention probability to $p \,\mapsto\, p/A_c^{\max}$, thereby preserving the product $\tfrac{pA_c}{A}$.
\end{remark}

\begin{figure}
  \captionsetup{skip=3pt} 
  \centering
  \includegraphics[width=0.4\textwidth]{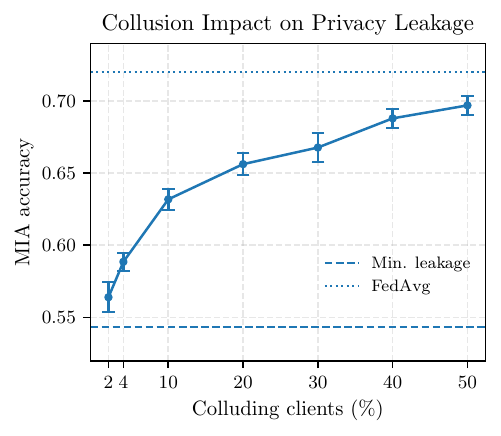} 
  \caption{Impact of honest-but-curious client collusion in \textsc{eris}.}
  \label{fig:collusion}
\end{figure}
\paragraph{Empirical Validation}
To assess the robustness of \textsc{eris} against coordinated leakage attempts, we evaluate how MIA accuracy evolves as multiple honest-but-curious clients collude by sharing their received shards. Figure~\ref{fig:collusion} reports the resulting leakage curve. As the collusion group grows, MIA accuracy increases smoothly but remains consistently below the FedAvg baseline and close to the minimum achievable leakage, even when 50\% of clients collude. These results confirm that the shard-based decomposition in \textsc{eris} meaningfully amplifies privacy, limiting the adversary’s advantage even under strong collusion scenarios.

\section{Experimental Setup} \label{app:experimental_setup}
This section provides additional details on the experimental configuration used throughout the paper, including model architectures, training protocols, and hardware resources. We also describe the software libraries, dataset licenses, and implementation details to ensure full reproducibility. 

\subsection{Models and Hyperparameter Settings} \label{app:model_and_hyper}
We use 5-fold cross-validation across all experiments, varying the random seed for both data generation and model initialization to ensure reproducibility. Each dataset is paired with an appropriate architecture: GPT-Neo \citep{gptneo} (\texttt{EleutherAI/gpt-neo-1.3B}, 1.3B parameters) from HuggingFace for CNN/DailyMail, DistilBERT \citep{Sanh2019DistilBERTAD} (\texttt{distilbert-base-uncased}, 67M parameters) for IMDB, ResNet-9 \citep{resnet2016} (1.65M parameters) for CIFAR-10, and LeNet-5~\citep{lenet} (62K parameters) for MNIST. 
For both IID and non-IID settings, we use one local update per client per round (i.e., unbiased gradient estimator), except for GPT-Neo, where memory constraints require two local epochs with a batch size of 8. In the biased setting (multiple local updates per round), we use a batch size of 16 for IMDB and 64 for CIFAR-10 and MNIST under IID conditions.  In all settings, each client reserves 30\% of its local data for evaluation. To ensure fair comparison of communication costs—which directly depend on the number of rounds—we cap the total rounds for all baselines at the point where FedAvg converges, determined by the minimum validation loss (generally the first to converge). This results in 2-4 rounds for CNN/DailyMail, 14–22 for IMDB, 80–140 for CIFAR-10, and 120–250 for MNIST in the unbiased setting. In the biased setting (two local epochs per round), the ranges are 4–16 for IMDB, 60–140 for CIFAR-10, and 80–200 for MNIST. 
Unless otherwise specified, we set $A=K$. Consequently, a fixed non-colluding aggregator observes one shard from each client per round, corresponding to $1/A=10\%$, $4\%$, $2\%$, and $2\%$ of the uncompressed update for CNN/DailyMail, IMDB, CIFAR-10, and MNIST, respectively. With DSC, the observable fraction becomes $p/A$; under the compression rates used in the main experiments (Appendix~\ref{app:scalability}), this corresponds to approximately $0.1\%$, $0.0005\%$, $0.012\%$, and $0.066\%$, respectively.
We use a learning rate of $5\text{e}{-5}$ for CNN/DailyMail and IMDB, and 0.01 for CIFAR-10 and MNIST. For optimization, we adopt Adam~\citep{kingmaAdam2017} (with $\texttt{weight\_decay}=0.0$, $\beta_1 = 0.9$, $\beta_2 = 0.999$, and $\epsilon = 1\text{e}{-8}$) on CNN/DailyMail and IMDB, and SGD~\citep{robbinsSDG1951} with momentum 0.9 for CIFAR-10 and MNIST. For experiments involving differential privacy, we use the Opacus library~\citep{opacus}.

\subsection{Implementation Details of Privacy Attacks} \label{app:privacy_attack_details}
We evaluate privacy leakage under the standard \emph{honest-but-curious} threat model, where an adversary (e.g., a compromised aggregator or server) can observe all transmitted model updates derived from each client's private dataset. We implement two widely studied categories of attacks: \textit{Membership Inference Attacks (MIA)} \citep{memb_inf_attack, zhangMIA2023, heEMIA2024, zariMIA2021, liPassiveMIA2022} and \textit{Data Reconstruction Attacks (DRA)} \citep{data_rec_gan, FCleak_label, gradinv, grnn, data_rec_gan2, dimitrovDataLeakage2022, zhangGradInv2023}.

\textit{Membership Inference Attacks.}
We adopt a distributed variant of the privacy auditing framework of \citet{steinkePrivacyAuditOne2023}. 
For each client, 50\% of the local data pool is designated as \emph{canary} samples and randomly split equally between samples included in and excluded from local training. After training, canaries are ranked by model confidence or gradient alignment; the top third are labeled as “in,” the bottom third as “out,” while the middle third are discarded to mitigate uncertainty bias. Evaluation is repeated on the same canary sets across all methods and folds of the cross-validation. To capture privacy leakage throughout training, MIA accuracy is computed at each round and for each client; the reported score corresponds to the maximum, over all $T$ rounds, of the average accuracy across $K$ clients. This ensures comparability across methods with different convergence speeds.

\textit{Data Reconstruction Attacks.}
For DRA, we adopt the strongest white-box threat model, where the adversary is assumed to access the gradient of a single training sample. We implement four representative gradient inversion methods: DLG \citep{DLG}, iDLG \citep{FCleak_label}, ROG \citep{yue_rog_attack2023}, and GGL \cite{liGGL2022a}. ROG is specifically tailored to reconstruct images from obfuscated gradients, while GGL further strengthens gradient inversion by using a generative prior and an adaptive loss that accounts for the gradient transformation induced by privacy defenses. All methods are evaluated on the same subset of 200 randomly sampled data points within each cross-validation fold to ensure fairness. 
For GGL, we use a fixed optimization budget of 300 steps per reconstruction trial. Reconstruction quality is assessed with LPIPS, 
capturing perceptual similarity. 
Further algorithmic descriptions on each attack are provided in Appendix~\ref{app:dra}.

\subsection{Licenses and Hardware} \label{app:code_licenses_hardware}
All experiments were implemented in Python 3.13 using open-source libraries: PyTorch 2.6 \citep{Paszke2019PyTorchAI} (BSD license), Flower 1.12 \citep{beutel2020flower} (Apache 2.0), Matplotlib 3.10 \citep{Hunter:2007} (BSD), Opacus 1.5 \citep{opacus} (Apache 2.0) and Pandas 2.2 \citep{pandas} (BSD). 
We use only publicly available datasets and pretrained models and follow their respective licenses and terms of use. Specifically, we use MNIST (GNU license), CIFAR-10, IMDB (subject to IMDb's Terms of Use), CNN/DailyMail (Apache License 2.0), LFW, and ImageNet. Where an explicit open-source license is not provided by the dataset distributor, we follow the corresponding dataset-specific access and usage terms. The pretrained GPT-Neo and DistilBERT checkpoints are obtained from their publicly available HuggingFace repositories and used according to their respective model licenses. All datasets, models, and software packages are credited through their original references in the paper.
The complete codebase and instructions for reproducing all experiments are available on GitHub\footnote{\href{https://github.com/dariofenoglio98/eris}{https://github.com/dariofenoglio98/eris}} under the MIT license.  
Publicly available implementations were used to reimplement all baselines, except for FedAvg and Min. Leakage, which we implemented directly using Flower Library \citep{beutel2020flower}. We follow the recommended hyperparameters for baselines, setting the compression ratio of SoteriaFL to 5\% and the graph degree of Shatter to 4.

Experiments were run on a workstation with four NVIDIA RTX A6000 GPUs (48 GB each), dual AMD EPYC 7513 32-core CPUs, and 512 GB RAM.

\section{Additional Experiments and Analysis} \label{app:add_exps}
This section presents complementary experiments and empirical validations that reinforce the theoretical claims made in the main paper. We inspect the distributional properties of model weights alongside the Gaussian example in our privacy analysis, evaluate the scalability of \textsc{eris} (with and without DSC) through distribution time comparisons, and further assess its robustness against data reconstruction attacks. Additionally, we provide detailed utility–privacy trade-off results under both IID and non-IID settings, and with unbiased and biased gradient estimators across multiple datasets and training configurations.

\subsection{Empirical Validation and Operational Estimation of $C_{\max}$}
\label{app:weight_distribution}
Remark~\ref{remark:privacy_remark} provides a Gaussian-channel instantiation of the per-coordinate quantity entering $C_{\max}$. As a distributional sanity check, we first inspect model-weight histograms during training. Figure~\ref{fig:weight_distribution} reports these distributions for three representative models---DistilBERT on IMDB, ResNet-9 on CIFAR-10, and LeNet-5 on MNIST. Each 3-D panel shows weight value (x-axis), training round (depth), and frequency (z-axis). Across datasets and training rounds, the distributions remain approximately Gaussian-shaped, with relatively stable variance and standard deviation below $0.2$ throughout. These cross-coordinate histograms do not establish the conditional Gaussian or independent-noise assumptions of Remark~\ref{remark:privacy_remark}; the quantitative plug-in estimates below are obtained from repeated stochastic updates under that model.
\begin{figure*}[h!]
  \centering
  \includegraphics[width=1.\textwidth]{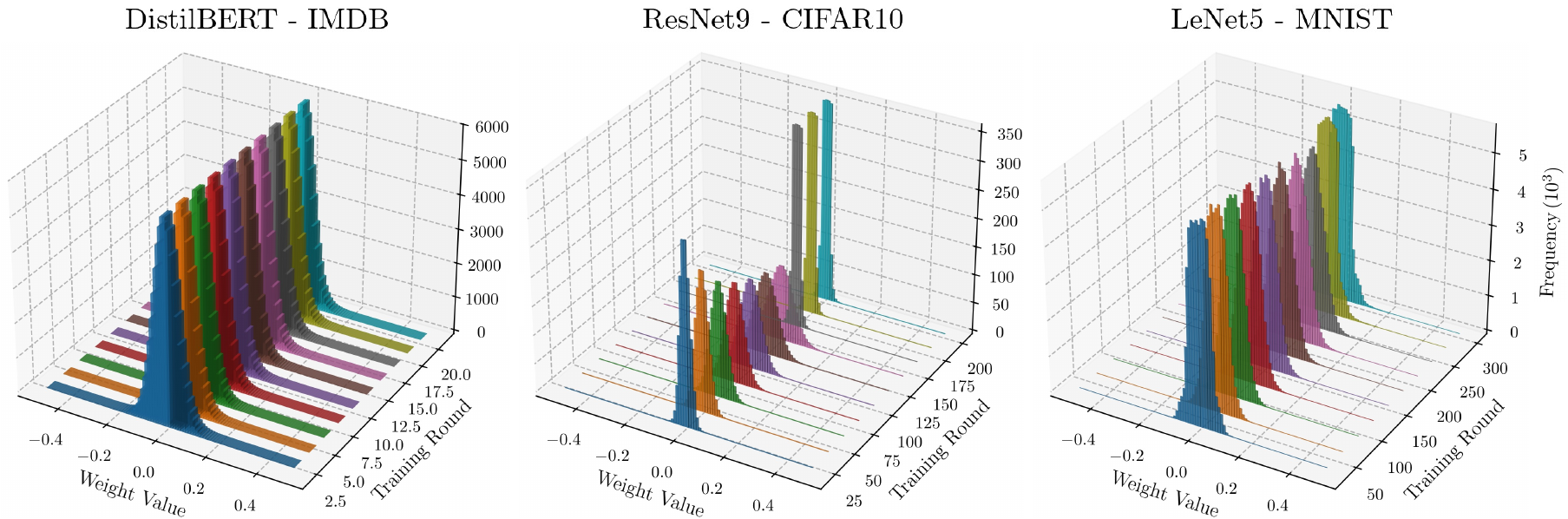}
\caption{
Weight-value histograms across coordinates
over training rounds for DistilBERT, ResNet-9, and LeNet-5.
Each 3-D plot shows the distribution of weight values (horizontal axis)
over time (depth axis), with frequency on the vertical axis.
In all cases, the distributions remain approximately Gaussian-shaped,
with empirical standard deviation below $0.2$ throughout
training. 
}
  \label{fig:weight_distribution}
\end{figure*}

\paragraph{Empirical audit of $C_{\max}$.}
We apply the estimator from Appendix~\ref{app:d1} to MNIST with LeNet-5, using $K=50$ clients, 64 samples per client, $T=250$ training rounds, and all $n=61{,}706$ model parameters. We audit the training states immediately before rounds $1$, $125$, and $250$. At each checkpoint, we sample $M=40$ independent client datasets and compute $R=20$ stochastic minibatch updates per dataset with batch size 16.

Across the resulting $3n=185{,}118$ audited coordinate--round pairs, the estimated per-coordinate information values are
\[
\operatorname{median}\!\left(\widehat{C}_{i,t}\right)
=
0.1774~\text{bits},
\qquad
Q_{0.99}\!\left(\widehat{C}_{i,t}\right)
=
0.3315~\text{bits},
\qquad
\widehat{C}_{\max}^{\mathrm{obs}}
=
0.5172~\text{bits}.
\]
Here, $\widehat{C}_{\max}^{\mathrm{obs}}$ denotes the largest value observed in the audit rather than the supremum over every possible dataset and training history.

For the ERIS configuration evaluated in this setting, $A=50$ and $\omega=250$, corresponding to a retention probability
\(
p=\frac{1}{1+\omega}=\frac{1}{251}.
\)
Hence, the expected number of exposed coordinate--round pairs appearing in Equation~\ref{eq:privay_bound} is
\[
\frac{nTp}{A}
=
\frac{61{,}706\cdot250}{50\cdot251}
\approx
1{,}229.20.
\]
Under the independent-noise Gaussian model of Remark~\ref{remark:privacy_remark}, using the largest observed per-coordinate value gives the empirical plug-in evaluation
\[
\frac{nTp}{A}\,
\widehat{C}_{\max}^{\mathrm{obs}}
\approx
635.7~\text{bits}.
\]
This model-based quantity concerns information about the complete 64-sample client dataset. The plug-in composition assigns the largest observed value to every exposed coordinate in every round; however, $\widehat{C}_{\max}^{\mathrm{obs}}$ itself is empirical and does not certify the worst-case value of $C_{\max}$ outside the audited states. The substantially smaller median and $99$th-percentile estimates further show that the plug-in composition is dominated by the most informative audited coordinates.

Finally, for a common $C_{\max}$, relative to observing every coordinate at every round, ERIS reduces the corresponding mutual-information bound by
\(
\frac{A}{p}
=
12{,}550\times.
\)
The audit therefore provides a model-based plug-in scale, rather than a certified absolute leakage bound. For a common $C_{\max}$, FSA and DSC scale the corresponding upper bound by $p/A$; the absolute leakage remains determined by the underlying model, data distribution, and training dynamics.


\subsection{Scalability and Efficiency of \textsc{eris}} \label{app:scalability}
To evaluate the scalability and communication efficiency of \textsc{eris}, we provide both a theoretical analysis of model distribution time and empirical comparisons with existing FL frameworks.

\subsubsection{Theoretical Analysis of Model Distribution Time}
We begin by quantifying the minimum time required to distribute models in a single training round under various FL setups. Here, the distribution time refers to the time needed for: (i) clients to transmit their local models to the aggregation parties (either a central server or a set of aggregators), and (ii) all clients to receive the updated global model. For clarity, we assume full client participation in each round; however, the same analysis readily extends to partial participation scenarios by adjusting the number of active clients accordingly. 

\paragraph{Single-server Federated Learning.} In traditional centralized FL, we consider a single server and $K$ clients. Let $u_s$ and $d_s$ denote the server’s upload and download rates, and let $u_k$, $d_k$ be the $k$-th client’s upload and download rates, respectively. Assume the model has $n$ parameters and each is represented as a 32-bit float, yielding a total model size of $b \approx 32 \cdot n$ bits. The distribution time in a single training round is governed by the following observations:

\begin{itemize}[noitemsep, topsep=0pt, parsep=2pt, partopsep=0pt, left=2em]
\item The server must collect $K$ local models, each of size $b$ bits, resulting in a total inbound traffic of $K \cdot b$ bits, received at a download rate $d_s$.
\item Each client $k$ uploads its local model at an individual rate $u_k$. The server cannot complete the upload phase until the slowest client—i.e., the one with the lowest $u_k$—has finished its transmission.
\item Once all local models are received, the server performs aggregation and then broadcasts the aggregated global model back to all $K$ clients. This requires transmitting another $K \cdot b$ bits at the server's upload rate $u_s$.
\item Model distribution concludes when every client has received the global model. This process is bounded by the client with the lowest download rate $d_k$, as it determines the last completed transfer.
\end{itemize}

Putting all these observations together, we derive the minimum distribution time in a single training round for a centralized FL setup without compression such as FedAvg, denoted by $D_{FedAvg}$.

\begin{equation}
D_{FedAvg} \geq \max\left\{\frac{K \cdot b}{d_s}, \frac{b}{\min\{u_1, \ldots, u_K\}}\right\} + \max\left\{\frac{K \cdot b}{u_s}, \frac{b}{\min\{d_1, \ldots, d_K\}}\right\}
\end{equation}

Here, the first term captures the server’s time to receive all local model uploads and the slowest client’s upload time, while the second term captures the server’s model broadcast time and the slowest client’s download time. 

To reduce distribution time, several FL methods focus on minimizing the volume of transmitted data per round, i.e., decreasing the effective model size $b$. For example, \textit{PriPrune}~\citep{priprune2024} applies structured pruning to eliminate a fraction $p$ of the model parameters before transmission, reducing the transmitted size to $b^{\prime} \leq 32 \cdot (1-p) \cdot n$ bits. Similarly, \textit{SoteriaFL}~\citep{soteriafl2022} compresses gradients using a shifting operator controlled by a compression factor $\omega$, leading to $b^{\prime} \leq 32 \cdot \frac{1}{\omega+1} \cdot n$ bits. In principle, \textsc{eris} could also compress both upload and download communication, e.g., by coordinating compression masks across clients so that the aggregated update remains compressed. However, for a fair comparison with baselines such as SoteriaFL and PriPrune, we adopt a worst-case communication setting in which clients use independent compression masks. As a result, upload communication is compressed, but the aggregated model shards returned by the aggregators are treated as uncompressed in the download phase.

\paragraph{Federated Shard Aggregation.}
We now extend the analysis to \textsc{eris}, where FSA distributes aggregation across $A \leq K$ aggregator nodes. Unlike centralized schemes, each client update is split into $A$ disjoint shards, and each aggregator handles only its assigned shard. Let $b'$ denote the transmitted update size: for FSA without DSC, $b'=b$, while with DSC, $b' \leq 32 \cdot \frac{1}{\omega+1} \cdot n$. To estimate the minimum distribution time in a single training round under \textsc{eris}, we consider the following:
\begin{itemize}[noitemsep, topsep=0pt, parsep=2pt, partopsep=0pt, left=2em]
\item Each aggregator must collect $K-1$ update shards from the clients, excluding its own, amounting to $(K-1)\cdot \frac{b'}{A}$ bits received per aggregator at download rate $d_k$. The aggregation process cannot proceed before the slowest aggregator receives all required shards.
\item Each client $k$ uploads one shard to each aggregator, sending a total of $b{\prime}$ bits. If the client is not serving as an aggregator (worst case), it must upload the entire set of $A$ shards at an upload rate $u_k$. The aggregation step is gated by the client with the lowest upload rate.
\item Once aggregation is complete, each aggregator redistributes its updated model shard to all $K$ clients. This amounts to sending $(K-1)\cdot \frac{b}{A}$ bits at upload rate $u_k$, constrained by the aggregator with the slowest upload speed.
\item Full model reconstruction occurs only after each client receives one shard from every aggregator. In the worst case, a non-aggregator client must download the complete updated model, i.e., $b$ bits, at rate $d_k$. The distribution concludes when the slowest client completes this transfer.
\end{itemize}

Putting all these observations together, we derive the minimum
distribution time in a single training round for \textsc{eris},
denoted by $D_{\textsc{eris}}$:

\begin{align}
D_{\textsc{eris}} \geq & \max\left\{\frac{(K-1)b'}{A \cdot \min\{d_1, \ldots, d_A\}}, \frac{b'}{\min\{u_1, \ldots, u_K\}}\right\} \nonumber \\
&+ \max\left\{\frac{(K - 1)b}{A \cdot \min\{u_1, \ldots, u_A\}}, \frac{b}{\min\{d_1, \ldots, d_K\}}\right\}
\end{align}

\paragraph{Decentralized baselines.}
While FSA distributes aggregation across multiple nodes, it differs from prior decentralized approaches because every round still recovers the full collaborative update after reassembly. We compare against two representative decentralized baselines. 

\textit{Ako}~\citep{ako2016} distributes gradient computations by splitting each model into $v$ 
disjoint partitions and randomly assigning them to worker nodes. However, this approach differs fundamentally from \textsc{eris}: in Ako, not all clients receive the full model in each round, which may hinder convergence to the standard FedAvg solution. Furthermore, in each round, a client uploads and receives $K$ partitions, equivalent to the full model size, resulting in substantial bandwidth usage. We can estimate the minimum distribution time for Ako, denoted by $D_{Ako}$, using a similar worst-case analysis:
\begin{equation}
 D_{Ako} \geq \max\left\{\frac{b}{\min\{d_1, \ldots, d_K\}}, \frac{b}{\min\{u_1, \ldots, u_K\}}\right\}
\end{equation}

\emph{Shatter}~\citep{biswasShatter2025} is also a privacy preserving distributed learning framework. In Shatter, each round consists of three steps. In the first step, each node updates its local model and divides the result in $l$ chunks.  Each client (real node) runs $l$ virtual nodes. The virtual nodes form an overlay network over which model parameter updates are multicast with a gossiping protocol. Once received $r$ updates for each of the $l$ virtual nodes running on a real node, the last step consists of the virtual nodes forwarding the received updates to the real node to perform the aggregation. Notice that in this setup, not all clients will receive all model updates, so the communication overhead is reduced at the cost of slower global model convergence. The model distribution occurs in the second step through a multicast among the virtual nodes. The model distribution cannot finish before each real node (via its virtual nodes) has finished uploading its model chunks, and has finished downloading model updates from $r$ other clients. Then, we need to account for the time to upload all the model updates with the total upload capacity being the sum of all individual node upload rates.
Therefore, we can estimate the minimum distribution time for Shatter, denoted by $D_{\text{Shatter}}$:

\begin{equation}
 D_{\text{Shatter}} \geq \max\left\{\frac{b}{\min\{u_1, \ldots, u_K\}}, \frac{r\cdot b}{\min\{d_1, \ldots, d_K\}}, \frac{r\cdot b}{\sum_{i=1}^K u_i}\right\}
\end{equation}

\subsubsection{Numerical Results}
\paragraph{Effect of Number of Clients and Model Size on Distribution Time.}
Figure~\ref{fig:distribution-time} compares the minimum distribution time per training round for \textsc{eris} (with and without DSC) and other FL frameworks under varying numbers of aggregators and model sizes. We assume homogeneous network conditions across all nodes, with upload and download rates fixed at 100 Mbps. For the baselines, we apply a pruning rate of 0.3 for PriPrune, a compression ratio of $1/(\omega+1) = 0.05$ for SoteriaFL and \textsc{eris} with DSC, and the overlay topology (i.e., graph degree) for Shatter forms a 4-regular graph. Note that $1/(\omega+1) = 0.05$ (with $\omega=18$) corresponds to the least aggressive compression used in our experiments, chosen for MNIST to preserve model utility. In other settings, such as IMDB, we adopt much stronger compression (e.g., $1/(\omega+1) = 0.00012$), further lowering communication overhead without harming performance. Thus, the results in Figure~\ref{fig:distribution-time} represent a conservative estimate of \textsc{eris}'s efficiency. 

On the left, Figure~\ref{fig:distribution-time} shows how distribution time scales with the number of participating clients: while all methods experience linear growth (except Ako and Shatter, which always exchange with a fixed number of neighbors), \textsc{eris} benefits from increased decentralization—achieving lower distribution times as the number of aggregators $A$ increases. In the worst-case setting ($A\!=\!2$), \textsc{eris} still achieves a 4$\times$ speedup over FedAvg and a 2$\times$ improvement over SoteriaFL. When $A\!=\!50$, these gains rise dramatically to 105$\times$ and 55$\times$, respectively, underscoring the scalability advantages of decentralized aggregation. Maximum efficiency is achieved when the number of aggregators matches the number of clients, maximizing parallelism and evenly distributing the communication load. Notably, Ako and Shatter remain constant with respect to the number of clients, as their communication pattern does not involve distributing the full model to all participants—at the expense of model consistency and convergence guarantees. On the right, Figure~\ref{fig:distribution-time} examines the impact of increasing model size with 50 clients. The results highlight the communication efficiency of decentralized approaches, especially \textsc{eris}, which outperforms traditional centralized frameworks as model size grows. This confirms the practical benefits of combining decentralized aggregation with compression.
\begin{figure*}[ht!]
    \captionsetup{skip=3pt}
    \centering
    \includegraphics[scale=0.55]{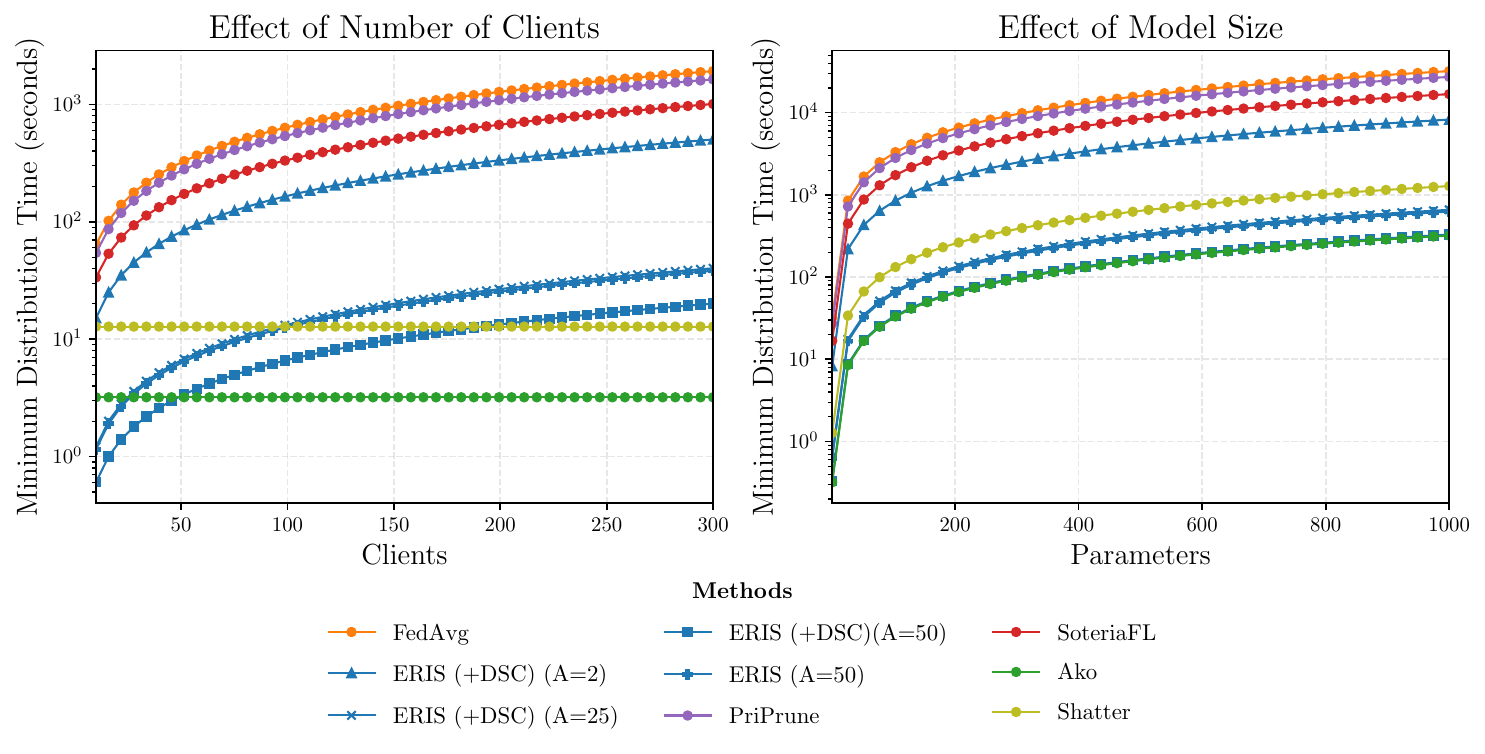}
    \caption{Minimum distribution time for a single training round for FedAvg, PriPrune, SoteriaFL, Ako, Shatter, and \textsc{eris}. The figure shows the minimum distribution time for a single round with $M = 320$ Mbit on a logarithmic scale (\textbf{left}), and the minimum distribution time for a single round with 50 training clients as the model size increases on a logarithmic scale (\textbf{right}).}
    \label{fig:distribution-time}
\end{figure*}

\paragraph{Effect of Transmission Rate.}
We further evaluate the sensitivity of \textsc{eris} to degraded communication links by varying the effective transmission rate while measuring the minimum per-round distribution time. This analysis complements Figure~\ref{fig:distribution-time}, which assumes fixed homogeneous bandwidth, by explicitly testing slower network conditions. As shown in Figure~\ref{fig:transmission-rate}, all methods exhibit the expected inverse relationship between transmission rate and distribution time. However, \textsc{eris} consistently remains more communication-efficient across the full range of rates. This advantage comes from two complementary properties: each communication link transmits only a model shard rather than a full update, and communication is parallelized across multiple aggregators. Increasing the number of aggregators further reduces latency by distributing the communication load more evenly, while DSC provides an additional reduction in the transmitted payload. Notably, even under substantially reduced transmission rates, \textsc{eris} with sufficiently many aggregators remains faster than the strongest baselines operating under more favorable bandwidth conditions.
\begin{figure*}[ht!]
    \captionsetup{skip=3pt}
    \centering
    \includegraphics[scale=0.55]{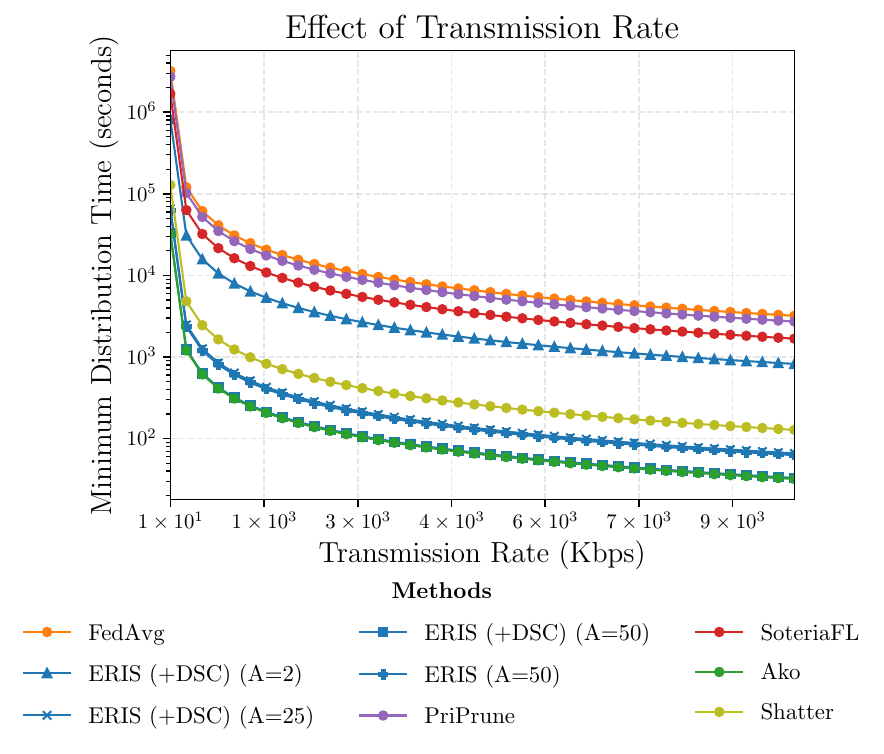}
    \caption{Effect of transmission rate on minimum per-round distribution time for FedAvg, PriPrune, SoteriaFL, Ako, Shatter, and \textsc{eris}. \textsc{eris} remains faster across transmission rates due to shard-based communication and distributed aggregation, with further gains from larger $A$ and DSC.}
    \label{fig:transmission-rate}
\end{figure*}

\paragraph{Communication efficiency.}
In addition to the main paper, we provide a detailed communication comparison across a larger set of FL and decentralized FL (DFL) baselines: FedAvg, Shatter, Ako, Q-DPSGD-1, PriPrune (with multiple pruning rates), SoteriaFL, and \textsc{eris} variants. Tables~\ref{tab:comm_full_2x2_a}--\ref{tab:comm_full_2x2_b} report, for each dataset/model pair, the compression ratio, per-client upload/download volume, total per-round communication per client, and the distribution time per round (20MB/s bandwidth), under the same experimental settings as in Table~\ref{tab:overall_acc_mia_all}.

Among decentralized methods, Shatter and Ako rely on model partitioning, whereas Q-DPSGD-1 reduces communication via quantization. We additionally include PriPrune at three pruning rates and SoteriaFL at 5\% compression. For \textsc{eris} with DSC, we report (i) a matched-compression setting with SoteriaFL (same ratio) to isolate the benefit of FSA and (ii) the more aggressive compression regime used in our main experiments. Overall, Tables~\ref{tab:comm_full_2x2_a}--\ref{tab:comm_full_2x2_b} show that, for a fixed compression ratio, \textsc{eris} with DSC consistently achieves substantially lower distribution time than both centralized and decentralized baselines. In addition, \textsc{eris} with DSC enables markedly more aggressive compression while preserving accuracy, as reported in the main text.

\begin{table*}[t]
  \captionsetup{skip=4pt}
  \caption{Direct communication comparison (per-client) across FL/DFL baselines.}
  \label{tab:comm_full_2x2_a}
  \centering
  \scriptsize
  \renewcommand{\arraystretch}{1.15}
  \setlength{\tabcolsep}{2.2pt}
  \begin{tabular}{lcccccccccc}
    \toprule
    \multirow{2}{*}{\textbf{Method}} &
    \multicolumn{5}{c}{\textbf{CNN/Daily Mail — GPT-Neo (1.3B)}} &
    \multicolumn{5}{c}{\textbf{IMDB — DistilBERT (69M)}} \\
    \cmidrule(lr){2-6} \cmidrule(lr){7-11}
    & \textbf{Comp.} & \textbf{Upload} & \textbf{Down.} & \textbf{Tot.} & \textbf{Time}
    & \textbf{Comp.} & \textbf{Upload} & \textbf{Down.} & \textbf{Tot.} & \textbf{Time} \\
    \midrule
    FedAvg                  & 100\%  & 5.2 GB    & 5.2 GB   & 10.4 GB  & 5200.0 s
                            & 100\%  & 268.0 MB  & 268.0 MB & 536.0 MB & 670.0 s  \\
    Shatter                 & 100\%  & 5.2 GB    & 5.2 GB   & 36.4 GB  & 780.0 s
                            & 100\%  & 268.0 MB  & 268.0 MB & 2.41 GB  & 53.6 s \\
    Ako ($p{=}5$)           & 100\%  & 9.36 GB   & 9.36 GB  & 18.72 GB & 936.0 s
                            & 100\%  & 1.29 GB   & 1.29 GB  & 2.57 GB  &  128.64 s \\
    Q-DPSGD-1 ($K_n{=}0.4$) & 18.8\% & 3.51 GB   & 3.51 GB  & 37.44 GB & 351.0 s
                            & 18.8\% & 482.4 MB  & 482.4 MB & 5.15 GB  & 48.24 s \\
    PriPrune (0.01)         & 90\%   & 4.68 GB   & 5.2 GB   & 9.88 GB  & 4940.0 s
                            & 90\%   & 241.2 MB  & 268.0 MB & 509.2 MB & 636.5 s \\
    PriPrune (0.05)         & 80\%   & 4.16 GB   & 5.2 GB   & 9.36 GB  & 4680.0 s
                            & 80\%   & 214.4 MB  & 268.0 MB & 482.4 MB & 603.0 s   \\
    PriPrune (0.1)          & 70\%   & 3.64 GB   & 5.2 GB   & 8.84 GB  & 4420.0 s  
                            & 70\%   & 187.6 MB  & 268.0 MB & 455.6 MB & 569.5 s   \\
    SoteriaFL               & 5\%    & 260.0 MB  & 5.2 GB   & 5.46 GB  & 2730.0 s  
                            & 5\%    & 13.4 MB   & 268.0 MB & 281.4 MB & 351.75 s  \\
    \textsc{eris}
                            & 100\%  & 4.68 GB   & 4.68 GB  & 9.36 GB  & 468.0 s   
                            & 100\%  & 257.28 MB & 257.28 MB & 514.56 MB & 25.73 s    \\  
    \textsc{eris} (+\textsc{dsc} with $\omega_{\text{SoteriaFL}}$)
                            & 5\%    & 234.0 MB  & 4.68 GB  & 4.91 GB  & 245.7 s   
                            & 5\%    & 12.86 MB  & 257.28 MB & 270.14 MB & 13.51 s   \\
    \textsc{eris} (+\textsc{dsc})
                            & 1\%    & 46.8 MB   & 4.68 GB  & 4.73 GB  & 236.34 s  
                            & 0.012\% & 30.87 KB & 257.28 MB & 257.31 MB  & 12.87 s   \\
    \bottomrule
  \end{tabular}
  \label{tab:comm_efficiency_all_1}
\end{table*}

\begin{table*}[t]
  \captionsetup{skip=4pt}
  \caption{Direct communication comparison (per-client) across FL/DFL baselines (continued).}
  \label{tab:comm_full_2x2_b}
  \centering
  \scriptsize
  \renewcommand{\arraystretch}{1.15}
  \setlength{\tabcolsep}{2.2pt}
  \begin{tabular}{lcccccccccc}
    \toprule
    \multirow{2}{*}{\textbf{Method}} &
    \multicolumn{5}{c}{\textbf{CIFAR-10 — ResNet9 (1.65M)}} &
    \multicolumn{5}{c}{\textbf{MNIST — LeNet5 (62K)}} \\
    \cmidrule(lr){2-6} \cmidrule(lr){7-11}
    & \textbf{Comp.} & \textbf{Upload} & \textbf{Down.} & \textbf{Tot.} & \textbf{Time}
    & \textbf{Comp.} & \textbf{Upload} & \textbf{Down.} & \textbf{Tot.} & \textbf{Time} \\
    \midrule
    FedAvg                  & 100\%  & 6.6 MB    & 6.6 MB   & 13.2 MB   & 33.0 s
                            & 100\%  & 248.0 KB  & 248.0 KB & 496.0 KB  & 1.24 s \\
    Shatter                 & 100\%  & 6.6 MB    & 6.6 MB   & 59.4 MB   & 1.32 s
                            & 100\%  & 248.0 KB  & 248.0 KB & 2.23 MB   & 0.05 s \\
    Ako ($p{=}5$)           & 100\%  & 64.68 MB  & 64.68 MB & 129.36 MB & 6.47 s
                            & 100\%  & 2.43 MB   & 2.43 MB  & 4.86 MB   & 0.24 s \\
    Q-DPSGD-1 ($K_n{=}0.4$) & 18.8\% & 24.25 MB  & 24.25 MB & 258.72 MB & 2.43 s
                            & 18.8\% & 911.4 KB  & 911.4 KB & 9.72 MB   & 0.09 s    \\
    PriPrune (0.01)         & 99\%   & 6.53 MB   & 6.6 MB   & 13.13 MB  & 32.84 s
                            & 99\%   & 245.52 KB & 248.0 KB & 493.52 KB & 1.23 s   \\
    PriPrune (0.05)         & 95\%   & 6.27 MB   & 6.6 MB   & 12.87 MB  & 32.17 s
                            & 95\%   & 235.6 KB  & 248.0 KB & 483.6 KB  & 1.21 s   \\
    PriPrune (0.1)          & 90\%   & 5.94 MB   & 6.6 MB   & 12.54 MB  & 31.35 s
                            & 90\%   & 223.2 KB  & 248.0 KB & 471.2 KB  & 1.18 s    \\
    SoteriaFL               & 5\%    & 330.0 KB  & 6.6 MB   & 6.93 MB   & 17.32 s
                            & 5\%    & 12.4 KB   & 248.0 KB & 260.4 KB  & 0.65 s   \\
    \textsc{eris}
                            & 100\%  & 6.47 MB   & 6.47 MB  & 12.94 MB  & 0.65 s
                            & 100\%  & 243.04 KB & 243.04 KB & 486.08 KB & 0.02 s   \\ 
    \textsc{eris} (+\textsc{dsc} with $\omega_{\text{SoteriaFL}}$)
                            & 5\%    & 323.4 KB  & 6.47 MB  & 6.79 MB   & 0.34 s
                            & 5\%    & 12.15 KB  & 243.04 KB & 255.19 KB & 0.01 s   \\
    \textsc{eris}(+\textsc{dsc})
                            & 0.6\%  & 38.81 KB  & 6.47 MB  & 6.51 MB   & 0.33 s
                            & 3.3\%  & 8.02 KB   & 243.04 KB & 251.06 KB & 0.01 s   \\
    \bottomrule
  \end{tabular}
  \label{tab:comm_efficiency_all_2}
\end{table*}

\begin{figure}  
  \captionsetup{skip=3pt}
  \
  \centering
  \includegraphics[width=0.50\textwidth]{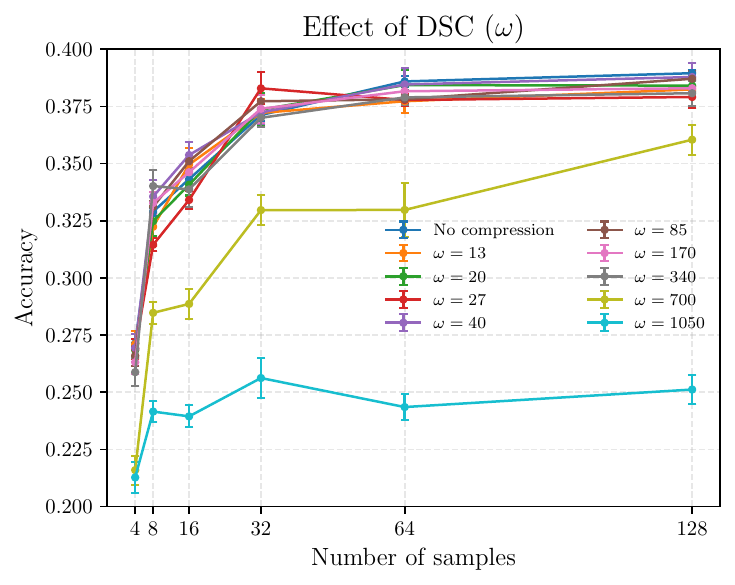}
  \caption{Effect of DSC on CIFAR-10, varying $\omega$ across different local training sample sizes.}
  \label{fig:shited_compr_app}
\end{figure}
\subsection{Effect of Distributed Shifted Compression on Model Utility} \label{app:shited_compr}
This section provides additional results complementing the analysis in Paragraph~\ref{pg:model_partitioning_compression}. Figure~\ref{fig:shited_compr_app} illustrates the impact of increasing the compression constant $\omega$ on test accuracy for CIFAR-10, under varying numbers of local training samples per client. We observe that up to $\omega = 340$—which corresponds to a compression rate of approximately 0.29\%—test accuracy remains statistically unchanged. This indicates that the communication cost can be substantially reduced, sharing only 0.29\% of gradients per client, without degrading performance. However, beyond this threshold, the aggressive compression starts discarding critical information, leading to compromised model convergence. As expected, further increasing $\omega$ results in progressively lower accuracy, highlighting a clear trade-off between compression strength and model utility.

\subsection{Effect of Local Data Size on the Utility--Privacy Trade-off} \label{app:effect_data_size}
The experiments reported in Table \ref{tab:overall_acc_mia_all} and in Appendix~\ref{app:balancing_ut_pri_iid}--\ref{app:pareto} focus on a low-data regime, where each client is assigned only from 4 to 128 local training samples. This setting was chosen deliberately to study privacy leakage under conditions that promote memorization and overfitting, which are known to amplify membership inference risk. As a result, these experiments are not intended as full-data benchmarks, but rather as controlled privacy--utility evaluations under limited local data.

To verify that the trends observed in this regime are not an artifact of unstable training or underfitting, we extend the CIFAR-10 analysis to larger local data sizes. Table~\ref{tab:cifar_more_data} reports test accuracy and MIA accuracy for 64, 128, 256, and 415 samples per client. The results show a stable and monotonic trend across all methods: increasing the number of local training samples consistently improves utility while reducing privacy leakage. This confirms that the original low-data setting reflects a deliberate operating regime for privacy auditing, rather than a failure of optimization.

Importantly, \textsc{eris} with DSC continues to closely track FedAvg in utility across all larger-data settings while maintaining substantially lower MIA leakage. For example, at 415 samples per client, FedAvg reaches 63.3\% accuracy, while \textsc{eris} attains 62.9\%, but reduces MIA accuracy from 82.2\% to 68.4\%. This indicates that the privacy advantages of \textsc{eris} persist even when the local data regime becomes substantially more favorable for utility. We note that the 415-sample setting corresponds to approximately 21k training samples used for federated optimization, while the remaining samples are reserved for validation and privacy auditing. This split is necessary to ensure a consistent and controlled evaluation of membership inference across methods.
\begin{table}[t]
\centering
\caption{CIFAR-10 in a larger-data regime: accuracy and MIA accuracy for increasing numbers of training samples per client. Results confirm a stable utility--privacy trend beyond the low-data regime.}
\label{tab:cifar_more_data}
\setlength{\tabcolsep}{1.7pt}
\resizebox{\textwidth}{!}{%
\begin{tabular}{lcccccccc}
\toprule
& \multicolumn{2}{c}{64 samples} 
& \multicolumn{2}{c}{128 samples} 
& \multicolumn{2}{c}{256 samples} 
& \multicolumn{2}{c}{415 samples} \\
\cmidrule(lr){2-3} 
\cmidrule(lr){4-5} 
\cmidrule(lr){6-7} 
\cmidrule(lr){8-9}
Method 
& Acc. ($\uparrow$) & MIA Acc. ($\downarrow$) 
& Acc. ($\uparrow$) & MIA Acc. ($\downarrow$) 
& Acc. ($\uparrow$) & MIA Acc. ($\downarrow$) 
& Acc. ($\uparrow$) & MIA Acc. ($\downarrow$) \\
\midrule
FedAvg     
& 46.4 $\pm$ 0.4 & 93.9 $\pm$ 0.7 
& 50.0 $\pm$ 0.6 & 89.9 $\pm$ 0.3
& 56.6 $\pm$ 0.5 & 85.9 $\pm$ 0.3 
& 63.3 $\pm$ 0.4 & 82.2 $\pm$ 0.1 \\

FedAvg+LDP 
& 27.2 $\pm$ 0.4 & 58.1 $\pm$ 0.2 
& 29.9 $\pm$ 0.2 & 56.4 $\pm$ 0.4
& 33.2 $\pm$ 0.3 & 54.5 $\pm$ 0.5 
& 36.3 $\pm$ 0.2 & 54.0 $\pm$ 0.8 \\

PriPrune   
& 45.0 $\pm$ 0.3 & 74.4 $\pm$ 0.8 
& 49.0 $\pm$ 0.5 & 74.1 $\pm$ 1.0
& 56.1 $\pm$ 0.5 & 78.1 $\pm$ 0.8 
& 63.1 $\pm$ 0.3 & 75.8 $\pm$ 0.3 \\

\textsc{eris} (+\textsc{dsc})       
& 46.2 $\pm$ 0.9 & 71.1 $\pm$ 0.8 
& 49.7 $\pm$ 0.7 & 69.8 $\pm$ 0.7
& 56.0 $\pm$ 0.7 & 68.6 $\pm$ 0.7 
& 62.9 $\pm$ 0.4 & 68.4 $\pm$ 0.5 \\
\midrule
Min.\ Leakage 
& 46.4 $\pm$ 0.3 & 72.6 $\pm$ 0.9 
& 49.9 $\pm$ 0.7 & 71.3 $\pm$ 0.8
& 57.0 $\pm$ 0.2 & 69.9 $\pm$ 0.5 
& 63.2 $\pm$ 0.3 & 67.9 $\pm$ 0.6 \\
\bottomrule
\end{tabular}%
}
\end{table}

\subsection{Robustness to Aggregator and Link Failures}
\label{app:aggregator_dropout}
We evaluate the robustness of \textsc{eris} under two failure modes: aggregator dropout and client--aggregator link failures. In both cases, failures remove a subset of model shards from the global update in a given round. However, because each shard corresponds only to a disjoint part of the model, failures do not invalidate the update; they primarily reduce its effective magnitude and therefore slow convergence.

\paragraph{Aggregator dropout.}
We first consider aggregator unavailability during training. 
At each round, a fixed proportion of aggregators is randomly deactivated and resampled independently across rounds, so their corresponding model shards are temporarily excluded from the global update. Consequently, coordinates omitted in one round can be updated again in subsequent rounds when their aggregator becomes available.
Figure~\ref{fig:agg_dropout} reports the effect of increasing dropout rates on both (i) test accuracy and (ii) the best validation round at which the model reaches its peak performance (i.e., the minimum validation loss). The results show that \textsc{eris} maintains nearly constant test accuracy up to a dropout rate of 70\%. The right plot explains this behavior: as dropout increases, the best validation round shifts steadily toward the 200-round training cap, indicating a convergence slowdown rather than algorithmic instability. Once the slowdown becomes too large, the model no longer reaches the same optimum within the fixed training budget, and accuracy drops.
\begin{figure*}[h!]
  \captionsetup{skip=3pt}
  \centering
  \includegraphics[width=.9\textwidth]{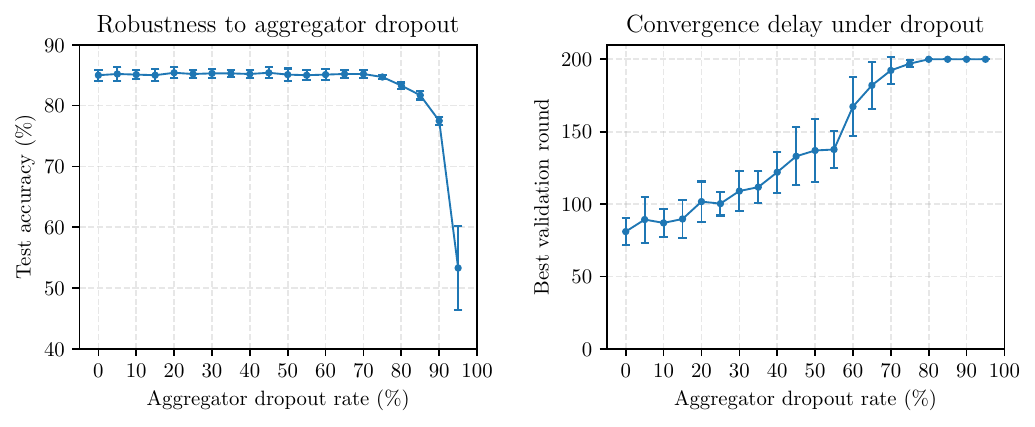}
  \caption{Robustness of \textsc{eris} to aggregator dropout. Left: test accuracy remains nearly constant up to a 70\% dropout rate. Right: convergence slows as dropout increases, eventually hitting the 200-round cap, which explains the accuracy drop beyond 70\%.} 
  \label{fig:agg_dropout}
  \vskip -0.1in   
\end{figure*}

\paragraph{Client--aggregator link failures.}
We further evaluate random failures of individual client--aggregator communication links. When such a link fails, only the corresponding shard contribution is lost for that round. Figure~\ref{fig:link_failures} shows a similar pattern to aggregator dropout: accuracy remains nearly unchanged up to moderate failure rates, while the best validation round progressively approaches the 200-round cap. In particular, accuracy remains close to the no-failure setting up to roughly 50\% link failures, whereas larger failure rates mainly hurt performance because training cannot fully converge within the fixed round budget. This contrasts with centralized FL, where a failed client--server link discards the entire client update; in \textsc{eris}, a failed link removes only one shard contribution.
\begin{figure*}[h!]
  \captionsetup{skip=3pt}
  \centering
  \includegraphics[width=.9\textwidth]{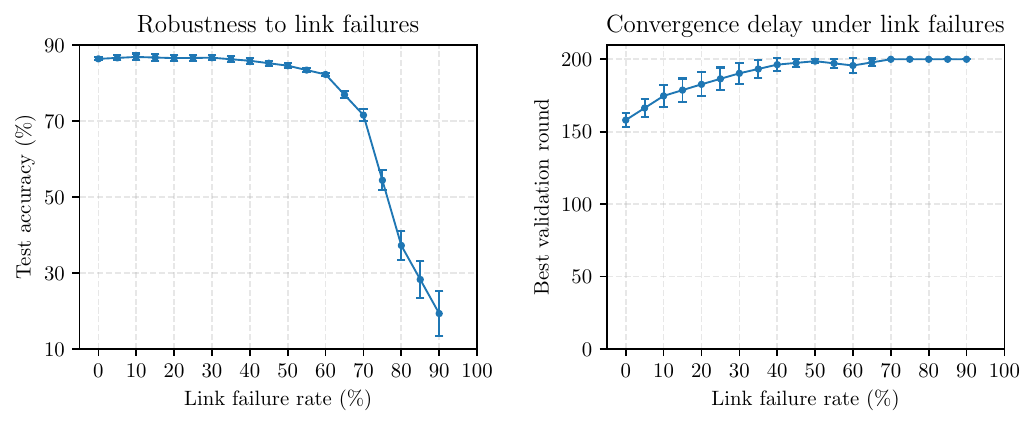}
  \caption{Robustness of \textsc{eris} to client--aggregator link failures. Left: accuracy remains close to the no-failure setting up to 50\% link failures. Right: convergence slows as link failures increase, with the best round approaching the 200-round cap, which explains the accuracy drop at higher failure rates.}
  \label{fig:link_failures}
  \vskip -0.1in
\end{figure*}

Together, these results show that \textsc{eris} degrades gracefully under both node- and edge-level failures. The system remains stable as long as a sufficient fraction of shards is aggregated, and performance degradation is primarily caused by convergence delay under a fixed round cap rather than by instability of the sharded aggregation mechanism. Thus, the robustness to random aggregator dropout or link failures should not be interpreted as indicating that the missing shards or shard contributions are unnecessary: all shards remain required to recover the exact centralized update at each round, while random failures only postpone their contribution to the optimization process.

\subsection{Data Reconstruction Attacks} \label{app:dra}
To further assess the privacy guarantees of \textsc{eris}, we evaluate its resilience to Data Reconstruction Attacks (DRAs), which represent one of the most severe privacy threats in FL. To favour the attacker and stress-test our approach, we consider the uncommon but worst-case scenario where each client performs gradient descent with a mini-batch of size 1 and transmits the resulting gradient—which can be intercepted by an eavesdropper or a compromised aggregator/server. Therefore, we assume the adversary has white-box access to the client gradient.

Given this gradient, reconstruction methods such as DLG~\citep{DLG}, iDLG~\citep{FCleak_label}, and ROG \citep{yue_rog_attack2023} aim to recover the original training sample by optimizing candidate inputs to match the leaked gradient. Unlike earlier gradient-matching attacks, ROG projects the unknown image into a low-dimensional latent space (e.g., via bicubic downsampling or an autoencoder) and optimizes that compact representation so that the decoded image’s gradients align with the leaked gradient, before applying a learned enhancement module to obtain perceptually faithful reconstructions. In our experiments, a dedicated enhancement decoder was trained for each dataset using a hold-out set. GGL instead constrains the reconstruction to the latent space of a pretrained generative model and incorporates an adaptive gradient-matching objective that accounts for the transformation induced by privacy defenses. This makes GGL particularly suitable for auditing degraded gradients, such as those produced by compression, sparsification, clipping, or additive noise.

Figure~\ref{fig:rec_lpips} reports the reconstruction quality, measured via the LPIPS score \citep{LPIPS2018}, as a function of the percentage of model parameters available to the attacker. The x-axis is plotted on a non-linear scale to improve readability in the low-percentage regime. The results are averaged over 200 reconstructed samples and tested across three datasets: MNIST, CIFAR-10, and LFW. The findings show that in the full-gradient setting (e.g., FedAvg), all DRAs can almost perfectly reconstruct the original image. However, as the proportion of accessible gradients decreases, the reconstruction quality of DLG and iDLG degrades significantly, with LPIPS scores approaching the baseline of random images when only 90\% of the parameters are visible. Remarkably, even in the least favourable configuration of \textsc{eris}—with only two aggregators—the system already provides sufficient obfuscation to render reconstruction attacks ineffective, as highlighted by the shaded regions in the figure. A different pattern emerges for ROG, which tends to maintain higher reconstruction quality. Closer inspection of MNIST and LFW examples, however, reveals that this apparent advantage stems primarily from the trained enhancement decoder. This module effectively biases reconstructions toward the training distribution, thereby inflating similarity scores. In fact, even when random gradients are passed through the decoder (purple dashed line), the outputs achieve LPIPS values lower than random images, underscoring that the improvement reflects postprocessing artefacts. 
However, once 25\% or fewer gradients are accessible, the reconstructed outputs are severely distorted and no longer capture any semantically meaningful features of the original data.
\begin{figure*}[h!]
  \captionsetup{skip=3pt}
  \centering
  \includegraphics[width=1.0\textwidth]{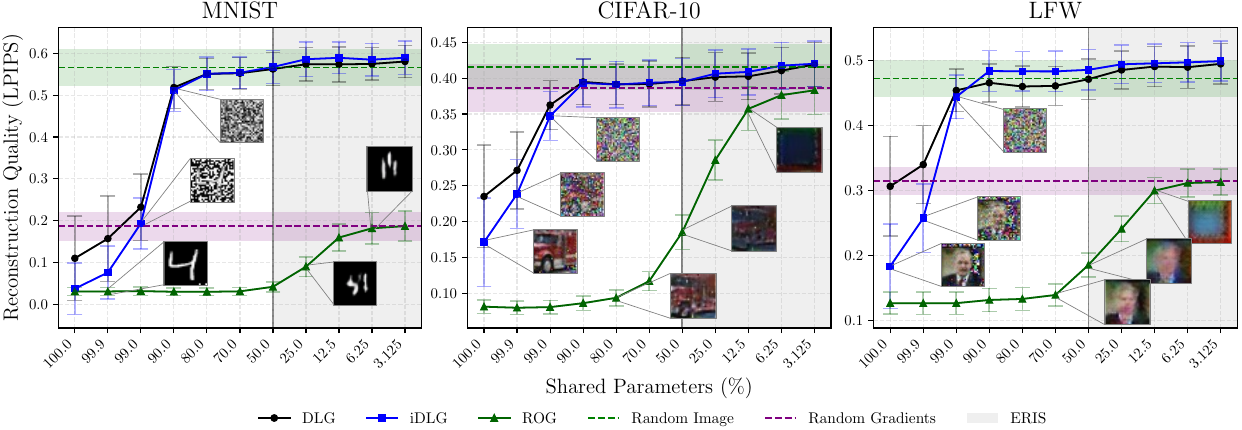}
  \caption{Reconstruction quality under DLG, iDLG, and ROG attacks as a function of the percentage of model parameters available to the attacker. The LPIPS score (higher is better) is averaged over 200 samples. The x-axis is plotted on a non-linear scale for improved clarity of the low-percentage regime. Shaded regions highlight the obfuscation achieved by \textsc{eris}, which renders reconstruction attacks ineffective even in its weakest configuration (two aggregators and no compression).}  
  \label{fig:rec_lpips}
\end{figure*}

To further characterize the privacy guarantees under stronger adaptive reconstruction settings, we complement the previous analysis with a focused evaluation of ROG and GGL across all implemented baselines and standalone compression mechanisms (Table~\ref{tab:comm_strategies_quality}). The results confirm that compression alone is insufficient: both QSGD and uniform quantization degrade reconstruction quality only at aggressive rates (e.g., $s=4$), while Top-$k$ sparsification becomes effective only at extreme sparsity levels (0.98--0.99), where utility is severely compromised. Similarly, differentially private training via DP-SGD shows a clear trade-off between privacy and utility: with mild clipping ($\text{clip}=10$) and low noise ($\sigma=10^{-4}$), reconstructions remain close to FedAvg, whereas stronger noise or tighter clipping substantially degrades image quality but at the expense of model performance. PriPrune exhibits a comparable pattern, with higher pruning probabilities providing stronger obfuscation but relying on increasingly aggressive gradient removal. Finally, \textsc{eris} provides robust protection even in its least favorable setting ($A=2$ aggregators), and its privacy guarantees strengthen as the number of aggregators increases or, when combined with DSC, as the compression strength $\omega$ increases. Notably, the strongest \textsc{eris} configurations approach the random-gradient baseline under ROG and substantially increase LPIPS under GGL. These results indicate that the combination of model partitioning and DSC effectively limits the information available to the attacker, including in the presence of adaptive generative reconstruction, while preserving utility.

\begin{table}[t]
\centering
\caption{Reconstruction quality under ROG and GGL attacks across privacy-preserving mechanisms and compression techniques on CIFAR-10 and ImageNet, respectively. For LPIPS, higher values indicate stronger defenses ($\uparrow$).}
\scriptsize
\renewcommand{\arraystretch}{1.25}
\setlength{\tabcolsep}{5pt}
\begin{tabular}{lcc}
\toprule
\textbf{Method} 
& \textbf{ROG} 
& \textbf{GGL} \\
\midrule
FedAvg 
& 0.193 $\pm$ 0.059 
& 0.666 $\pm$ 0.053 \\
\midrule
QSGD ($s=16$) 
& 0.209 $\pm$ 0.060 
& 0.670 $\pm$ 0.057 \\
QSGD ($s=8$)  
& 0.250 $\pm$ 0.065 
& 0.671 $\pm$ 0.058 \\
QSGD ($s=4$)  
& 0.343 $\pm$ 0.074 
& 0.676 $\pm$ 0.055 \\
\midrule
Uniform Quantization ($s=16$) 
& 0.243 $\pm$ 0.066 
& 0.676 $\pm$ 0.055 \\
Uniform Quantization ($s=8$)  
& 0.302 $\pm$ 0.071 
& 0.673 $\pm$ 0.064 \\
Uniform Quantization ($s=4$)  
& 0.403 $\pm$ 0.079 
& 0.686 $\pm$ 0.053 \\
\midrule
Top-\emph{k} Sparsification (sparsity=$0.90$) 
& 0.228 $\pm$ 0.063 
& 0.670 $\pm$ 0.059 \\
Top-\emph{k} Sparsification (sparsity=$0.98$) 
& 0.392 $\pm$ 0.083 
& 0.674 $\pm$ 0.051 \\
Top-\emph{k} Sparsification (sparsity=$0.99$) 
& 0.456 $\pm$ 0.094 
& 0.679 $\pm$ 0.061 \\
\midrule
DP-SGD (clip=$10$, $\sigma=10^{-4}$) 
& 0.200 $\pm$ 0.062 
& 0.668 $\pm$ 0.062 \\
DP-SGD (clip=$10$, $\sigma=10^{-3}$) 
& 0.340 $\pm$ 0.071 
& 0.671 $\pm$ 0.061 \\
DP-SGD (clip=$10$, $\sigma=10^{-2}$) 
& 0.498 $\pm$ 0.091 
& 0.676 $\pm$ 0.051 \\
DP-SGD (clip=$1$, $\sigma=10^{-4}$) 
& 0.432 $\pm$ 0.094 
& 0.676 $\pm$ 0.059 \\
DP-SGD (clip=$1$, $\sigma=10^{-3}$) 
& 0.436 $\pm$ 0.095 
& 0.682 $\pm$ 0.048 \\
DP-SGD (clip=$1$, $\sigma=10^{-2}$) 
& 0.472 $\pm$ 0.090 
& 0.689 $\pm$ 0.042 \\
\midrule
PriPrune ($p=10^{-5}$) 
& 0.305 $\pm$ 0.075 
& 0.686 $\pm$ 0.047 \\
PriPrune ($p=10^{-3}$) 
& 0.506 $\pm$ 0.087 
& 0.688 $\pm$ 0.036 \\
PriPrune ($p=0.1$) 
& 0.569 $\pm$ 0.067 
& 0.696 $\pm$ 0.035 \\
\midrule
\textsc{eris} ($A=2$) 
& 0.458 $\pm$ 0.081 
& 0.665 $\pm$ 0.060 \\
\textsc{eris} ($A=4$) 
& 0.514 $\pm$ 0.078 
& 0.675 $\pm$ 0.054 \\
\textsc{eris} ($A=8$) 
& 0.546 $\pm$ 0.075 
& 0.678 $\pm$ 0.056 \\
\textsc{eris} (+\textsc{dsc}) ($\omega=1$, $A=2$) 
& 0.453 $\pm$ 0.081 
& 0.676 $\pm$ 0.053 \\
\textsc{eris} (+\textsc{dsc}) ($\omega=4$, $A=2$) 
& 0.524 $\pm$ 0.078 
& 0.673 $\pm$ 0.052 \\
\textsc{eris}(+\textsc{dsc})  ($\omega=9$, $A=2$) 
& 0.547 $\pm$ 0.073 
& 0.683 $\pm$ 0.048 \\
\textsc{eris} (+\textsc{dsc}) ($\omega=49$, $A=50$) 
& 0.569 $\pm$ 0.073
& 0.690 $\pm$ 0.064 \\
\midrule
Random Gradients 
& 0.572 $\pm$ 0.065 
& 0.812 $\pm$ 0.047 \\
\bottomrule
\end{tabular}
\label{tab:comm_strategies_quality}
\end{table}

\subsection{Balancing Utility and Privacy - IID Setting} \label{app:balancing_ut_pri_iid}
This section provides detailed numerical results supporting the analysis in Paragraph~\ref{pg:balancing_u_p}. Specifically, we report test accuracy and Membership Inference Attack (MIA) accuracy for all evaluated methods across multiple datasets and varying local training sizes in IID setting. 
Tables~\ref{tab:rouge_mia_cnn}–\ref{tab:acc_mia_acc_mnist_large} report results for CNN/DailyMail, IMDB, CIFAR-10, and MNIST. IMDB, CIFAR-10, and MNIST are evaluated across 4–128 client training samples, while CNN/DailyMail is limited to 16–128 samples, as overfitting saturates already at 16. Part of these values serve as the coordinates for Figure~\ref{fig:unbiased_acc_priv}, which visualizes the utility–privacy trade-off achieved by \textsc{eris} and baselines. Importantly, the tables extend beyond the conditions illustrated in the figure by covering a wider set of hyperparameter configurations—namely, additional pruning rates ($p$) and privacy budgets ($\epsilon$) for LDP-based methods.

Notably, these results (Figure~\ref{fig:unbiased_acc_priv} and Tables~\ref{tab:rouge_mia_cnn}–\ref{tab:acc_mia_acc_mnist_large}) show two clear trends. First, a smaller amount of local data leads all methods to lower task accuracy and higher MIA accuracy, reflecting the stronger overfitting in this regime. Second, especially under low-data conditions, \textsc{eris} (with and without DSC) consistently delivers markedly better privacy preservation while retaining competitive accuracy. For instance, on CNN/DailyMail with 16 samples, for the same ROUGE-1 score as FedAvg, \textsc{eris} reduces MIA accuracy from 100\% to 77.7\%; on IMDB with 4 samples, it lowers MIA accuracy from 82.9\% to 65.2\%, closely approaching the unattainable upper-bound of 64.4\%. These findings highlight \textsc{eris}’s robustness across data modalities and model capacities.

\begin{table}[htbp]
\captionsetup{skip=3pt}
\centering
\caption{
Comparison of \textsc{eris} with and without compression ($\omega$) against SOTA baselines in terms of ROUGE-1 and MIA accuracy on CNN/DailyMail with 16, 32, 64, and 128 local training samples.
}
\scriptsize
\renewcommand{\arraystretch}{1.2}
\setlength{\tabcolsep}{1.6pt}
\begin{tabular}{lcccccccc}
\toprule
\textbf{Local Training Size} & \multicolumn{2}{c}{\textbf{$16$ samples}} & \multicolumn{2}{c}{\textbf{$32$ samples}} & \multicolumn{2}{c}{\textbf{$64$ samples}} & \multicolumn{2}{c}{\textbf{$128$ samples}} \\
\cmidrule(lr){2-3} \cmidrule(lr){4-5} \cmidrule(lr){6-7} \cmidrule(lr){8-9}
\textbf{Method} & \textbf{R-1 ($\uparrow$)} & \textbf{MIA Acc. ($\downarrow$)} & \textbf{R-1 ($\uparrow$)} & \textbf{MIA Acc. ($\downarrow$)} & \textbf{R-1 ($\uparrow$)} & \textbf{MIA Acc. ($\downarrow$)} & \textbf{R-1 ($\uparrow$)} & \textbf{MIA Acc. ($\downarrow$)} \\
\midrule
FedAvg & 30.37$\,$±$\,$1.25 & 100.00$\,$±$\,$0.00 & 32.21$\,$±$\,$1.46 & 98.75$\,$±$\,$1.25 & 34.27$\,$±$\,$0.65 & 96.46$\,$±$\,$0.36 & 36.04$\,$±$\,$0.57 & 96.57$\,$±$\,$0.90 \\
FedAvg ($10,\delta$)-LDP & 25.66$\,$±$\,$0.86 & 54.17$\,$±$\,$5.20 & 26.26$\,$±$\,$0.07 & 50.00$\,$±$\,$3.31 & 26.30$\,$±$\,$0.10 & 49.33$\,$±$\,$2.53 & 25.78$\,$±$\,$0.10 & 54.41$\,$±$\,$1.47 \\
FedAvg ($100,\delta$)-LDP & 26.33$\,$±$\,$0.07 & 54.17$\,$±$\,$5.20 & 26.36$\,$±$\,$0.10 & 50.42$\,$±$\,$2.89 & 24.90$\,$±$\,$0.13 & 48.75$\,$±$\,$2.25 & 26.32$\,$±$\,$0.10 & 54.31$\,$±$\,$1.19 \\
SoteriaFL ($\epsilon=100$) &  25.89$\,$±$\,$0.12 & 54.13$\,$±$\,$1.15 & 26.09$\,$±$\,$0.07 & 54.27$\,$±$\,$5.20 & 26.01$\,$±$\,$0.11 & 50.35$\,$±$\,$2.89 & 24.02$\,$±$\,$0.15 & 49.15$\,$±$\,$2.15 \\
SoteriaFL ($\epsilon=10$) &  25.78$\,$±$\,$0.91 & 54.02$\,$±$\,$1.48 & 25.90$\,$±$\,$0.36 & 54.17$\,$±$\,$5.20 & 25.90$\,$±$\,$0.75 & 50.83$\,$±$\,$2.60 & 24.75$\,$±$\,$0.79 & 49.54$\,$±$\,$2.60 \\
PriPrune ($p=0.1$) & 32.21$\,$±$\,$0.44 & 95.83$\,$±$\,$2.89 & 26.01$\,$±$\,$0.70 & 82.92$\,$±$\,$1.91 & 30.97$\,$±$\,$0.55 & 90.00$\,$±$\,$1.08 & 33.90$\,$±$\,$1.45 & 88.73$\,$±$\,$0.90 \\
PriPrune ($p=0.2$) &  29.96$\,$±$\,$0.94 & 88.33$\,$±$\,$2.89 & 26.61$\,$±$\,$0.72 & 76.67$\,$±$\,$3.61 & 29.00$\,$±$\,$0.41 & 79.38$\,$±$\,$3.80 & 31.49$\,$±$\,$1.89 & 79.02$\,$±$\,$0.61 \\
PriPrune ($p=0.3$) & 18.41$\,$±$\,$15.11 & 74.17$\,$±$\,$3.82 & 21.80$\,$±$\,$0.67 & 70.00$\,$±$\,$3.31 & 29.00$\,$±$\,$1.10 & 70.83$\,$±$\,$3.44 & 29.49$\,$±$\,$1.67 & 70.39$\,$±$\,$0.74 \\
Shatter & 30.05$\,$±$\,$1.22 & 78.73$\,$±$\,$7.50 & 30.38$\,$±$\,$0.59 & 69.02$\,$±$\,$2.60 & 33.04$\,$±$\,$0.65 & 66.78$\,$±$\,$5.12 & 34.35$\,$±$\,$0.38 & 68.13$\,$±$\,$0.90 \\
\midrule
\textsc{eris} & 30.37$\,$±$\,$1.25 & 78.50$\,$±$\,$7.21 & 32.41$\,$±$\,$1.46 & 69.12$\,$±$\,$2.64 & 34.04$\,$±$\,$0.76 & 66.18$\,$±$\,$5.00 & 36.04$\,$±$\,$0.57 & 68.16$\,$±$\,$0.96 \\
\textsc{eris} (+\textsc{dsc}) ($\omega_{\text{SoteriaFL}}$) & 30.05$\,$±$\,$0.86 & 78.33$\,$±$\,$6.29 & 32.15$\,$±$\,$1.75 & 68.75$\,$±$\,$3.31 & 34.12$\,$±$\,$0.38 & 63.75$\,$±$\,$4.38 & 35.06$\,$±$\,$1.27 & 67.65$\,$±$\,$0.78 \\
\textsc{eris} (+\textsc{dsc}) ($\omega \!\approx\! 100$) & 30.04$\,$±$\,$0.95 & 77.73$\,$±$\,$6.29 & 31.60$\,$±$\,$0.95 & 68.27$\,$±$\,$3.61 & 34.14$\,$±$\,$0.76 & 64.38$\,$±$\,$5.12 & 35.62$\,$±$\,$0.48 & 67.84$\,$±$\,$0.74 \\
\midrule
Min. Leakage & 30.37$\,$±$\,$1.25 & 67.83$\,$±$\,$10.10 & 32.41$\,$±$\,$1.46 & 60.58$\,$±$\,$1.91 & 34.27$\,$±$\,$0.65 & 54.08$\,$±$\,$4.77 & 36.04$\,$±$\,$0.57 & 59.61$\,$±$\,$2.54 \\
\bottomrule
\end{tabular}
\label{tab:rouge_mia_cnn}
\end{table}

\begin{table}[htbp]
\captionsetup{skip=3pt}
\centering
\caption{
Comparison of \textsc{eris} with and without compression ($\omega$) against SOTA baselines in terms of test and MIA accuracy on IMDB with 4, 8, and 16 local training samples. 
}
\scriptsize
\renewcommand{\arraystretch}{1.2}
\setlength{\tabcolsep}{4pt}
\begin{tabular}{lcccccc}
\toprule
\textbf{Local Training Size} & \multicolumn{2}{c}{\textbf{$4$ samples}} & \multicolumn{2}{c}{\textbf{$8$ samples}} & \multicolumn{2}{c}{\textbf{$16$ samples}} \\
\cmidrule(lr){2-3} \cmidrule(lr){4-5} \cmidrule(lr){6-7}
\textbf{Method} & \textbf{Accuracy ($\uparrow$)} & \textbf{MIA Acc. ($\downarrow$)} & \textbf{Accuracy ($\uparrow$)} & \textbf{MIA Acc. ($\downarrow$)} & \textbf{Accuracy ($\uparrow$)} & \textbf{MIA Acc. ($\downarrow$)} \\
\midrule
FedAvg & 71.73 ± 4.93 & 82.93 ± 6.39 & 79.56 ± 0.61 & 78.40 ± 3.95 & 80.52 ± 0.30 & 66.91 ± 1.81 \\
FedAvg ($\epsilon=100,\delta$)-LDP & 53.79 ± 0.08 & 55.56 ± 4.91 & 53.82 ± 0.06 & 53.33 ± 1.00 & 53.92 ± 0.11 & 50.06 ± 2.81 \\
FedAvg ($\epsilon=10,\delta$)-LDP & 53.80 ± 0.03 & 52.80 ± 5.82 & 53.81 ± 0.02 & 50.40 ± 3.12 & 53.83 ± 0.06 & 49.89 ± 1.74 \\
SoteriaFL ($\epsilon=100,\delta$) & 53.46 ± 0.15 & 55.56 ± 6.00 & 54.73 ± 0.15 & 54.40 ± 1.96 & 53.74 ± 0.16 & 50.30 ± 2.69 \\
SoteriaFL ($\epsilon=10,\delta$) & 53.36 ± 0.29 & 55.20 ± 5.63 & 54.01 ± 0.24 & 51.36 ± 3.48 & 53.67 ± 0.29 & 50.11 ± 1.58 \\
PriPrune ($p=0.1$) & 54.40 ± 5.46 & 80.53 ± 4.59 & 71.70 ± 2.09 & 74.72 ± 3.63 & 77.64 ± 1.44 & 65.31 ± 2.18 \\
PriPrune ($p=0.2$) & 52.78 ± 2.39 & 74.67 ± 4.99 & 58.55 ± 5.30 & 71.36 ± 2.23 & 65.33 ± 8.65 & 62.84 ± 2.41 \\
PriPrune ($p=0.3$) & 53.52 ± 2.79 & 70.40 ± 3.20 & 55.92 ± 3.49 & 65.76 ± 3.63 & 60.32 ± 5.71 & 59.82 ± 2.85 \\
Shatter & 68.52	± 4.66 & 67.52 ± 2.80 & 74.84 ± 1.88 & 62.56 ± 3.77 & 77.91 ± 0.55 & 54.75 ± 1.97 \\
\midrule
\textsc{eris} & 70.15 ± 4.24 & 67.67 ± 2.89 & 79.39 ± 0.57 & 62.45 ± 3.79 & 80.49 ± 0.33 & 54.72 ± 1.90 \\
\textsc{eris} (+\textsc{dsc}) ($\omega_{\text{SoteriaFL}}$) & 71.74$\,$±$\,$4.94 & 65.87$\,$±$\,$3.22 & 79.55$\,$±$\,$0.61 & 60.52$\,$±$\,$3.00 & 80.51$\,$±$\,$0.31 & 54.56$\,$±$\,$1.57 \\
\textsc{eris} (+\textsc{dsc}) ($\omega\!\approx\!8000$) & 71.28 ± 4.74 & 65.22 ± 2.95 & 79.28 ± 0.71 & 60.51 ± 2.50 & 80.11 ± 0.46 & 54.12 ± 1.78 \\
\midrule
Min. Leakage & 72.39 ± 1.99 & 64.44 ± 2.27 & 79.33 ± 0.75 & 58.67 ± 2.10 & 80.68 ± 0.09 & 53.21 ± 2.53 \\
\bottomrule
\end{tabular}
\label{tab:acc_mia_acc_imdb_small}
\end{table}

\begin{table}[htbp]
\captionsetup{skip=3pt}
\centering
\caption{
Comparison of \textsc{eris} with and without compression ($\omega$) against SOTA baselines in terms of test and MIA accuracy on IMDB with 32, 64, and 128 local training samples. 
}
\scriptsize
\renewcommand{\arraystretch}{1.2}
\setlength{\tabcolsep}{4pt}
\begin{tabular}{lcccccc}
\toprule
\textbf{Local Training Size} & \multicolumn{2}{c}{\textbf{$32$ samples}} & \multicolumn{2}{c}{\textbf{$64$ samples}} & \multicolumn{2}{c}{\textbf{$128$ samples}} \\
\cmidrule(lr){2-3} \cmidrule(lr){4-5} \cmidrule(lr){6-7}
\textbf{Method} & \textbf{Accuracy ($\uparrow$)} & \textbf{MIA Acc. ($\downarrow$)} & \textbf{Accuracy ($\uparrow$)} & \textbf{MIA Acc. ($\downarrow$)} & \textbf{Accuracy ($\uparrow$)} & \textbf{MIA Acc. ($\downarrow$)} \\
\midrule
FedAvg & 81.62 ± 0.11 & 63.58 ± 1.47 & 81.70 ± 0.05 & 60.54 ± 2.11 & 82.45 ± 0.18 & 56.89 ± 0.81 \\
FedAvg ($\epsilon=100,\delta$)-LDP & 54.11 ± 0.15 & 51.56 ± 1.71 & 54.50 ± 0.19 & 50.67 ± 1.60 & 55.09 ± 0.32 & 51.26 ± 0.49 \\
FedAvg ($\epsilon=10,\delta$)-LDP & 53.97 ± 0.10 & 50.02 ± 1.52 & 54.12 ± 0.12 & 49.84 ± 0.86 & 54.30 ± 0.19 & 50.37 ± 0.92 \\
SoteriaFL ($\epsilon=100,\delta$) & 54.87 ± 0.72 & 51.94 ± 1.94 & 55.81 ± 0.62 & 50.98 ± 1.87 & 56.08 ± 1.19 & 51.17 ± 0.52 \\
SoteriaFL ($\epsilon=10,\delta$) & 54.35 ± 0.39 & 50.10 ± 1.83 & 54.79 ± 0.35 & 50.12 ± 0.93 & 55.28 ± 0.57 & 50.64 ± 0.92 \\
PriPrune ($p=0.1$) & 79.93 ± 0.23 & 61.87 ± 1.25 & 80.34 ± 0.10 & 59.61 ± 2.11 & 80.87 ± 0.11 & 56.11 ± 0.93 \\
PriPrune ($p=0.2$) & 71.48 ± 7.20 & 59.54 ± 1.71 & 72.12 ± 1.91 & 58.21 ± 2.11 & 77.52 ± 0.55 & 55.02 ± 1.09 \\
PriPrune ($p=0.3$) & 61.20 ± 7.02 & 56.43 ± 1.55 & 62.01 ± 5.91 & 56.63 ± 1.76 & 68.95 ± 1.92 & 54.20 ± 1.18 \\
Shatter & 79.34 ± 0.64 & 53.79 ± 1.56 & 80.19 ± 0.48 & 53.68 ± 0.95 & 80.84 ± 0.19 & 52.02 ± 1.01 \\
\midrule
\textsc{eris} & 81.63 ± 0.12 & 53.98 ± 1.48 & 81.70 ± 0.06 & 53.54 ± 0.93 & 82.38 ± 0.13 & 52.08 ± 1.07 \\
\textsc{eris} (+\textsc{dsc}) ($\omega_{\text{SoteriaFL}}$) & 81.62$\,$±$\,$0.12 & 53.81$\,$±$\,$1.47 & 81.71$\,$±$\,$0.05 & 52.78$\,$±$\,$1.70 & 82.45$\,$±$\,$0.16 & 51.54$\,$±$\,$1.27 \\
\textsc{eris} (+\textsc{dsc}) ($\omega\!\approx\!8000$) & 80.99 ± 0.22 & 53.44 ± 1.26 & 81.16 ± 0.19 & 52.77 ± 1.62 & 81.59 ± 0.21 & 51.80 ± 1.07 \\
\midrule
Min. Leakage & 81.58 ± 0.11 & 52.57 ± 1.27 & 81.67 ± 0.04 & 53.05 ± 1.32 & 82.44 ± 0.09 & 51.53 ± 1.12 \\
\bottomrule
\end{tabular}
\label{tab:acc_mia_acc_imdb_large}
\end{table}

\begin{table}[htbp]
\captionsetup{skip=3pt}
\centering
\caption{Comparison of \textsc{eris} with and without compression ($\omega$) against SOTA baselines in terms of test and MIA accuracy on CIFAR-10 with 4, 8, and 16 local training samples. 
}
\scriptsize
\renewcommand{\arraystretch}{1.2}
\setlength{\tabcolsep}{4pt}
\begin{tabular}{lcccccc}
\toprule
\textbf{Local Training Size} & \multicolumn{2}{c}{\textbf{$4$ samples}} & \multicolumn{2}{c}{\textbf{$8$ samples}} & \multicolumn{2}{c}{\textbf{$16$ samples}} \\
\cmidrule(lr){2-3} \cmidrule(lr){4-5} \cmidrule(lr){6-7}
\textbf{Method} & \textbf{Accuracy ($\uparrow$)} & \textbf{MIA Acc. ($\downarrow$)} & \textbf{Accuracy ($\uparrow$)} & \textbf{MIA Acc. ($\downarrow$)} & \textbf{Accuracy ($\uparrow$)} & \textbf{MIA Acc. ($\downarrow$)} \\
\midrule
FedAvg & 27.12 ± 1.20 & 84.80 ± 4.59 & 32.98 ± 0.61 & 75.84 ± 2.85 & 34.43 ± 1.04 & 70.15 ± 1.41 \\
FedAvg ($\epsilon=10,\delta$)-LDP & 10.33 ± 0.53 & 81.20 ± 3.90 & 14.93 ± 2.01 & 72.40 ± 2.60 & 18.92 ± 1.31 & 62.55 ± 1.19 \\
FedAvg ($\epsilon=1,\delta$)-LDP & 10.34 ± 0.22 & 66.50 ± 3.78 & 10.00 ± 0.00 & 63.60 ± 2.42 & 10.00 ± 0.00 & 56.05 ± 0.49 \\
SoteriaFL ($\epsilon=10,\delta$) & 10.00 ± 0.00 & 69.87 ± 1.86 & 10.06 ± 0.12 & 64.16 ± 1.25 & 10.85 ± 1.06 & 58.25 ± 2.10 \\
SoteriaFL ($\epsilon=1,\delta$) & 9.99 ± 0.00 & 65.67 ± 1.11 & 10.00 ± 0.00 & 62.10 ± 1.56 & 10.00 ± 0.00 & 53.86 ± 0.67 \\
PriPrune ($p=0.01$) & 13.74 ± 2.05 & 74.80 ± 2.87 & 28.42 ± 0.39 & 75.36 ± 3.60 & 29.57 ± 0.70 & 69.82 ± 1.59 \\
PriPrune ($p=0.05$) & 10.09 ± 0.36 & 67.33 ± 2.63 & 12.77 ± 2.52 & 61.68 ± 3.04 & 11.55 ± 1.80 & 54.44 ± 1.50 \\
PriPrune ($p=0.1$) & 10.00 ± 0.00 & 64.27 ± 2.62 & 10.03 ± 0.04 & 58.80 ± 2.83 & 10.00 ± 0.00 & 52.62 ± 1.37 \\
Shatter & 11.47 ± 1.75 & 77.95 ± 5.63 & 11.57 ± 1.96 & 70.49 ± 2.74 & 12.42 ± 1.65 & 64.22 ± 1.85
\\
\midrule
\textsc{eris} & 27.13 ± 1.19 & 77.90 ± 5.55 & 32.90 ± 0.40 & 70.75 ± 2.70 & 34.32 ± 0.91 & 64.21 ± 1.95 \\
\textsc{eris} (+\textsc{dsc}) ($\omega_{\text{SoteriaFL}}$) & 26.84$\,$±$\,$0.68 & 73.86$\,$±$\,$5.27 & 32.48$\,$±$\,$1.43 & 67.95$\,$±$\,$1.45 & 34.08$\,$±$\,$1.05 & 59.23$\,$±$\,$1.43 \\
\textsc{eris} (+\textsc{dsc}) ($\omega\!\approx\!170$) & 26.31 ± 1.16 & 71.63 ± 4.28 & 33.28 ± 1.06 & 68.48 ± 2.30 & 34.62 ± 1.42 & 59.58 ± 2.26 \\
\midrule
Min. Leakage & 27.11 ± 1.17 & 70.27 ± 4.69 & 33.11 ± 0.62 & 65.44 ± 2.43 & 34.62 ± 0.91 & 56.87 ± 1.34 \\
\bottomrule
\end{tabular}
\label{tab:acc_mia_acc_cifar_small}
\end{table}

\begin{table}[htbp]
\captionsetup{skip=3pt}
\centering
\caption{
Comparison of \textsc{eris} with and without compression ($\omega$) against SOTA baselines in terms of test and MIA accuracy on CIFAR-10 with 32, 64, and 128 local training samples. 
}
\scriptsize
\renewcommand{\arraystretch}{1.2}
\setlength{\tabcolsep}{4pt}
\begin{tabular}{lcccccc}
\toprule
\textbf{Local Training Size} & \multicolumn{2}{c}{\textbf{$32$ samples}} & \multicolumn{2}{c}{\textbf{$64$ samples}} & \multicolumn{2}{c}{\textbf{$128$ samples}} \\
\cmidrule(lr){2-3} \cmidrule(lr){4-5} \cmidrule(lr){6-7}
\textbf{Method} & \textbf{Accuracy ($\uparrow$)} & \textbf{MIA Acc. ($\downarrow$)} & \textbf{Accuracy ($\uparrow$)} & \textbf{MIA Acc. ($\downarrow$)} & \textbf{Accuracy ($\uparrow$)} & \textbf{MIA Acc. ($\downarrow$)} \\
\midrule
FedAvg & 37.24 ± 0.41 & 64.57 ± 0.72 & 38.50 ± 0.44 & 59.29 ± 0.79 & 38.88 ± 0.32 & 56.11 ± 0.75 \\
FedAvg ($\epsilon=10,\delta$)-LDP & 22.31 ± 1.12 & 57.14 ± 1.39 & 23.36 ± 0.85 & 53.99 ± 0.83 & 24.13 ± 0.32 & 52.81 ± 0.49 \\
FedAvg ($\epsilon=1,\delta$)-LDP & 10.00 ± 0.00 & 57.57 ± 0.26 & 13.96 ± 1.14 & 54.42 ± 0.40 & 19.29 ± 0.37 & 53.12 ± 0.16 \\
SoteriaFL ($\epsilon=10,\delta$) & 19.68 ± 0.78 & 55.68 ± 1.09 & 26.04 ± 0.52 & 52.94 ± 0.74 & 26.46 ± 0.25 & 52.07 ± 0.55 \\
SoteriaFL ($\epsilon=1,\delta$) & 10.00 ± 0.00 & 54.57 ± 0.50 & 10.00 ± 0.00 & 53.28 ± 0.61 & 12.20 ± 1.25 & 52.58 ± 0.53 \\
PriPrune ($p=0.01$) & 29.39 ± 0.50 & 63.73 ± 0.96 & 28.70 ± 0.51 & 57.09 ± 0.67 & 27.99 ± 0.32 & 53.22 ± 0.70 \\
PriPrune ($p=0.05$) & 11.80 ± 2.44 & 52.80 ± 1.80 & 11.05 ± 1.60 & 52.01 ± 0.35 & 10.21 ± 0.28 & 51.01 ± 0.20 \\
PriPrune ($p=0.1$) & 10.00 ± 0.00 & 51.94 ± 1.86 & 10.00 ± 0.01 & 51.06 ± 0.64 & 10.00 ± 0.00 & 50.48 ± 0.81 \\
Shatter & 12.32 ± 2.03 & 58.58 ± 0.95 & 12.96 ± 2.16 & 54.65 ± 0.54 & 13.64 ± 1.55 & 52.03 ± 0.49 \\

\midrule
\textsc{eris} & 37.12 ± 0.55 & 58.63 ± 0.88 & 38.59 ± 0.50 & 54.60 ± 0.39 & 38.95 ± 0.32 & 52.04 ± 0.41 \\
\textsc{eris} (+\textsc{dsc}) ($\omega_{\text{SoteriaFL}}$) & 37.36$\,$±$\,$1.59 & 56.54$\,$±$\,$0.94 & 38.43$\,$±$\,$1.45 & 53.56$\,$±$\,$0.57 & 38.41$\,$±$\,$0.51 & 51.66$\,$±$\,$0.52 \\
\textsc{eris} (+\textsc{dsc}) ($\omega\!\approx\!170$) & 37.40 ± 1.36 & 57.49 ± 0.85 & 38.16 ± 1.01 & 53.98 ± 0.41 & 38.30 ± 0.88 & 51.70 ± 0.53 \\
\midrule
Min. Leakage & 37.25 ± 0.38 & 55.81 ± 0.81 & 38.57 ± 0.37 & 53.06 ± 0.47 & 38.88 ± 0.36 & 51.67 ± 0.48 \\
\bottomrule
\end{tabular}
\label{tab:acc_mia_acc_cifar_large}
\end{table}

\begin{table}[htbp]
\captionsetup{skip=3pt}
\centering
\caption{
Comparison of \textsc{eris} with and without compression ($\omega$) against SOTA baselines in terms of test and MIA accuracy on MNIST with 4, 8, and 16 local training samples. 
}
\scriptsize
\renewcommand{\arraystretch}{1.2}
\setlength{\tabcolsep}{4pt}
\begin{tabular}{lcccccc}
\toprule
\textbf{Local Training Size} & \multicolumn{2}{c}{\textbf{$4$ samples}} & \multicolumn{2}{c}{\textbf{$8$ samples}} & \multicolumn{2}{c}{\textbf{$16$ samples}} \\
\cmidrule(lr){2-3} \cmidrule(lr){4-5} \cmidrule(lr){6-7}
\textbf{Method} & \textbf{Accuracy ($\uparrow$)} & \textbf{MIA Acc. ($\downarrow$)} & \textbf{Accuracy ($\uparrow$)} & \textbf{MIA Acc. ($\downarrow$)} & \textbf{Accuracy ($\uparrow$)} & \textbf{MIA Acc. ($\downarrow$)} \\
\midrule
FedAvg & 80.69 ± 1.71 & 82.13 ± 1.65 & 86.42 ± 0.88 & 72.00 ± 3.01 & 89.23 ± 0.74 & 65.78 ± 2.24 \\
FedAvg ($\epsilon=10,\delta$)-LDP & 39.65 ± 3.14 & 69.07 ± 1.67 & 50.84 ± 4.75 & 59.44 ± 2.00 & 64.40 ± 1.53 & 57.67 ± 1.61 \\
FedAvg ($\epsilon=1,\delta$)-LDP & 9.73 ± 0.46 & 69.50 ± 1.72 & 10.70 ± 1.27 & 58.50 ± 0.71 & 19.80 ± 1.38 & 57.55 ± 1.66 \\
SoteriaFL ($\epsilon=10,\delta$) & 8.83 ± 2.65 & 71.47 ± 1.81 & 32.15 ± 2.00 & 57.68 ± 1.83 & 67.01 ± 1.31 & 56.87 ± 1.69 \\
SoteriaFL ($\epsilon=1,\delta$) & 10.31 ± 0.38 & 67.50 ± 2.33 & 10.84 ± 0.79 & 57.90 ± 2.37 & 12.72 ± 1.33 & 57.27 ± 1.44 \\
PriPrune ($p=0.01$) & 47.89 ± 8.33 & 77.20 ± 3.33 & 70.60 ± 3.67 & 68.32 ± 4.28 & 84.81 ± 0.31 & 63.56 ± 2.09 \\
PriPrune ($p=0.05$) & 17.01 ± 4.22 & 58.00 ± 3.45 & 18.97 ± 3.12 & 50.16 ± 2.60 & 26.47 ± 2.59 & 54.51 ± 1.62 \\
PriPrune ($p=0.1$) & 11.99 ± 1.87 & 56.67 ± 3.18 & 13.33 ± 2.21 & 49.44 ± 2.54 & 19.46 ± 1.35 & 53.42 ± 1.54 \\
Shatter & 11.96 ± 2.33 & 70.42 ± 2.21 & 12.32 ± 2.92 & 56.51 ± 2.86 & 14.55 ± 4.24 & 55.61 ± 1.33 \\
\midrule
\textsc{eris} & 80.47 ± 1.75 & 69.82 ± 2.01 & 86.28 ± 1.00 & 56.35 ± 2.89 & 89.27 ± 0.73 & 55.60 ± 1.27 \\
\textsc{eris} (+\textsc{dsc}) ($\omega_{\text{SoteriaFL}}$) & 78.31$\,$±$\,$1.45 & 68.70$\,$±$\,$3.08 & 85.59$\,$±$\,$0.67 & 54.43$\,$±$\,$2.14 & 90.04$\,$±$\,$0.43 & 55.58$\,$±$\,$1.77 \\
\textsc{eris} (+\textsc{dsc}) ($\omega\!\approx\!30$) & 78.72 ± 1.19 & 68.48 ± 3.11 & 84.84 ± 0.58 & 55.14 ± 2.73 & 90.26 ± 0.11 & 56.11 ± 1.58 \\

\midrule
Min. Leakage & 80.68 ± 1.95 & 66.67 ± 2.67 & 86.30 ± 1.06 & 54.32 ± 1.67 & 89.26 ± 0.74 & 55.38 ± 1.56 \\
\bottomrule
\end{tabular}
\label{tab:acc_mia_acc_mnist_small}
\end{table}

\begin{table}[htbp]
\captionsetup{skip=3pt}
\centering
\caption{
Comparison of \textsc{eris} with and without compression ($\omega$) against SOTA baselines in terms of test and MIA accuracy on MNIST with 32, 64, and 128 local training samples. 
}
\scriptsize
\renewcommand{\arraystretch}{1.2}
\setlength{\tabcolsep}{4pt}
\begin{tabular}{lcccccc}
\toprule
\textbf{Local Training Size} & \multicolumn{2}{c}{\textbf{$32$ samples}} & \multicolumn{2}{c}{\textbf{$64$ samples}} & \multicolumn{2}{c}{\textbf{$128$ samples}} \\
\cmidrule(lr){2-3} \cmidrule(lr){4-5} \cmidrule(lr){6-7}
\textbf{Method} & \textbf{Accuracy ($\uparrow$)} & \textbf{MIA Acc. ($\downarrow$)} & \textbf{Accuracy ($\uparrow$)} & \textbf{MIA Acc. ($\downarrow$)} & \textbf{Accuracy ($\uparrow$)} & \textbf{MIA Acc. ($\downarrow$)} \\
\midrule
FedAvg & 91.48 ± 0.37 & 59.94 ± 1.11 & 92.55 ± 0.06 & 56.68 ± 1.92 & 93.11 ± 0.16 & 54.14 ± 0.65 \\
FedAvg ($\epsilon=10,\delta$)-LDP & 70.35 ± 1.31 & 53.05 ± 1.09 & 70.43 ± 1.27 & 53.00 ± 1.00 & 70.48 ± 0.39 & 51.20 ± 0.91 \\
FedAvg ($\epsilon=1,\delta$)-LDP & 57.75 ± 0.87 & 54.71 ± 0.85 & 70.96 ± 0.71 & 52.72 ± 0.74 & 70.56 ± 0.34 & 51.40 ± 0.66 \\
SoteriaFL ($\epsilon=10,\delta$) & 77.95 ± 2.38 & 53.28 ± 0.60 & 78.15 ± 2.66 & 52.16 ± 0.84 & 79.50 ± 1.45 & 51.29 ± 0.77 \\
SoteriaFL ($\epsilon=1,\delta$) & 8.49 ± 1.87 & 54.14 ± 0.53 & 40.37 ± 1.44 & 52.83 ± 0.24 & 68.75 ± 0.38 & 51.89 ± 0.53 \\
PriPrune ($p=0.01$) & 87.77 ± 0.15 & 56.99 ± 1.14 & 87.01 ± 0.17 & 54.54 ± 1.03 & 86.38 ± 0.19 & 52.65 ± 0.69 \\
PriPrune ($p=0.05$) & 33.01 ± 0.73 & 51.41 ± 1.34 & 34.81 ± 1.98 & 51.39 ± 1.00 & 33.87 ± 0.63 & 50.69 ± 0.93 \\
PriPrune ($p=0.1$) & 21.08 ± 1.39 & 50.74 ± 1.58 & 20.68 ± 0.90 & 51.14 ± 0.98 & 20.46 ± 0.37 & 50.65 ± 0.87 \\
Shatter & 16.51 ± 6.16 & 52.02 ± 1.09 & 18.50 ± 6.06 & 51.46 ± 0.79 & 21.29 ± 7.22 & 51.61 ± 0.76 \\
\midrule
\textsc{eris} & 91.52 ± 0.32 & 52.08 ± 1.04 & 92.54 ± 0.04 & 51.45 ± 0.85 & 93.12 ± 0.16 & 51.53 ± 0.79 \\
\textsc{eris} (+\textsc{dsc}) ($\omega_{\text{SoteriaFL}}$) & 92.27$\,$±$\,$0.21 & 52.66$\,$±$\,$0.73 & 93.23$\,$±$\,$0.10 & 51.57$\,$±$\,$0.74 & 93.78$\,$±$\,$0.10 & 51.71$\,$±$\,$0.67 \\
\textsc{eris}  (+\textsc{dsc}) ($\omega\!\approx\!30$) & 92.56 ± 0.30 & 52.71 ± 0.86 & 93.58 ± 0.23 & 51.74 ± 0.66 & 94.02 ± 0.19 & 51.65 ± 0.64 \\

\midrule
Min. Leakage & 91.47 ± 0.39 & 52.06 ± 1.04 & 92.56 ± 0.05 & 51.41 ± 0.84 & 93.14 ± 0.17 & 51.50 ± 0.85 \\
\bottomrule
\end{tabular}
\label{tab:acc_mia_acc_mnist_large}
\end{table}

\subsection{Balancing Utility and Privacy - non-IID Setting}  \label{app:nonIID_tradeoff}
We further evaluate the utility--privacy trade-off under non-IID client data, using $K{=}10$ clients for IMDB and $K{=}25$ for CIFAR-10 and MNIST, with a Dirichlet partition with $\alpha{=}0.5$ for IMDB and $\alpha{=}0.2$ for CIFAR-10 and MNIST.
Figure~\ref{fig:non_IID_tradeoff} illustrates the utility–privacy trade-off across methods and datasets, where the ideal region corresponds to the top-right corner (high accuracy and low privacy leakage). Non-IID distributions generally make convergence more challenging, lowering overall accuracy and increasing variability across clients \cite{fenoglioFLUX2025}. Nevertheless, \textsc{eris} with DSC remains stable and consistently reduces privacy leakage. For example, on IMDB with 4 local samples, \textsc{eris} matches FedAvg in accuracy while reducing MIA accuracy from 91.7\% to 80.7\%, approaching the ideal upper-bound of not sharing local gradients. On CIFAR-10 and MNIST, \textsc{eris} even matches or slightly surpass non-private FedAvg in terms of accuracy, while still offering strong privacy protection. By contrast, privacy-enhancing baselines such as Shatter and FedAvg-LDP struggle to maintain utility, often remaining close to random-guess performance, particularly when models are trained from scratch. Table~\ref{tab:overall_acc_mia_all_biased_setting} reports detailed mean accuracy and MIA accuracy values, averaged over varying local sample sizes. Together, these results confirm that the advantages of \textsc{eris} extend robustly to heterogeneous data distributions.  

\begin{figure*}[h!]
  \centering
  \includegraphics[width=1.0\textwidth]{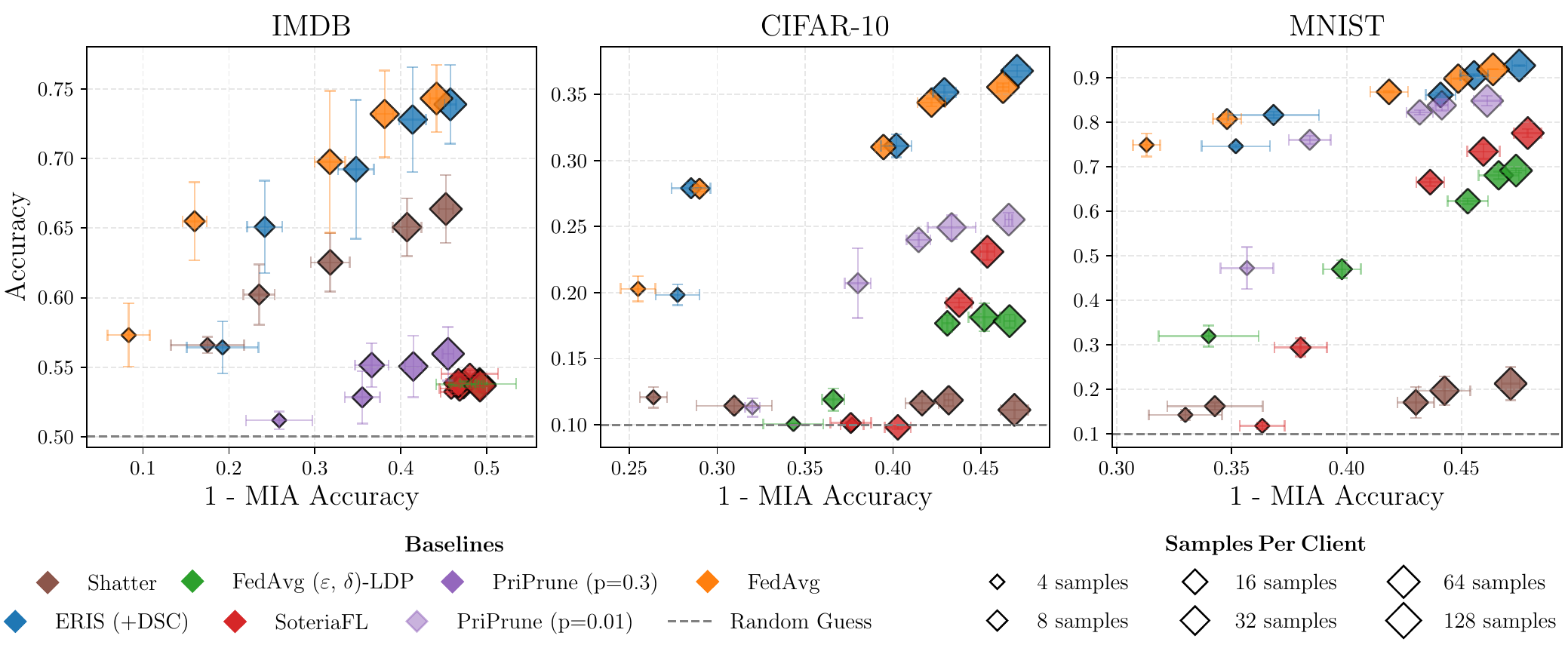}
  \caption{Comparison of accuracy and (1-MIA) accuracy across varying model sizes 
  and client-side overfitting levels, controlled via the number of client training samples under non-IID setting. 
  }
  \label{fig:non_IID_tradeoff}
\end{figure*}

\begin{table}[htbp]
    \captionsetup{skip=5pt}
    \centering
    \caption{Mean test accuracy and MIA accuracy, averaged over varying local sample sizes under non-IID setting. For DP-based methods, $\epsilon{=}10$; for PriPrune, pruning rates are $p\!\in\!\{0.1,0.2,0.3\}$ on IMDB and $p\!\in\!\{0.01,0.05,0.1\}$ on CIFAR-10/MNIST.}
    \tiny
    \renewcommand{\arraystretch}{1.2}
    \setlength{\tabcolsep}{4pt}
    \begin{tabular}{l|cc|cc|cc}
    \toprule
    & \multicolumn{2}{c|}{\textbf{IMDB – DistilBERT}} 
    & \multicolumn{2}{c|}{\textbf{CIFAR-10 – ResNet9}} 
    & \multicolumn{2}{c}{\textbf{MNIST – LeNet5}}\\
    \textbf{Method} 
    & \textbf{Acc. ($\uparrow$)} & \textbf{MIA Acc. ($\downarrow$)} 
    & \textbf{Acc. ($\uparrow$)} & \textbf{MIA Acc. ($\downarrow$)} 
    & \textbf{Acc. ($\uparrow$)} & \textbf{MIA Acc. ($\downarrow$)}\\
    \midrule
    FedAvg & 68.02 $\pm$ 7.03 & 72.34 $\pm$ 3.33 
           & 29.83 $\pm$ 0.85 & 63.52 $\pm$ 1.28 
           & 84.80 $\pm$ 1.76 & 60.17 $\pm$ 1.29 \\
    FedAvg ($\epsilon,\delta$)\,-LDP & 53.79 $\pm$ 0.30 & 52.00 $\pm$ 3.97 
           & 15.13 $\pm$ 1.37 & 58.83 $\pm$ 1.73 
           & 55.69 $\pm$ 2.83 & 57.39 $\pm$ 2.31 \\
    SoteriaFL ($\epsilon,\delta$) & 53.75 $\pm$ 0.81 & 52.70 $\pm$ 3.86 
           & 14.46 $\pm$ 0.55 & 59.07 $\pm$ 1.64 
           & 51.75 $\pm$ 2.77 & 57.64 $\pm$ 1.80 \\
    PriPrune ($p_1$) & 57.01 $\pm$ 7.37 & 70.48 $\pm$ 3.42 
           & 21.30 $\pm$ 2.36 & 59.73 $\pm$ 1.52 
           & 74.80 $\pm$ 3.78 & 58.50 $\pm$ 1.56 \\
    PriPrune ($p_2$) & 53.64 $\pm$ 3.40 & 67.52 $\pm$ 3.52 
           & 11.51 $\pm$ 1.41 & 57.71 $\pm$ 3.03 
           & 24.94 $\pm$ 5.34 & 53.95 $\pm$ 1.89 \\
    PriPrune ($p_3$) & 54.03 $\pm$ 3.68 & 63.04 $\pm$ 4.23 
           & 10.98 $\pm$ 0.91 & 55.53 $\pm$ 2.24 
           & 15.54 $\pm$ 1.47 & 52.64 $\pm$ 2.45 \\
    Shatter & 62.16 $\pm$ 4.19 & 68.26 $\pm$ 4.85 
           & 11.63 $\pm$ 1.31 & 62.19 $\pm$ 2.18 
           & 17.73 $\pm$ 5.56 & 59.67 $\pm$ 2.67 \\
    \textsc{eris} (+\textsc{dsc}) & 67.49 $\pm$ 7.51 & 66.95 $\pm$ 4.72 
           & 30.16 $\pm$ 1.35 & 62.72 $\pm$ 2.04 
           & 85.10 $\pm$ 0.84 & 58.17 $\pm$ 2.24 \\
    \midrule
    Min.\ Leakage & 68.88 $\pm$ 6.75 & 68.85 $\pm$ 3.00
               & 29.80 $\pm$ 0.86 & 61.92 $\pm$ 3.09 
           & 84.95 $\pm$ 1.73 & 56.08 $\pm$ 1.67 \\
    \bottomrule
    \end{tabular}
    \label{tab:overall_acc_mia_all_biased_setting}
    \vspace{-0pt}
\end{table}

\subsection{Balancing Utility and Privacy - Biased Gradient Estimator} \label{app:biased_ut_pr}
In this section, we extend our analysis of the utility–privacy trade-off to the biased setting (already adopted for CNN/DailyMail dataset), where each client performs multiple local updates per communication round, 
thereby introducing bias into the gradient estimator. The training hyperparameters are detailed in Section \ref{app:model_and_hyper}.

Figure~\ref{fig:fig1_biased} summarizes the performance of \textsc{eris} with DSC and several SOTA privacy-preserving baselines in terms of test accuracy and MIA accuracy across different model sizes and local training regimes. As in the main paper (Figure~\ref{fig:unbiased_acc_priv}), we evaluate datasets with distinct memorization characteristics—ranging from lightweight models such as LeNet-5 on MNIST to large-scale architectures like GPT-Neo 1.3B on CNN/DailyMail—and vary client-side overfitting by controlling the number of training samples per client. The observed trends mirror those under unbiased conditions: \textsc{eris} with DSC consistently achieves the best overall trade-off, retaining accuracy comparable to non-private FedAvg while substantially reducing privacy leakage toward the idealized \textit{Min.~Leakage} baseline. For instance, on IMDB with 4 local samples per client and identical training conditions (e.g., same communication rounds), \textsc{eris} with DSC achieves an accuracy of 67.8 $\pm$ 4.9, comparable to FedAvg’s 66.9 $\pm$ 5.5, while significantly reducing MIA accuracy from 92.3\% to 68.2\%—approaching the unattainable upper bound of 66.9\% obtained by not sharing local gradients. The only methods that surpass \textsc{eris} with DSC in privacy protection are DP-based approaches, which, however, degrade test accuracy to nearly random-guess levels, namely SoteriaFL (53.1 $\pm$ 0.8) and FedAvg-LDP (53.4 $\pm$ 0.5). Indeed, DP-based methods reduce leakage only at the cost of severe utility degradation, particularly for larger models, while decentralized methods with partial gradient exchange, such as Shatter, often fail to converge within the predefined number of communication rounds—especially when models are trained from scratch.  

Table~\ref{tab:overall_acc_mia_all_nonIID} reports mean test and MIA accuracy under the biased setting, complementing trends in Figure~\ref{fig:fig1_biased}. Consistent with the figure, \textsc{eris} with DSC delivers the strongest utility–privacy balance across data- sets, maintaining accuracy close to FedAvg while reducing leakage toward the \textit{Min.~Leakage} baseline. DP-based methods achieve lower leakage but at a steep accuracy cost, PriPrune trades off utility and privacy depending on the pruning rate, and Shatter struggles to converge reliably. These results further confirm that \textsc{eris} with DSC offers the most favorable trade-off, even in biased local training regimes.  


\begin{figure*}[h!]
   \captionsetup{skip=5pt}
  \centering
  \includegraphics[width=1.0\textwidth]{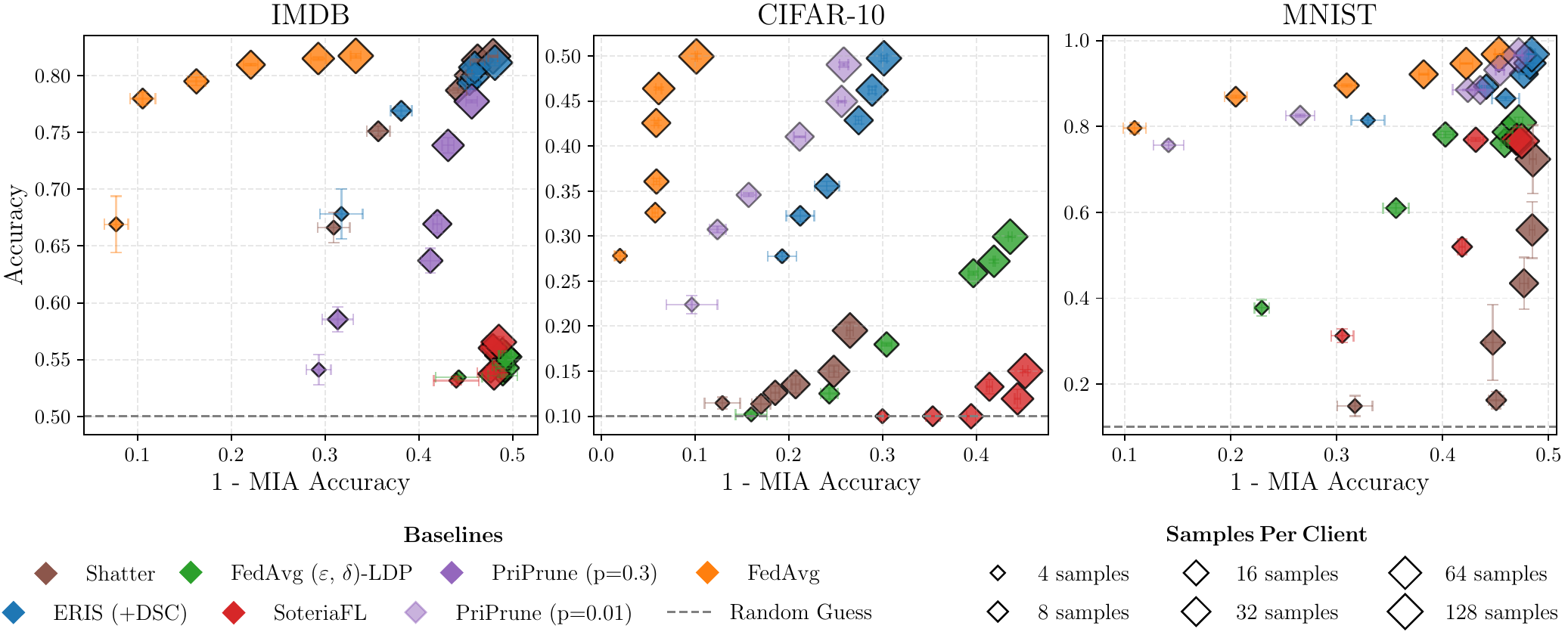}
    \caption{Comparison of test accuracy and (1-MIA) accuracy across varying model capacities (one per dataset) and client-side overfitting levels, controlled via the number of training samples per client using a biased gradient estimator. 
    }
  \label{fig:fig1_biased}
\end{figure*}

\begin{table}[htbp]
    \captionsetup{skip=5pt}
    \centering
    \caption{Mean test accuracy and MIA accuracy, averaged over varying local sample sizes using a biased gradient estimator. For DP-based methods, $\epsilon\!\in\!\{10,100\}$ on IMDB and  $\epsilon\!\in\!\{1,10\}$ on others; for PriPrune, pruning rates are $p\!\in\!\{0.1,0.2,0.3\}$ on IMDB and $p\!\in\!\{0.01,0.05,0.1\}$ on others.}
    \tiny
    \renewcommand{\arraystretch}{1.2}
    \setlength{\tabcolsep}{4pt}
    \begin{tabular}{l|cc|cc|cc}
    \toprule
    & \multicolumn{2}{c|}{\textbf{IMDB – DistilBERT}} 
    & \multicolumn{2}{c|}{\textbf{CIFAR-10 – ResNet9}} 
    & \multicolumn{2}{c}{\textbf{MNIST – LeNet5}}\\
    \textbf{Method} 
    & \textbf{Acc. ($\uparrow$)} & \textbf{MIA Acc. ($\downarrow$)} 
    & \textbf{Acc. ($\uparrow$)} & \textbf{MIA Acc. ($\downarrow$)} 
    & \textbf{Acc. ($\uparrow$)} & \textbf{MIA Acc. ($\downarrow$)}\\
    \midrule
    FedAvg & 78.11 $\pm$ 1.42 & 80.13 $\pm$ 2.11 
           & 39.23 $\pm$ 0.76 & 94.03 $\pm$ 0.69 
           & 89.95 $\pm$ 0.77 & 68.65 $\pm$ 1.63 \\
    FedAvg ($\epsilon_1,\delta$)\,-LDP & 54.39 $\pm$ 0.57 & 51.75 $\pm$ 2.69 
           & 10.78 $\pm$ 0.36 & 59.17 $\pm$ 1.21
           & 29.02 $\pm$ 1.20 & 58.81 $\pm$ 0.95 \\
    FedAvg ($\epsilon_2,\delta$)\,-LDP & 55.00 $\pm$ 1.10 & 52.69 $\pm$ 2.74 
           & 20.62 $\pm$ 0.57 & 67.34 $\pm$ 1.46 
           & 68.77 $\pm$ 2.09 & 60.33 $\pm$ 1.42 \\
    SoteriaFL ($\epsilon_1,\delta$) & 54.90 $\pm$ 0.74 & 52.66 $\pm$ 2.73 
           & 10.03 $\pm$ 0.01 & 56.30 $\pm$ 0.84 
           & 14.23 $\pm$ 1.39 & 56.25 $\pm$ 0.95 \\
    SoteriaFL ($\epsilon_2,\delta$) & 55.40 $\pm$ 1.88 & 53.44 $\pm$ 2.62 
           & 11.70 $\pm$ 0.75 & 60.69 $\pm$ 0.89 
           & 65.11 $\pm$ 1.88 & 57.16 $\pm$ 1.16 \\
    PriPrune ($p_1$) & 76.52 $\pm$ 1.08 & 73.00 $\pm$ 2.30 
           & 37.12 $\pm$ 0.75 & 81.59 $\pm$ 2.17 
           & 91.11 $\pm$ 0.23 & 57.83 $\pm$ 1.63 \\
    PriPrune ($p_2$) & 71.62 $\pm$ 1.30 & 66.66 $\pm$ 2.24 
           & 25.32 $\pm$ 1.13 & 64.16 $\pm$ 1.81 
           & 61.29 $\pm$ 1.16 & 55.14 $\pm$ 1.39 \\
    PriPrune ($p_3$) & 65.82 $\pm$ 1.90 & 61.22 $\pm$ 2.10 
           & 13.91 $\pm$ 0.83 & 56.35 $\pm$ 1.67 
           & 53.99 $\pm$ 1.46 & 51.72 $\pm$ 1.42 \\
    Shatter & 77.33 $\pm$ 0.87 & 58.26 $\pm$ 2.03 
           & 13.91 $\pm$ 1.77 & 79.90 $\pm$ 1.74 
           & 32.03 $\pm$ 11.54 & 56.47 $\pm$ 1.68 \\
    \textsc{eris} (+\textsc{dsc}) & 77.59 $\pm$ 1.38 & 57.44 $\pm$ 2.24 
           & 39.06 $\pm$ 1.01 & 74.81 $\pm$ 1.99 
           & 90.16 $\pm$ 0.68 & 55.45 $\pm$ 1.68 \\
    \midrule
    Min.\ Leakage & 78.11 $\pm$ 1.42 & 57.02 $\pm$ 2.10
               & 39.23 $\pm$ 0.87 & 76.22 $\pm$ 1.91
               & 90.06 $\pm$ 0.71 & 55.20 $\pm$ 1.50 \\
    \bottomrule
    \end{tabular}
    \label{tab:overall_acc_mia_all_nonIID}
\end{table}

\subsection{Pareto Analysis under Varying Privacy Constraints} \label{app:pareto}
This section complements the analysis in Paragraph~\ref{pg:pareto} of the main text with additional details and numerical results. We study how the utility–privacy trade-off evolves under different strengths of privacy-preserving mechanisms and varying numbers of local training samples. Shatter is excluded, as it already fails to converge reliably with 16 samples per client (Figure~\ref{fig:pareto_16}). For DP-based approaches (FedAvg+LDP and SoteriaFL), we vary the privacy budget $\epsilon$ together with the clipping norm $C$ to simulate different protection levels. Following the same configuration, we also evaluate \textsc{eris} with LDP, where LDP is applied on top of its native masking mechanism. For pruning-based methods (PriPrune), we vary the pruning rate $p$ to control information flow through gradient sparsification. The exact configurations of these hyperparameters are in Table~\ref{tab:ppm_acc_privleak_pct}.

Figure~\ref{fig:pareto} shows the utility–privacy trade-off across different numbers of local training samples. As expected, the Pareto frontier shifts toward higher accuracy and lower privacy leakage as clients are assigned more local data. Across all regimes, \textsc{eris} dominates the frontier, contributing the large majority of favorable points, while baselines are mostly dominated. 

Table~\ref{tab:ppm_acc_privleak_pct} reports the underlying quantitative results, including additional configurations not visualized in the figure. These include finer granularity in both $p$ and $\epsilon$ values, enabling a more exhaustive comparison. The results confirm the trends observed in the main text: \textsc{eris} consistently occupies favorable positions in the utility–privacy space, contributing most points along the Pareto frontier. Moreover, when augmented with LDP, \textsc{eris} demonstrates further privacy gains with only minor utility losses—outperforming other baselines that suffer substantial degradation as privacy constraints tighten. Overall, this detailed breakdown reinforces that \textsc{eris} achieves strong privacy guarantees and high utility, even under stringent privacy budgets and aggressive compression strategies.

\begin{figure*}[h!]
  \centering
  \includegraphics[width=0.99\textwidth]{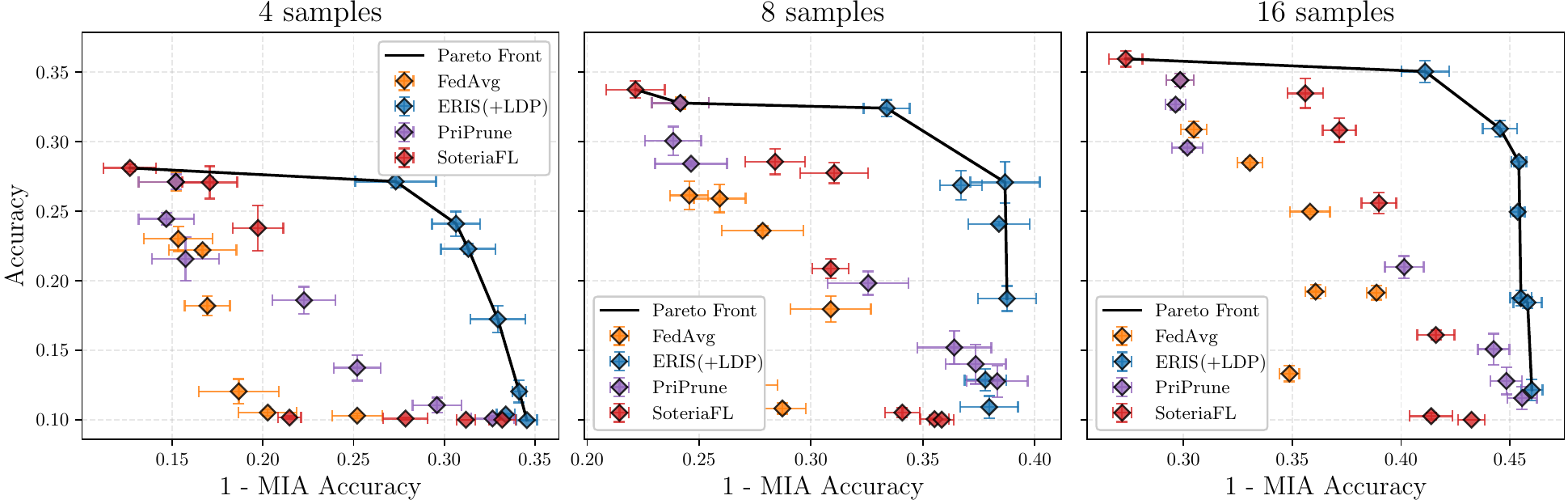}
  \caption{Utility–privacy trade-off on CIFAR-10 under varying strengths of the privacy-preserving mechanisms. Each subplot shows test accuracy vs. (1-MIA) accuracy for methods with different client training samples. The Pareto front represents a set of optimal trade-off points.}
  \label{fig:pareto}
  \vspace{-10pt} 
\end{figure*}

\begin{table}[htbp]
\centering
\vspace{-10pt}
\caption{Mean test accuracy and privacy leakage (with standard deviation) for various privacy‐preserving mechanisms across different local sample sizes. DP‐based methods use epsilon $\epsilon$ and clipping norm $C$; PriPrune uses rate $p$.}
\scriptsize
\renewcommand{\arraystretch}{1.2}
\setlength{\tabcolsep}{4pt}
\begin{tabular}{l|cc|cc|cc}
\toprule
& \multicolumn{2}{c|}{\textbf{4 samples}}
& \multicolumn{2}{c|}{\textbf{8 samples}}
& \multicolumn{2}{c}{\textbf{16 samples}}\\
\textbf{Method}
& \textbf{Accuracy ($\uparrow$)} & \textbf{MIA Acc. ($\downarrow$)}
& \textbf{Accuracy ($\uparrow$)} & \textbf{MIA Acc. ($\downarrow$)}
& \textbf{Accuracy ($\uparrow$)} & \textbf{MIA Acc. ($\downarrow$)}\\
\midrule
\multicolumn{7}{l}{\textbf{FedAvg + LDP}}\\
No LDP                              & 27.12\% $\pm$ 1.20\% & 84.80\% $\pm$ 4.59\% 
                                    & 32.98\% $\pm$ 0.61\% & 75.84\% $\pm$ 2.85\%
                                    & 34.43\% $\pm$ 1.04\% & 70.15\% $\pm$ 1.41\% \\
LDP ($\epsilon$=0.001,\,C=10)         & 23.01\% $\pm$ 2.00\% & 84.67\% $\pm$ 4.24\%
                                    & 26.13\% $\pm$ 2.30\% & 75.44\% $\pm$ 1.90\%
                                    & 30.88\% $\pm$ 1.30\% & 69.53\% $\pm$ 1.31\% \\
LDP ($\epsilon$=0.01,\,C=5)           & 22.20\% $\pm$ 0.71\% & 83.33\% $\pm$ 4.17\%
                                    & 25.91\% $\pm$ 2.23\% & 74.08\% $\pm$ 2.58\%
                                    & 28.48\% $\pm$ 0.72\% & 66.95\% $\pm$ 1.29\% \\
LDP ($\epsilon$=0.1,\,C=2)            & 18.20\% $\pm$ 1.58\% & 83.07\% $\pm$ 2.78\%
                                    & 23.60\% $\pm$ 0.76\% & 72.16\% $\pm$ 4.06\%
                                    & 24.97\% $\pm$ 0.48\% & 64.18\% $\pm$ 2.05\% \\
LDP ($\epsilon$=0.3,\,C=1)            & 12.04\% $\pm$ 1.98\% & 81.33\% $\pm$ 4.94\%
                                    & 17.95\% $\pm$ 2.08\% & 69.12\% $\pm$ 4.01\%
                                    & 19.14\% $\pm$ 1.12\% & 61.13\% $\pm$ 0.99\% \\
LDP ($\epsilon$=0.6,\,C=1)            & 10.52\% $\pm$ 0.55\% & 79.73\% $\pm$ 3.59\%
                                    & 12.50\% $\pm$ 1.42\% & 72.88\% $\pm$ 3.10\%
                                    & 19.22\% $\pm$ 1.12\% & 63.93\% $\pm$ 1.05\% \\
LDP ($\epsilon$=1.0,\,C=1)            & 10.29\% $\pm$ 0.40\% & 74.80\% $\pm$ 3.08\%
                                    & 10.81\% $\pm$ 0.88\% & 71.28\% $\pm$ 2.37\%
                                    & 13.32\% $\pm$ 1.24\% & 65.13\% $\pm$ 1.01\% \\
\midrule
\multicolumn{7}{l}{\textbf{\textsc{eris} + LDP}}\\
No LDP                              & 27.14\% $\pm$ 0.95\% & 72.67\% $\pm$ 4.99\%
                                    & 32.21\% $\pm$ 1.32\% & 66.62\% $\pm$ 2.30\%
                                    & 35.05\% $\pm$ 1.75\% & 58.89\% $\pm$ 2.45\% \\
LDP ($\epsilon$=0.001,\,C=10)         & 24.10\% $\pm$ 1.98\% & 69.35\% $\pm$ 2.95\%
                                    & 26.87\% $\pm$ 2.37\% & 63.30\% $\pm$ 2.10\%
                                    & 30.95\% $\pm$ 1.32\% & 55.45\% $\pm$ 1.77\% \\
LDP ($\epsilon$=0.01,\,C=5)           & 22.29\% $\pm$ 0.86\% & 68.67\% $\pm$ 3.37\%
                                    & 27.08\% $\pm$ 3.33\% & 59.32\% $\pm$ 3.47\%
                                    & 28.55\% $\pm$ 0.71\% & 54.59\% $\pm$ 0.77\% \\
LDP ($\epsilon$=0.1,\,C=2)            & 17.24\% $\pm$ 2.15\% & 67.04\% $\pm$ 3.40\%
                                    & 24.09\% $\pm$ 0.33\% & 61.60\% $\pm$ 3.09\%
                                    & 24.96\% $\pm$ 0.63\% & 54.64\% $\pm$ 0.73\% \\
LDP ($\epsilon$=0.3,\,C=1)            & 12.03\% $\pm$ 1.77\% & 65.87\% $\pm$ 0.84\%
                                    & 18.72\% $\pm$ 2.03\% & 61.24\% $\pm$ 2.92\%
                                    & 18.76\% $\pm$ 1.23\% & 54.49\% $\pm$ 1.10\% \\
LDP ($\epsilon$=0.6,\,C=1)            & 10.38\% $\pm$ 0.16\% & 66.61\% $\pm$ 1.12\%
                                    & 12.87\% $\pm$ 1.73\% & 62.21\% $\pm$ 2.06\%
                                    & 18.43\% $\pm$ 0.93\% & 54.18\% $\pm$ 1.45\% \\
LDP ($\epsilon$=1.0,\,C=1)            &  9.97\% $\pm$ 0.04\% & 65.44\% $\pm$ 1.27\%
                                    & 10.91\% $\pm$ 1.80\% & 62.94\% $\pm$ 2.88\%
                                    & 12.15\% $\pm$ 1.69\% & 53.99\% $\pm$ 1.11\% \\
\midrule
\multicolumn{7}{l}{\textbf{SoteriaFL}}\\
No LDP                              & 28.11\% $\pm$ 0.61\% & 87.33\% $\pm$ 3.24\%
                                    & 33.76\% $\pm$ 1.38\% & 77.84\% $\pm$ 2.92\%
                                    & 35.96\% $\pm$ 1.28\% & 72.65\% $\pm$ 1.72\% \\
LDP ($\epsilon$=0.001,\,C=10)         & 27.08\% $\pm$ 2.60\% & 82.93\% $\pm$ 3.39\%
                                    & 28.56\% $\pm$ 2.05\% & 71.60\% $\pm$ 2.98\%
                                    & 33.48\% $\pm$ 2.37\% & 64.40\% $\pm$ 1.82\% \\
LDP ($\epsilon$=0.01,\,C=5)           & 23.79\% $\pm$ 3.62\% & 80.27\% $\pm$ 3.12\%
                                    & 27.76\% $\pm$ 1.67\% & 68.96\% $\pm$ 3.39\%
                                    & 30.84\% $\pm$ 1.93\% & 62.84\% $\pm$ 1.70\% \\
LDP ($\epsilon$=0.1,\,C=2)            & 10.16\% $\pm$ 0.26\% & 78.53\% $\pm$ 1.42\%
                                    & 20.86\% $\pm$ 1.56\% & 69.12\% $\pm$ 1.81\%
                                    & 25.59\% $\pm$ 1.71\% & 61.02\% $\pm$ 1.78\% \\
LDP ($\epsilon$=0.3,\,C=1)            & 10.09\% $\pm$ 0.18\% & 72.13\% $\pm$ 2.75\%
                                    & 10.52\% $\pm$ 0.76\% & 65.92\% $\pm$ 1.72\%
                                    & 16.10\% $\pm$ 0.78\% & 58.40\% $\pm$ 1.91\% \\
LDP ($\epsilon$=0.6,\,C=1)            & 10.00\% $\pm$ 0.00\% & 68.80\% $\pm$ 1.15\%
                                    & 10.04\% $\pm$ 0.08\% & 64.48\% $\pm$ 1.44\%
                                    & 10.26\% $\pm$ 0.52\% & 58.62\% $\pm$ 2.20\% \\
LDP ($\epsilon$=1.0,\,C=1)            & 10.00\% $\pm$ 0.00\% & 66.80\% $\pm$ 1.65\%
                                    & 10.00\% $\pm$ 0.00\% & 64.16\% $\pm$ 1.20\%
                                    & 10.00\% $\pm$ 0.00\% & 56.76\% $\pm$ 1.37\% \\
\midrule
\multicolumn{7}{l}{\textbf{PriPrune}}\\
No Pruning                         & 27.12\% $\pm$ 1.20\% & 84.80\% $\pm$ 4.59\%
                                    & 32.98\% $\pm$ 0.61\% & 75.84\% $\pm$ 2.85\%
                                    & 34.43\% $\pm$ 1.04\% & 70.15\% $\pm$ 1.41\% \\
Pruning ($p$=0.0005)               & 24.45\% $\pm$ 1.00\% & 85.33\% $\pm$ 3.40\%
                                    & 30.06\% $\pm$ 2.30\% & 76.16\% $\pm$ 2.80\%
                                    & 32.67\% $\pm$ 0.79\% & 70.36\% $\pm$ 1.02\% \\
Pruning ($p$=0.001)                & 21.57\% $\pm$ 3.51\% & 84.27\% $\pm$ 4.12\%
                                    & 28.42\% $\pm$ 0.39\% & 75.36\% $\pm$ 3.60\%
                                    & 29.57\% $\pm$ 0.70\% & 69.82\% $\pm$ 1.59\% \\
Pruning ($p$=0.005)                & 18.60\% $\pm$ 2.15\% & 77.73\% $\pm$ 3.88\%
                                    & 19.83\% $\pm$ 1.89\% & 67.44\% $\pm$ 4.04\%
                                    & 20.98\% $\pm$ 1.77\% & 59.85\% $\pm$ 1.99\% \\
Pruning ($p$=0.01)                 & 13.74\% $\pm$ 2.05\% & 74.80\% $\pm$ 2.87\%
                                    & 15.20\% $\pm$ 2.67\% & 63.60\% $\pm$ 3.71\%
                                    & 15.08\% $\pm$ 2.49\% & 55.75\% $\pm$ 1.61\% \\
Pruning ($p$=0.03)                 & 11.05\% $\pm$ 1.22\% & 70.40\% $\pm$ 3.00\%
                                    & 14.00\% $\pm$ 3.18\% & 62.64\% $\pm$ 3.00\%
                                    & 12.79\% $\pm$ 2.17\% & 55.16\% $\pm$ 1.63\% \\
Pruning ($p$=0.05)                 & 10.09\% $\pm$ 0.36\% & 67.33\% $\pm$ 2.63\%
                                    & 12.77\% $\pm$ 2.52\% & 61.68\% $\pm$ 3.04\%
                                    & 11.55\% $\pm$ 1.80\% & 54.44\% $\pm$ 1.50\% \\
\bottomrule
\end{tabular}
\vspace{-10pt}
\label{tab:ppm_acc_privleak_pct}
\end{table}

\subsection{Temporal Evolution of Privacy Leakage and Utility}
\label{app:temporal_privacy}

To complement the worst-case MIA summaries reported in the main paper, we examine how privacy leakage evolves throughout training. We consider the 64-sample IID MNIST configuration used in Table \ref{tab:overall_acc_mia_all}, with $K=50$ clients and $A=50$ aggregators, and train for $T=250$ communication rounds. We compare \textsc{eris} with DSC against FedAvg, PriPrune ($p=0.01$), and SoteriaFL. At every round, we record client validation accuracy and MIA accuracy. Curves and shaded regions report the mean and standard deviation across three seeds.

\begin{figure*}[h!]
    \centering
    \includegraphics[width=.9\textwidth]{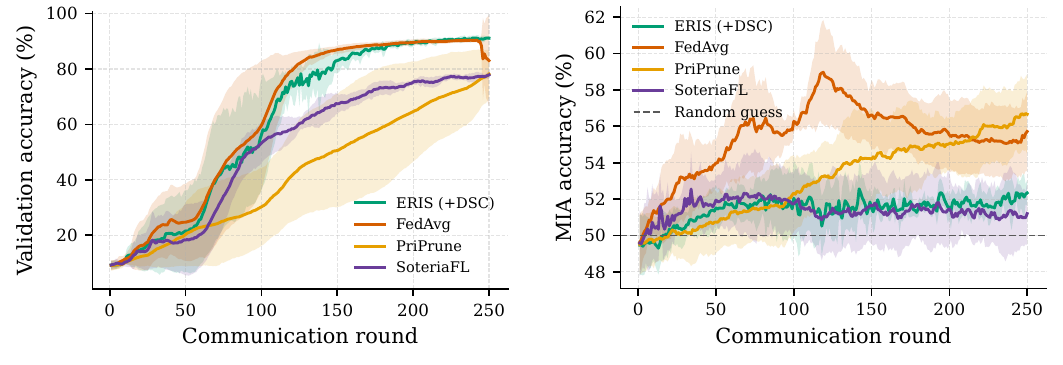}
    \caption{Temporal evolution of utility and privacy leakage on MNIST. \textit{Left:} validation accuracy. \textit{Right:} MIA accuracy, with random guessing indicated at 50\%. Curves and shaded regions report the mean and standard deviation over three seeds.}
    \label{fig:mia_rounds}
    \vskip -0.1in
\end{figure*}

Figure~\ref{fig:mia_rounds} shows that the privacy trajectories are method-dependent and need not evolve monotonically. FedAvg exhibits a pronounced increase in MIA accuracy as training progresses, followed by a partial decrease after reaching its peak leakage, while PriPrune shows a more gradual increase. In contrast, \textsc{eris} with DSC remains close to random-guess MIA accuracy throughout training, with only mild fluctuations, while closely tracking the validation performance of FedAvg. SoteriaFL achieves similarly low MIA accuracy but at substantially lower validation accuracy. Overall, these trajectories show that the privacy advantage of \textsc{eris} is maintained throughout training.

The per-round MIA trajectories and Theorem~\ref{thm:privacy} provide complementary views of privacy leakage. The theorem bounds the information accumulated in the adversary's transcript over $T$ rounds and therefore scales with $T$, whereas Figure~\ref{fig:mia_rounds} measures the success of a specific attack at each individual round. Accordingly, the empirical MIA trajectories may stabilize or fluctuate even as the cumulative worst-case bound increases with training duration.

\newpage
\clearpage

\end{document}